\documentclass[11pt,letterpaper]{article}
\pdfoutput=1
\usepackage{fullpage}
\usepackage{makeidx}         
\usepackage{graphicx}        

\usepackage{amsmath,amssymb,amsthm,mathrsfs}
\usepackage[colorlinks,bookmarksopen,bookmarksnumbered,citecolor=red,urlcolor=red]{hyperref}
\usepackage{natbib}
\usepackage[font=footnotesize]{subcaption}
\usepackage[font=footnotesize]{caption}
\usepackage{algorithmic}
\usepackage{algorithm}
\usepackage{color}
\usepackage{bm}
\usepackage{bbm}
\usepackage{float}
\usepackage{stackengine}
\usepackage{epstopdf}
\usepackage{enumitem}
\usepackage[title]{appendix}
\usepackage{accents}
\usepackage{authblk}

\usepackage{psfrag}
\usepackage{xspace}

\newtheoremstyle{nospace}
{2pt}   				
{2pt}   				
{\itshape}  			
{} 		  		    	
{\bfseries} 			
{.}         			
{5pt plus 1pt minus 1pt}
{}          			

\theoremstyle{nospace} \newtheorem{theorem}{Theorem}
\theoremstyle{nospace} 
\theoremstyle{nospace} \newtheorem{lemma}[theorem]{Lemma}
\theoremstyle{nospace} \newtheorem{remark}{Remark}
\theoremstyle{nospace} 
\theoremstyle{nospace} \newtheorem{definition}{Definition}
\theoremstyle{nospace} 
\theoremstyle{nospace}

\newcommand{\argmin}{\operatornamewithlimits{argmin}}
\newcommand{\argmax}{\operatornamewithlimits{argmax}}
\newcommand{\softmin}{\operatornamewithlimits{softmin}}

\newcommand{\RR}{\mathbb{R}}
\newcommand{\W}{\mathcal{W}}
\newcommand{\D}{\mathcal{D}}
\newcommand{\Z}{\mathcal{Z}}
\newcommand{\PP}{\mathbb{P}}
\newcommand{\Pp}{\mathcal{P}}
\newcommand{\F}{\mathcal{F}}
\newcommand{\B}{\mathcal{B}}
\newcommand{\U}{\mathcal{U}}
\newcommand{\EE}{\mathbb{E}}
\newcommand{\I}{\mathcal{I}}
\newcommand{\J}{\mathcal{J}}
\newcommand{\N}{\mathcal{N}}
\newcommand*\Eval[3]{\left.#1\right\rvert_{#2}^{#3}}

\newcommand{\revision}[1]{{\color{black}{#1}}}

\graphicspath{{./fig/}}

\begin{document}

\title{Risk-sensitive Inverse Reinforcement Learning \\ via Semi- and Non-Parametric Methods}
\date{}


%

\author[1]{Sumeet Singh}
\author[2]{Jonathan Lacotte}
\author[3]{Anirudha Majumdar}
\author[1]{Marco Pavone}
\affil[1]{Department of Aeronautics and Astronautics, Stanford University \thanks{\{ssingh19, pavone\}@stanford.edu}}
\affil[2]{Department of Electrical Engineering, Stanford University \thanks{lacotte@stanford.edu}}
\affil[3]{Department of Mechanical and Aerospace Engineering, Princeton University \thanks{ani.majumdar@princeton.edu}}



%

\maketitle

\begin{abstract}
The literature on Inverse Reinforcement Learning (IRL) typically  assumes that humans take actions in order to minimize the expected value of a cost function, i.e., that humans are {\em risk neutral}. Yet, in practice, humans are often far from being risk neutral.
To fill this gap, the objective of this paper is to devise a framework for {\em risk-sensitive} IRL in order to explicitly account for a human's risk sensitivity.  
To this end, we propose a flexible class of models based on \emph{coherent risk measures}, which allow us to capture an entire spectrum of risk preferences from risk-neutral to worst-case. We propose efficient non-parametric algorithms based on linear programming and semi-parametric algorithms based on maximum likelihood for inferring a human's underlying risk measure and cost function for a rich class of static and dynamic decision-making settings. The resulting approach is demonstrated on a simulated driving game with ten human participants. Our method is able to infer and mimic a wide range of qualitatively different driving styles from highly risk-averse to risk-neutral in a data-efficient manner. Moreover, comparisons of the Risk-Sensitive (RS) IRL approach with a risk-neutral model show that the RS-IRL framework more accurately captures observed participant behavior both qualitatively and quantitatively, especially in scenarios where catastrophic outcomes such as collisions can occur.
\end{abstract}


\section{Introduction}\label{sec:introduction}

Imagine a world where robots and humans coexist and work seamlessly together. In order to realize this vision, robots \revision{should, among other things,} be able to (1) accurately predict the actions of humans in their environment, (2) quickly learn the preferences of human agents in their proximity and act accordingly, and (3) learn how to accomplish new tasks from human demonstrations. Inverse Reinforcement Learning (IRL) \citep{Russell1998, NgRussell2000, AbbeelNg2005, LevineKoltun2012, RamachandranAmir2007, ZiebartMaasEtAl2008, EnglertToussaint2015}  is a powerful and flexible framework for tackling these challenges and has been previously used for a wide range of tasks, including modeling and mimicking human driver behavior \citep{AbbeelNg2004, KudererGulatiEtAl2015, SadighSastryEtAl2016},  pedestrian trajectory prediction \citep{ZiebartRatliffEtAl2009, MombaurTruongEtAl2010,KretzschmarSpiesEtAl2016}, and legged robot locomotion \citep{ZuckerBagnellEtAl2010, KolterAbbeelEtAl2007, ParkLevine2013}. More recently, the popular technique of Max-Entropy (MaxEnt) IRL, an inspiration for some of the techniques leveraged in this work, has been adopted in a deep learning framework~\citep{WulfmeierOndruskaEtAl2015}, and embedded within the guided policy optimization algorithm~\citep{FinnLevineEtAl2016}. The underlying assumption behind IRL is that humans act optimally with respect to an (unknown) cost function. The goal of IRL is then to infer this cost function from observed actions of the human. By learning the human's underlying preferences (in contrast to, e.g., directly learning a policy for a given task), IRL allows one to generalize one's predictions to novel scenarios and environments.


The prevalent modeling assumption made by existing IRL techniques is that humans take actions in order to minimize the \emph{expected value} of a random cost. Such a model, referred to as the expected value (EV) model, implies that humans are \emph{risk neutral} with respect to the random cost; yet, humans are often far from being risk neutral. A generalization of the EV model is represented by the expected utility (EU) theory in economics~\citep{NeumannMorgenstern1944}, whereby one assumes that a human is an optimizer of the expected value of a disutility function of a random cost. Despite the historical prominence of EU theory in modeling human behavior, a large body of literature from the theory of human decision making strongly suggests that humans behave in a manner that is \emph{inconsistent} with the EU model. At a high level, the EU model has two main limitations: (1) experimental evidence consistently confirms that this model is lacking in its ability to describe human behavior in risky scenarios~\citep{Allais1953,Ellsberg1961,KahnemanTversky1979}, and (2)  the EU model assumes that humans make no distinction between scenarios in which the probabilities of outcomes are known and ones in which they are unknown, which is often not the case. Consequently, a robot interacting with a human in a safety-critical setting (e.g., autonomous driving or navigation using shared autonomy), while leveraging such an inference model, could make incorrect assumptions about the human agent's behavior, potentially leading to catastrophic outcomes.

The known and unknown probability scenarios are referred to as \emph{risky} and \emph{ambiguous} respectively in the decision theory literature. An elegant illustration of the role of ambiguity is provided by the \emph{Ellsberg paradox}~\citep{Ellsberg1961}. Imagine an urn (Urn 1) containing 50 red and 50 black balls. Urn 2 also contains 100 red and black balls, but the relative composition of colors is unknown. Suppose that there is a payoff of $\$10$ if a red ball is drawn (and no payoff for black). In human experiments, subjects display an overwhelming preference towards having a ball drawn from Urn 1. However, now suppose the subject is told that a black ball has $\$10$ payoff (and no payoff for red). Humans \emph{still} prefer to draw from Urn 1. This is a {\em paradox}, since choosing to draw from Urn 1 in the first case (payoff for red) indicates that the human assesses the proportion of red in Urn 1 to be higher than in Urn 2, while choosing Urn 1 in the second case (payoff for black) indicates that the human assesses a lower proportion of red in Urn 1 than in Urn 2. Indeed, there is no utility function for the two outcomes that can resolve such a contradictory assessment of underlying probabilities since it stems from a subjective distortion of outcome \emph{probabilities} rather than \emph{rewards}. 

The limitations of EU theory in modeling human behavior has prompted substantial work on various alternative theories such as rank-dependent expected utility~\citep{Quiggin1982}, expected uncertain utility~\citep{GulPesendorfer2014}, dual theory of choice (distortion risk measures)~\citep{Yaari1987}, prospect theory~\citep{KahnemanTversky1979, Barberis2013}, and many more (see~\citep{MajumdarPavone2017} for a recent review of the various axiomatic underpinnings of these risk measures). Further, one way to interpret the Ellsberg paradox is that humans are not only risk averse, but are also \emph{ambiguity averse} -- an observation that has sparked an alternative set of literature in decision theory on ``ambiguity-averse'' modeling; see, e.g., the recent review~\citep{GilboaMarinacci2016}. It is clear that the assumptions made by EU theory thus represent significant restrictions from a modeling perspective in an IRL context since a human expert is likely to be both risk and ambiguity averse, especially in safety critical applications such as driving where outcomes are inherently ambiguous and can possibly incur very high cost. 

The key insight of this paper is to address these challenges by modeling humans as evaluating costs according to an (unknown) \emph{risk measure}. A risk measure is a function that maps an uncertain cost to a real number (the expected value is thus a particular risk measure and corresponds to risk neutrality). In particular, we will consider the class of \emph{coherent risk measures} (CRMs) \citep{ArtznerDelbaenEtAl1999, Shapiro2009, Ruszczynski2010}. CRMs were proposed within the operations research community and have played an influential role within the modern theory of risk in finance \citep{RockafellarUryasev2000, AcerbiTasche2002, Acerbi2002,Rockafellar2007}. This theory has also recently been adopted for risk-sensitive (RS) Model Predictive Control and decision making \citep{ChowPavone2014, ChowTamarEtAl2015}, and guiding autonomous robot exploration for maximizing information gain in time-varying environments \citep{AxelrodCarloneEtAl2016}. 

Coherent risk measures enjoy a number of useful properties that jointly provide key advantages over EV and EU theories in the context of IRL. First, they capture an entire spectrum of risk assessments from risk-neutral to worst-case and thus offer a significant degree of modeling flexibility. Second, they capture risk sensitivity in an \emph{axiomatically justified} manner; specifically, they formally capture a number of intuitive properties that one would expect any risk measure should satisfy (see Section \ref{sec:coherent risks}). Third, a representation theorem for CRMs (Section \ref{sec:coherent risks}) implies that they can be interpreted as computing the expected value of a cost function in a worst-case sense over a \emph{set} of probability distributions (referred to as the \emph{risk envelope}). Thus, CRMs capture both risk and ambiguity aversion within the \emph{same modeling framework} since the risk envelope can be interpreted as capturing uncertainty about the underlying probability distribution that generates outcomes in the world. Finally, they are tractable from a computational perspective; the representation theorem allows us to solve both the inverse and forward problems in a computationally tractable manner for a rich class of static and dynamic decision-making settings. 

\emph{Statement of contributions:} This paper presents an IRL algorithm that explicitly takes into account risk sensitivity under \emph{general} axiomatically-justified risk models that jointly capture risk and ambiguity within the same modeling framework.
To this end, this paper makes four primary contributions. First, we propose a flexible modeling framework for capturing risk sensitivity in humans by assuming that the human demonstrator (hereby referred to as the ``expert'') acts according to a CRM. This framework allows us to capture an entire spectrum of risk assessments from risk-neutral to worst-case. Second, we develop efficient algorithms based on Linear Programming (LP) for inferring an expert's underlying risk measure for a broad range of static (Section~\ref{sec:single period}) decision-making settings, including a proof of convergence of the predictive capability of the algorithm in the case where we only attempt to learn the risk measure. We additionally consider cases where both the cost and risk measure of the expert are unknown. Third, we develop a maximum likelihood based model for inferring the expert's risk measure and cost function for a rich class of dynamic decision-making settings (Section~\ref{sec:multi period}), generalizing our work in~\citep{MajumdarSinghEtAl2017}. Fourth,
we demonstrate our approach on a simulated driving game (visualized in Figure \ref{fig:visualization}) using a state-of-the-art commercial driving simulator and present results on ten human participants (Section~\ref{sec:results}). We show that our approach is able to infer and mimic qualitatively different driving styles ranging from highly risk-averse to risk-neutral using only a minute of training data from each participant. We also compare the predictions made by our risk-sensitive IRL (RS-IRL) approach with one that models the expert using expected value theory and demonstrate that the RS-IRL framework more accurately captures observed participant behavior both qualitatively and quantitatively, especially in scenarios involving significant risk to the participant-driven car.

\begin{figure}[h]
\centering
\begin{subfigure}[t]{0.3\textwidth}
	\includegraphics[width=\textwidth]{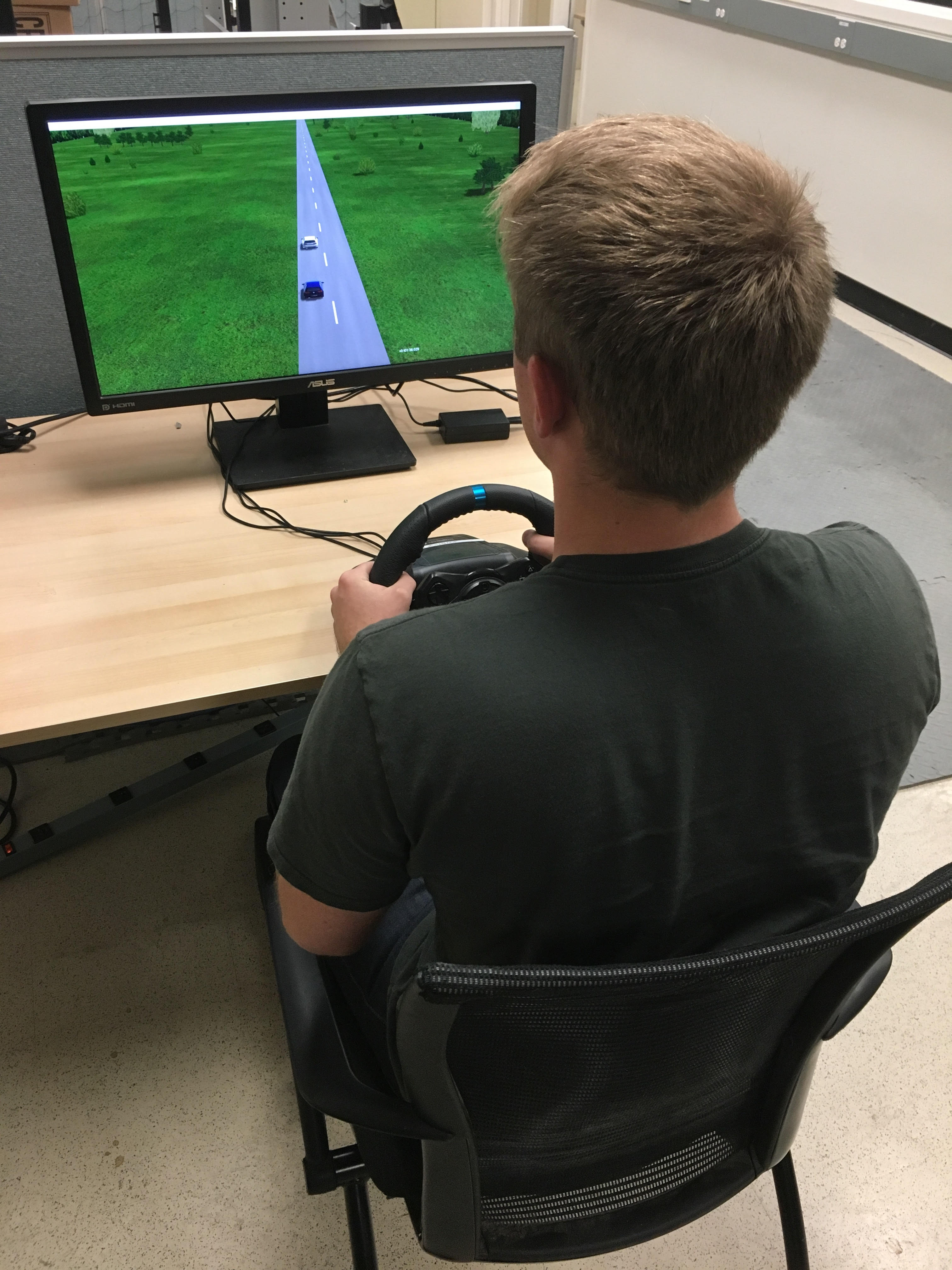}
	\caption{Visualization of simulator during the interactive game experiment as seen by participant.}
	\label{fig:visu_part} 
\end{subfigure}
\qquad
\begin{subfigure}[t]{0.3\textwidth}
	\includegraphics[width=\textwidth]{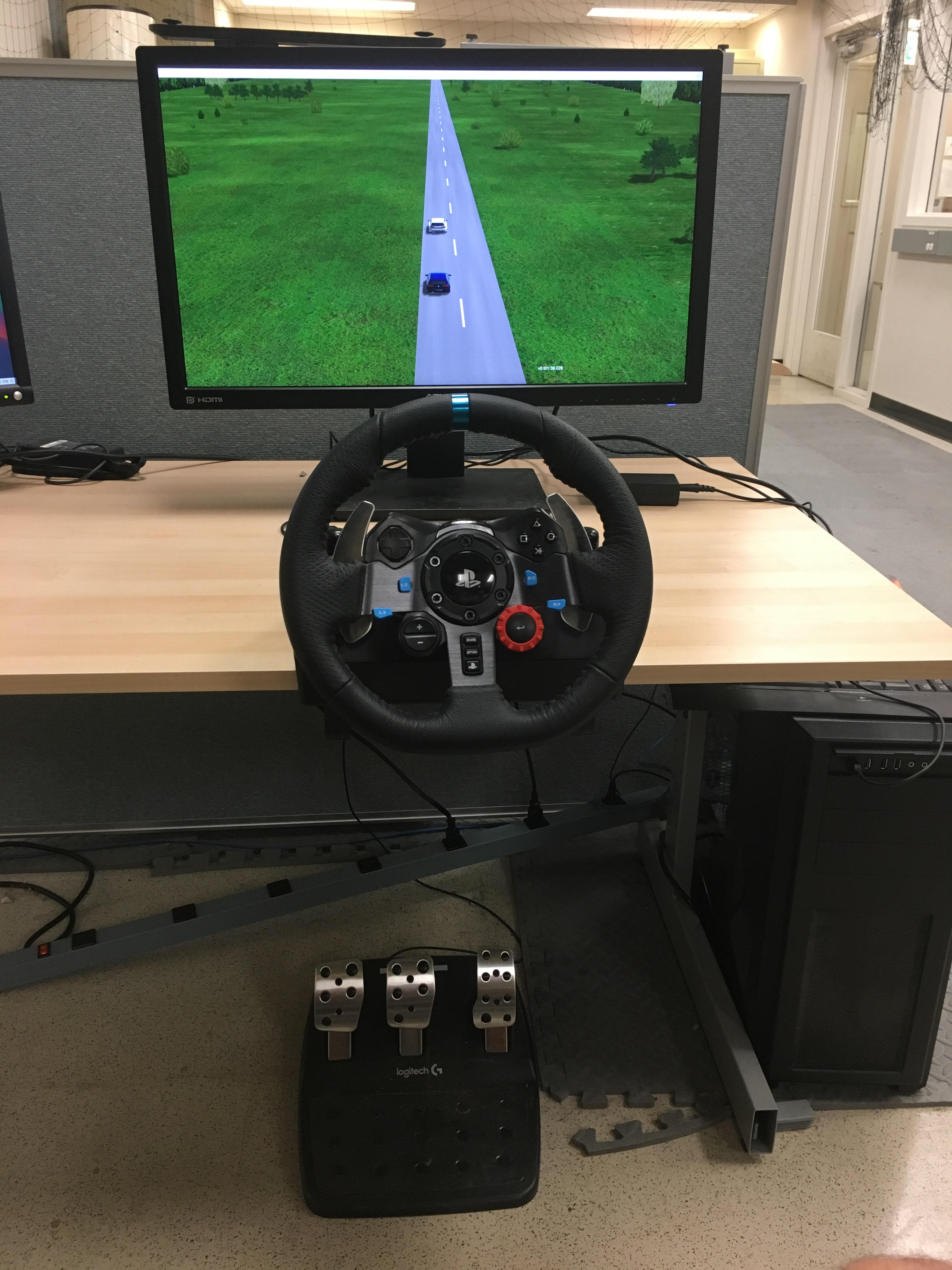}
	\caption{Logitech G29 game input hardware consists of a force-feedback steering wheel and accelerator and brake pedals.}
	\label{fig:hardware} 
\end{subfigure}
\caption{The simulated driving game considered in this paper. The human controls the follower car using a force-feedback steering wheel and two pedals and must follow the leader (an ``erratic driver") as closely as possible without colliding. We observed a wide range of behaviors from participants reflecting varying attitudes towards risk.}
\label{fig:visualization}
\end{figure}

\emph{Related Work:} Safety-critical control and decision making applications demand increased resilience to events of low probability and detrimental consequences (e.g., a UAV crashing due to unexpectedly large wind gusts or an autonomous car failing to accommodate for an erratic neighboring vehicle). Such problems have inspired the recent advancement of various restricted versions of the problems considered here. In particular, there is a large body of work on RS decision making. For instance, in~\citep{HowardMatheson1972} the authors leverage the exponential (or entropic) risk. This has historically been a very popular technique for parameterizing risk-attitudes in decision theory but suffers from the usual drawbacks of the EU framework such as the calibration theorem \citep{Rabin2000}. The latter states that very little risk aversion over moderate costs leads to unrealistically high degrees of risk aversion over large costs, which is undesirable from a modeling perspective. Other RS Markov Decision Process (MDP) formulations include Markowitz-inspired mean-variance~\citep{FilarKallenbergEtAl1989,TamarDiCastroEtAl2012}, percentile criteria on objectives~\citep{WuYuanlie1999} and constraints~\citep{GeibelWysotzki2005}, and cumulative prospect theory~\citep{PrashanthJieEtAl2016}. This has driven research in the design of learning-based solution algorithms, i.e., RS reinforcement learning~\citep{MihatschNeuneier2002,BaeuerleOtt2011,TamarDiCastroEtAl2012,PetrikSubramanian2012,ShenTobiaEtAl2014,TamarChowEtAl2016}. Ambiguity in MDPs is also well studied via the robust MDP framework, see, e.g.,~\citep{NilimElGhaoui2005,XuMannor2010}, as well as~\citep{Osogami2012,ChowTamarEtAl2015} where the risk and ambiguity duality of CRMs is exploited. The key difference between this literature and the present work is that we consider the \emph{inverse} reinforcement learning problem.

Results in the RS-IRL setting are more limited and have largely been pursued in the 
\emph{neuroeconomics} literature \citep{GlimcherFehr2014}. For example, \citep{HsuBhattEtAl2005} performed Functional Magnetic Resonance Imaging (FMRI) studies of humans making decisions in risky and ambiguous settings and modeled risk and ambiguity aversion using parametric utility and weighted probability models. In a similar vein, \citep{ShenTobiaEtAl2014} models risk aversion using utility based shortfalls (with utility functions fixed a priori) and presents FMRI studies on humans performing a sequential investment task. While this literature may be interpreted in the context of IRL, the models used to predict risk and ambiguity aversion are quite limited.  Risk in~\citep{SadighSastryEtAl2016b} is captured via a \emph{single} parameter to represent the aggressiveness of the expert driver -- a fairly limited model that additionally does not account for probabilistic uncertainty. More recently, the authors in~\citep{RatliffMazumdar2017} leverage the shortfall-risk model and associated $Q-$value decomposition introduced in~\citep{ShenTobiaEtAl2014} to devise a gradient-based RS-IRL algorithm. The model again assumes an a priori known risk measure and parameterized utility function and the learning loss function is taken to be the likelihood of the observed actions assuming the Boltzmann distribution fit to the optimal $Q-$values. There are two key limitations of this approach. First, learning is performed assuming a known utility function and risk measure -- both of which, in general, are difficult to fix \emph{a priori} for a given application. Second, computing gradients involves taking expectations with respect to the optimal policy as determined by the current value of the parameters. This must be determined by solving the fixed-point equations defining the ``forward'' RL problem -- a computationally demanding task for large or infinite domains. This limitation is not an artifact of RS-IRL but in fact a standard complexity issue in any MaxEnt IRL-based algorithm. In contrast, this work (1) harnesses the elegant dual representation results for CRMs to avoid having to assume a known risk measure, and (2) solves a significantly less complex forward problem by leveraging a receding-horizon planning model for the expert -- a technique used to great effect also in~\citep{SadighSastryEtAl2016}. 

A first version of this work was presented in~\citep{MajumdarSinghEtAl2017}. In this revised and extended edition, we include the following additional contributions: (1) a significant improvement in the multi-step RS-IRL model which now accounts for an expert planning over sequential disturbance modes (as opposed to the single-stage model in~\citep{MajumdarSinghEtAl2017}); (2) a formal proof of convergence guaranteeing that in the limit, the single-step RS-IRL model will exactly replicate the expert's behavior; (3) introduction of a new maximum likelihood based approach for inferring both the risk measure and cost function for the multi-step model without assuming any a priori functional form; (4) extensive experimental validation on a realistic driving simulator where we demonstrate a significant improvement in predictive performance enabled by the RS-IRL algorithm over the standard risk-neutral model.

\section{Problem Formulation}
\label{sec:problem formulation}

\subsection{Dynamics}

Consider the following discrete-time dynamical system:
\begin{equation}
\label{eq:dynamics}
x_{k+1} = f(x_k, u_k, w_k),
\end{equation}
where $k$ is the time index, $x_k \in \RR^n$ is the state, $u_k \in \RR^m$ is the control input, and $w_k \in \W$ is the disturbance. The control input is assumed to be bounded component-wise: $ u_k \in \mathcal{U}:= \{u : u^- \leq u  \leq u^+ \}$. We take $\W$ to be a finite set $\{w^{[1]},\dots,w^{[L]}\}$ with probability mass function (pmf) $p := [p(1),p(2),\dots,p(L)]$, where $\sum_{i=1}^L p(i) = 1$ and $p(i) > 0, \forall i \in \{1,\dots,L\}$. The time-sampling of the disturbance $w_k$ will be discussed in Section~\ref{sec:multi period}. We assume that we are given demonstrations from an \emph{expert} in the form of sequences of state-control pairs $\{(x_k^*,u_k^*)\}_k$
 and that the expert has knowledge of the underlying dynamics \eqref{eq:dynamics} and disturbance realizations $\W$, but not the disturbance pmf $p$. We will refer back to this assumption within the context of the experimental setting in Section~\ref{sec:results}.

\subsection{Model of the Expert}
\label{sec:coherent risks}

We model the expert as a \emph{risk-sensitive} decision-making agent acting according to a \emph{coherent risk measure} (defined formally below). We refer to such a model as a \emph{coherent risk model}. 

We assume that the expert has a cost function $C(x_k,u_k)$ that captures his/her preferences about outcomes. 
Let $Z$ denote the cumulative cost accrued by the agent \revision{when planning over some finite horizon into the future}. 
Since the process $\{x_k\}$ is stochastic, $Z$ is a random variable adapted to the sequence $\{x_k\}$. A \emph{risk measure} is a function $\rho(Z)$ that maps this uncertain cost to a real number. We will assume that the expert is assessing risks according to a \emph{coherent risk measure}, defined as,

\begin{definition}[Coherent Risk Measures]
\label{def:coherent}
Let $( {\Omega},  {\F},  {\PP})$ be a probability space and let $\Z$ be the space of random variables on $\Omega$. A coherent risk measure (CRM) is a mapping $\rho: \Z \rightarrow \RR$ that obeys the following four axioms. For all $Z,Z' \in \Z$:

{\bf A1. Monotonicity:} $Z \leq Z' \Rightarrow \rho(Z) \leq \rho(Z')$.
\vspace{0.25em} 

{\bf A2. Translation invariance:} $\forall a \in \RR$, $\rho(Z+a) = \rho(Z)+a$.
\vspace{0.25em} 

{\bf A3. Positive homogeneity:} $\forall \lambda  \geq 0$, $\rho(\lambda Z) = \lambda \rho(Z)$.
\vspace{0.25em} 

{\bf A4. Subadditivity:} $\rho(Z + Z') \leq \rho(Z) + \rho(Z')$.
\end{definition}
These axioms were originally proposed in \citep{ArtznerDelbaenEtAl1999} to ensure the ``rationality" of risk assessments. For example, A1 states that if a random cost $Z$ is less than or equal to a random cost $Z'$ \emph{regardless of the disturbance realizations}, then $Z$ must be considered less risky (one may think of the cost distributions $Z$ and $Z'$ stemming from different control policies). A4 reflects the intuition that a risk-averse agent should prefer to \emph{diversify}. We refer the reader to \citep{ArtznerDelbaenEtAl1999,MajumdarPavone2017} for a thorough justification of these axioms. \revision{We provide, below, a hallmark example of coherent risk measures, the Conditional Value-at-Risk (CVaR) at level $\alpha \in (0,1]$. 

For an integrable cost random variable $Z \in \Z$, let the quantity $v_{1-\alpha}(Z) := \inf\{z \in \mathbb{R} \,|\,\PP(Z \le z) \ge 1-\alpha\}$ denote its $(1-\alpha)$-quantile (also referred to as the Value-at-Risk, or VaR). For continuous distributions\footnote{More general definitions can be found in~\citep{RockafellarUryasev2002}.}, $\mathrm{CVaR}_{\alpha}(Z)$ is defined as:
\[
 	\mathrm{CVaR}_{\alpha}(Z) := \mathbb{E}[Z\,|\,Z \ge v_{1-\alpha}(Z)].
\]
That is, $\mathrm{CVaR}_{\alpha}(Z)$ is the expected value of the $\alpha$-tail distribution of $Z$ (see Figure~\ref{fig:cvar}). In particular, one can show that  when $\alpha = 1$, $\mathrm{CVaR}_{\alpha}(Z) $ reduces to the standard expected value $\mathbb{E}[Z]$. Thus, the expected value is a \emph{special case} of CVaR. See~\citep[Chapter 6]{ShapiroDentchevaEtAl2014} for additional examples within this rich class of risk measures, which include CVaR, mean absolute semi-deviation, spectral risk measures, optimized certainty equivalent, and the distributionally robust risk.

\begin{figure}[h]
\centering
	\includegraphics[width=0.5\textwidth]{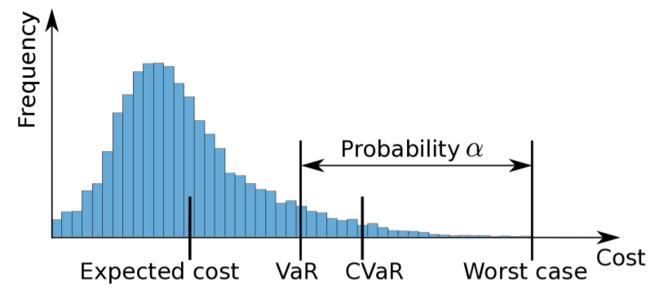}
	\caption{Illustration of the $\mathrm{CVaR}_{\alpha}$ CRM. $\mathrm{CVaR}_{\alpha}(Z)$ quantifies the mean of the $\alpha$-tail of the cost distribution of $Z$.}
\label{fig:cvar}
\end{figure}

}

An important characterization of CRMs is provided by the following representation theorem.

\begin{theorem}[Representation Theorem for Coherent Risk Measures \citep{ArtznerDelbaenEtAl1999}] \label{thm:CRM}
Let $( {\Omega},  {\F},  {\PP})$ be a probability space, where $ {\Omega}$ is a finite set with cardinality $| {\Omega}|$, $ {\F}$ is the $\sigma-$algebra over subsets in $\Omega$ (i.e., $ {\F} = 2^{ {\Omega}}$), probabilities are assigned according to $ {\PP} = ({p}(1),{p}(2),\dots,{p}(| {\Omega}|))$, and $ {\Z}$ is the space of random variables on $ {\Omega}$. Denote by $\mathcal{C}$ the set of probability densities:
\begin{equation}  
\mathcal{C} := \left \{  {\zeta} \in \RR^{| {\Omega}|} \ | \ \sum_{i=1}^{| {\Omega}|} {p}(i){\zeta}(i) = 1, \, {\zeta} \geq 0 \right \}. 
\end{equation}
Define ${q}_{ {\zeta}} \in \RR^{| {\Omega}|}$ where ${q}_{ {\zeta}}(i) = {p}(i) {\zeta}(i)$, $i = 1,\ldots,| {\Omega}|$. A risk measure $\rho:  {\Z} \rightarrow \RR$ with respect to the space $(\Omega, \mathcal{F}, \PP)$ is a CRM if and only if there exists a compact convex set $\B \subset \mathcal{C}$ such that for any $ {Z} \in  {\Z}$:
\begin{equation}  
\rho( {Z}) = \underset{ {\zeta} \in \B}{\max} \  \EE_{{q}_{ {\zeta}}} [ {Z}] = \underset{ {\zeta} \in \B}{\max} \ \sum_{i=1}^{| {\Omega}|} {p}(i) {\zeta}(i)  {Z}(i).
\label{rho_den_def}
\end{equation}
\end{theorem}
This theorem is important for two reasons. Conceptually, it gives us an interpretation of CRMs as computing the worst-case expectation of the cost \revision{with respect to a \emph{set} of \emph{distorted} distributions $q_{\zeta} = p \cdot \zeta$}. Coherent risk measures thus allow us to consider risk and ambiguity (\revision{see} Section \ref{sec:introduction}) in a unified framework since one may interpret an agent acting according to a coherent risk model as being \emph{uncertain about the underlying probability density}. \revision{Practically, estimating this set of distributions provides us with an algorithmic handle for inferring the expert's risk preferences, and indeed will form the basis of our IRL methodology}. 

In this work, we will take the set $\B$ in~\eqref{rho_den_def} to be a polytope. We refer to such risk measures as \emph{polytopic risk measures}, which were also considered in~\citep{EichhornRoemisch2005}. Let $\Delta^{|\Omega|}$ denote the $|\Omega|-$dimensional probability simplex, defined as:
\[
 	\Delta^{| {\Omega}|} := \{q \in \RR^{| {\Omega}|} \ | \ \sum_{i=1}^{| {\Omega}|} q(i) = 1, \ q \geq 0 \}.
\]
By absorbing the density $\zeta$ into the pmf $p$ \revision{in eq.~\eqref{rho_den_def}}, we can represent (without loss of generality) a polytopic risk measure as:
\begin{equation} 
\rho(Z) = \underset{q \in \Pp}{\max} \  \EE_q [Z],
\label{rho_abs}
\end{equation}
where $\Pp$ is a polytopic subset of the probability simplex $\Delta^{| {\Omega}|}$:
\begin{equation}  
\Pp = \left \{q \in \Delta^{| {\Omega}|} \ | \ A_{\mathrm{ineq}}q \leq b_{\mathrm{ineq}} \right \},
\end{equation}
\revision{where the matrix $A_{\mathrm{ineq}} \in \RR^{d\times |\Omega|}$ and vector $b_{\mathrm{ineq}} \in \RR^{d}$ define a set of $d$ halfspace constraints. The polytope $\Pp$ is hereby referred to as the \emph{risk envelope}.} Polytopic risk measures constitute a rich class of risk measures, encompassing a spectrum ranging from risk neutrality ($\Pp = \{p\}$) to worst-case assessments ($\Pp = \Delta^{| {\Omega}|}$); see also~\citep{ChowPavone2014, ShapiroDentchevaEtAl2014}. We further note that the \emph{ambiguity} interpretation of CRMs is reminiscent of Gilboa \& Schmeidler's Minmax EU model for ambiguity-aversion~\citep{GilboaSchmeidler1989} which was shown to outperform various competing models in~\citep{HeyLotitoEtAl2010} for single-stage decision problems, albeit with more restrictions on the set $\B$.


\emph{Goal:} Given demonstrations from the expert in the form of state-control trajectories, the \revision{goal of this paper is to devise an algorithmic framework for risk-sensitive IRL whereby an expert's risk preferences will be estimated by finding an \emph{approximation} of their risk envelope $\Pp$.}

\section{Risk-sensitive IRL: Single Decision Period} \label{sec:single period}
In this section we consider the single step decision problem. \revision{That is, given a current (known) state $x_0$, the expert chooses a single control action $u_0$ to minimize a coherent risk assessment of a random cost $Z$, represented by a non-negative cost function $C(x_1, u_0)$ where $x_1 = f(x_0,u_0,w_0)$. Thus, the uncertain cost $Z$ is a random variable on the discrete probability space $(\W, 2^{\W}, p)$.}

\subsection{Known Cost Function} \label{sec:single period known cost}
We first consider the static decision-making setting where the expert's cost function is known but the risk measure is unknown. A coherent risk model then implies that the expert is solving the following optimization problem at state $x_0$ in order to compute an optimal action: 
\begin{align} 
\tau^* := & \min_{u_0 \in \mathcal{U}} \  \rho(C(x_1,u_0)) = \min_{u_0 \in \mathcal{U}} \max_{q \in \Pp} \  \EE_q[C(x_1,u_0)]  \label{eq:static problem}\\
:= & \min_{u_0 \in \mathcal{U}} \max_{q \in \Pp} \  g(x_0,u_0)^T q, \label{eq:static_g0}
\end{align}
where $g(x_0,u_0)(j)$ is the cost when the disturbance $w_0 = w^{[j]} \in \W$ is realized, and $\rho(\cdot)$ is a CRM with respect to the space $(\W, 2^{\W}, p)$ with risk envelope $\Pp$ being a subset of the probability simplex $\Delta^{L}$. Since the inner maximization problem is linear in $p$, the optimal value is achieved at a vertex of the polytope $\Pp$. Let $\text{vert}(\Pp) = \{v_i\}$ denote the set of vertices of $\Pp$ and let $N_V$ be the cardinality of this set. Then, we can rewrite problem \eqref{eq:static problem} as: 
\begin{alignat}{3}
\label{eq:static problem tau}
\min_{u_0 \in \U, \tau}& & \quad &\tau    \\
s.t.& & &\tau \geq g(x_0,u_0)^Tv_i, \quad  i \in \{1,\dots,N_V\}. & & \nonumber 
\end{alignat}
If the cost function $C(\cdot,\cdot)$ is convex in the control input $u$, the resulting optimization problem is convex.
Given a dataset $\D = \{(x^{*,d},u^{*,d})\}_{d=1}^{D}$ of state-control pairs of the expert taking action $u^{*,d}$ at state $x^{*,d}$, our goal is to deduce an approximation $\Pp_{o}$ of $\Pp$ from the given data. The key idea of our technical approach is to examine the Karush-Kuhn-Tucker (KKT) conditions for Problem \eqref{eq:static problem tau}. The use of KKT conditions for Inverse Optimal Control is a technique also adopted in~\citep{EnglertToussaint2015}. The KKT conditions are necessary for optimality in general and are also sufficient in the case of convex problems. We can thus use the KKT conditions along with the dataset $\D$ to \emph{constrain the constraints} of problem\eqref{eq:static problem tau}. In other words, the KKT conditions will allow us to constrain where the vertices of $\Pp$ must lie in order to be consistent with the fact that the state-control pairs represent optimal solutions to problem \eqref{eq:static problem tau}. Importantly, we will \emph{not} assume access to the number of vertices $N_V$ of $\Pp$. 

\revision{Specifically}, let $(x^*,u^*)$ be an optimal state-control pair and let $\J^+$ and $\J^-$ denote the sets of components of the control input $u^*$ that are saturated above and below respectively (i.e., $u^*(j) = u^+(j)$ \revision{for all} $j \in \J^+$ and $u^*(j) = u^-(j)$ \revision{for all} $j \in \J^-$).

\begin{theorem}[KKT-Based Inference]
\label{thm:bound polytope}
Consider the following optimization problem:

\begin{alignat}{3}
 \label{eq:tau bound known cost}
\max\limits_{\substack{v \in \Delta^L \\ \sigma_+, \sigma_- \geq 0}}& & \quad &g(x^*,u^*)^Tv    \\
s.t.& & &0 = \Eval{\nabla_{u(j)} g(x,u)^Tv}{x^*,u^*}{} + \sigma_+(j), \forall j \in \J^+  & & \nonumber \\
& & &0 = \Eval{\nabla_{u(j)} g(x,u)^Tv}{x^*,u^*}{} - \sigma_-(j), \forall j \in \J^- \nonumber \\
& & &0 = \Eval{\nabla_{u(j)} g(x,u)^Tv}{x^*,u^*}{},  \forall j  \notin \J^+, j \notin \J^- \nonumber  \\
& & & \revision{\sigma_{+}(j) = 0 , \  \sigma_{-} (j) = 0, \quad  \forall j \notin \J^+, j \notin \J^-}  \nonumber
\end{alignat}
Denote the optimal value of this problem by $\tau'$ and define the halfspace:
\begin{equation}
\mathcal{H}_{(x^*,u^*)} := \{v \in \RR^L  \ | \ \tau' \geq g(x^*,u^*)^T v \}.
\end{equation}
Then, the risk envelope $\Pp$ satisfies $\Pp \subset (\mathcal{H}_{(x^*,u^*)} \cap \Delta^L).$
\end{theorem}
\begin{proof}
The KKT conditions for Problem \eqref{eq:static problem tau} are:

\begin{flalign}
&1 = \sum_{i=1}^{N_V} \lambda_i, \\
&0= \lambda_i [g(x^*,u^*)^T v_i - \tau], \ i = 1,\dots,N_V, \\
&\textrm{and for} \ j = 1,\dots,m: \nonumber \\ 
&0 = \sigma_+(j) - \sigma_-(j) + \sum_{i=1}^{N_V} \lambda_i \Eval{\nabla_{u(j)} g(x,u)^T v_i}{x^*,u^*}{}, \\
&0 = \sigma_+(j)[u^*(j) - u^+(j)], \ 0 = \sigma_-(j)[u^-(j) - u^*(j)],
\end{flalign}
where $\lambda_i, \sigma_+(j), \sigma_-(j) \geq 0$ are multipliers. Now, suppose there are multiple optimal vertices $\{v_i\}_{i \in \I}$ for Problem \eqref{eq:static problem tau} in the sense that $\tau^* = g(x^{*}, u^{*})^{T}v_i$, for all $i \in \I$. Defining $\bar{v} := \sum_{i \in \I} \lambda_i v_i$, we see that $\bar{v}$ satisfies:
\begin{equation} 
0 = \nabla_{u(j)} g(x^*,u^*(j))^T \bar{v} + \sigma_+(j) - \sigma_-(j), \quad j = 1,\dots,m,
\end{equation}
and $\tau^* = g(x^*,u^{*})^{T}\bar{v}$ since $\sum_{i\in \I}\lambda_i = 1$. Now, since $\bar{v}$ satisfies the constraints of Problem \eqref{eq:tau bound known cost} (which are implied by the KKT conditions), it follows that $\tau' \geq \tau^*$. From problem \eqref{eq:static problem tau}, we see that $\tau' \geq \tau^* \geq g(x^*,u^*)^T v_i$  for all $v_i \in \text{vert}(\Pp)$ and thus $\Pp \subset (\mathcal{H}_{(x^*,u^*)} \cap \Delta^L)$.
\end{proof}
Problem \eqref{eq:tau bound known cost} is a \emph{Linear Program (LP)} and can thus be solved efficiently.
For each demonstration $(x^{*,d},u^{*,d}) \in \D$, Theorem \ref{thm:bound polytope} provides a halfspace constraint on the risk envelope $\Pp$. By aggregating these constraints, we obtain a \emph{polytopic} outer approximation $\Pp_o$ of $\Pp$. This is summarized in Algorithm \ref{a:bound polytope}. Note that Algorithm \ref{a:bound polytope} operates sequentially through the data $\D$ and is thus directly applicable in \emph{online} settings. \revision{An illustration of the sequential pruning process in Algorithm~\ref{a:bound polytope} is provided in Figure~\ref{fig:single_cartoon}.

\begin{algorithm}[h]
  \caption{{Sequential Halfspace Pruning}}
  \label{a:bound polytope}
  \begin{algorithmic}[1]
    \STATE Initialize $\Pp_o = \Delta^L$
    \FOR{$d = 1,\dots,D$}
    	\STATE Solve Linear Program \eqref{eq:tau bound known cost} with $(x^{*,d},u^{*,d})$ to obtain a hyperplane $\mathcal{H} _{(x^{*,d},u^{*,d})}$
	\STATE Update $\Pp_o \leftarrow \Pp_o \cap \mathcal{H} _{(x^{*,d},u^{*,d})}$
    \ENDFOR
    \STATE Return $\Pp_o$    
  \end{algorithmic}
\end{algorithm}

\begin{figure}[h]
\centering
	\includegraphics[width=1\textwidth]{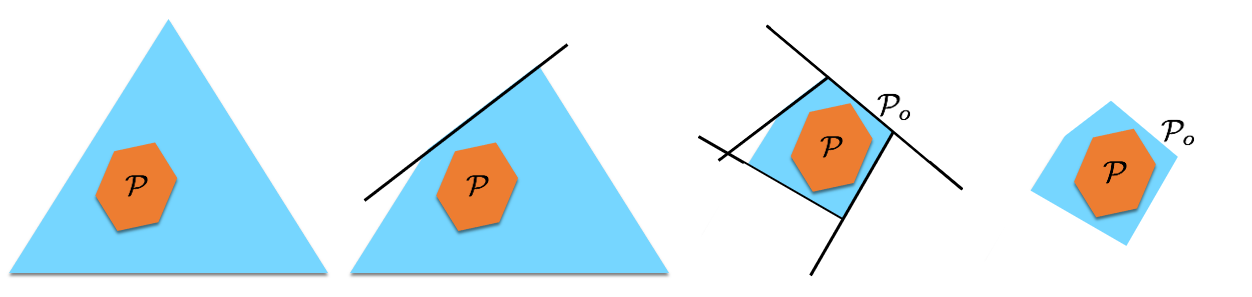}
\caption{Schematic illustration of Algorithm~\ref{a:bound polytope}. Probability simplex (3 scenarios) is shown in blue while the true (unknown) risk envelope is shown in orange. Algorithm~\ref{a:bound polytope} sequentially prunes portions of probability simplex that are inconsistent with the observed actions by leveraging the \emph{necessary} conditions for optimality (KKT conditions) for problem~\eqref{eq:static problem}. The end result is an outer approximation $\Pp_o$ of the true risk envelope $\Pp$. }
\label{fig:single_cartoon}
\end{figure}
}

\begin{remark}
Algorithm~\ref{a:bound polytope} is a \emph{non-parametric} algorithm for inferring the expert's risk measure; i.e., we are not fitting parameters for an a priori chosen risk measure. Instead, by leveraging the dual representation of CRMs as provided by Theorem~\ref{thm:CRM} and reasoning directly over the risk envelope $\Pp$, Algorithm~\ref{a:bound polytope} can recover any risk measure within the class of CRMs, that best explains the expert's demonstrations. 
\end{remark}

\begin{remark}
As we collect more half-space constraints in Algorithm \ref{a:bound polytope}, the constraint $v \in \Delta^L$ in Problem \eqref{eq:tau bound known cost} above can be replaced by $v \in \Pp_o$, where $\Pp_o$ is the current outer approximation of the risk envelope. It is easily verified that the results of Theorem \ref{thm:bound polytope} still hold. This allows us to obtain a tighter (i.e., lower) upper bound $\tau'$ for $\tau^*$, thus resulting in tighter halfspace constraints for each new demonstration processed by the algorithm. 
\end{remark}

Denote \revision{by} $\mathcal{P}_D$ the output of Algorithm \ref{a:bound polytope} after processing sequentially the first $D$ demonstrations $\{(x^{*,d}, u^{*,d})\}_{d=1}^D$. Observe that for all $D \ge 1$, $\mathcal{P}_{D+1} \subseteq \mathcal{P}_D$. We can then define the limiting set as $\Pp_{\infty} := \bigcap_{d=1}^{\infty} \mathcal{P}_d$. An important consideration for this algorithm is whether it is possible to recover, at least from an imitation perspective, the risk envelope $\mathcal{P}$ from sufficiently many optimal demonstrations. In other words, we are specifically interested in the question of whether the limiting set $\Pp_{\infty}$ (whenever such a limit exists) allows one to {\em exactly} predict the actions of a decision maker that operates under a risk model characterized by the set $\mathcal{P}$. In the following theorem we establish, under \revision{some restrictive technical conditions}, that this is indeed possible. The proof is provided in Appendix~\ref{app:theoretical}.

\begin{theorem}[Convergence of Algorithm~\ref{a:bound polytope}]
\label{thm: consistency algorithm}
Let $\mathcal{S} \subseteq \mathbb{R}^n$ be a convex, compact subset of the state space. Let $\{(x^{*,d}, u^{*,d})\}_{d=1}^{\infty}$ be a set of infinitely many optimal demonstrations such that the sequence $\{x^{*,d}\}$ is dense in $\mathcal{S}$. Assume that the following technical conditions hold:
\begin{enumerate}[label=\textbf{A.\arabic*}]
\item \revision{The expert's cost vector $g(x,u)$ is strictly convex with respect to the control input $u$ and continuous with respect to the state variable $x$.} 
\item For all $j \in \{1,\hdots, L\}$ and any state $x \in \mathcal{S}$, the cost function associated with the $j$-th disturbance $u \mapsto g(x,u)(j)$ has bounded level sets.
\end{enumerate}
Finally, for any risk envelope $\mathcal{P}^{\prime} \subseteq \Delta^L$ and any state $x \in \mathcal{S}$, define
\[
	u(\mathcal{P}^{\prime},x) := \argmin_{u \in \mathcal{U}} \max_{v \in \mathcal{P}^{\prime}} v^T g(x,u),
\]
as the (unique) optimal control action of an expert with risk envelope $\mathcal{P}^{\prime}$ at state $x$. Then, for any state $x \in \mathcal{S}$, 
\begin{equation}
\label{consistency}
u(\Pp_{\infty},x) = u(\mathcal{P},x).
\end{equation}
That is, for any state $x \in \mathcal{S}$, the optimal action predicted using the limiting envelope $\Pp_{\infty}$ matches that computed using the true expert polytope $\Pp$.
\end{theorem}
\revision{
\begin{remark}
The technical condition \textbf{A.1} assumes convexity of the cost vector with respect to the control input, and the proof of Theorem~\ref{thm: consistency algorithm} heavily relies on this assumption. Finding conditions under which Algorithm~\ref{a:bound polytope} is guaranteed to be consistent for the general case of non-convex cost functions is an open problem left for future research; we re-emphasize, though, that Algorithm~\ref{a:bound polytope} is guaranteed to provide a conservative outer approximation regardless of the convexity of the cost vectors. 
\end{remark}
} 
Once we have recovered an approximation $\Pp_o$ of $\Pp$,  we can solve the ``forward" problem (i.e., compute actions at a given state $x$) by solving the optimization problem \eqref{eq:static problem} with $\Pp_o$ as the risk envelope. 

\subsubsection{Example: Linear-Quadratic System}
\label{sec:linear quadratic known cost}

As a simple illustrative example to gain intuition for the convergence properties of Algorithm \ref{a:bound polytope}, consider a linear dynamical system with multiplicative uncertainty of the form $f(x_k,u_k,w_k) = A(w_k) x_k + B(w_k) u_k$. We consider the one-step decision-making process with a quadratic cost on state and action: $C := u_0^T R u_0 + x_1^T Q x_1$, where $x_1 =  A(w_0)x_0 + B(w_0)u_0$. Here, $R \succ 0$ and $Q \succeq 0$. We consider a $10$-dimensional state space with a $5$-dimensional control input space. The number of realizations is taken to be $L = 3$ for ease of visualization. The $L$ different $A(w_k)$ and $B(w_k)$ matrices corresponding to each realization are generated randomly by independently sampling elements of the matrices from the standard normal distribution $\N(0,1)$. The cost matrix $Q$ is a randomly generated positive semi-definite matrix and $R$ is the identity. The initial states $x^*$ were drawn randomly from the standard normal distribution $\N(0,I)$ where $I$ denotes the identity matrix. The true envelope was generated by taking the convex hull of a set of random samples in the probability simplex $\Delta^L$.

Figure \ref{fig:lq_polytopes} shows the outer approximations of the risk envelope obtained using Algorithm \ref{a:bound polytope}. We observe rapid convergence (approximately 20 sampled states $x^*$) of the outer approximations $\Pp_o$ (red) to the true risk envelope $\Pp$ (green). Figure \ref{fig:predicted_actions} shows the mean squared error (on an independent test set with 30 demonstrations) between actions predicted using the sequentially refined polytope approximations generated by Algorithm~\ref{a:bound polytope} and the expert's true actions, as a function of the number of training demonstrations. One can observe rapid convergence in prediction performance after just 10 demonstration samples, further highlighting the data efficiency of the proposed algorithm.

\begin{figure}[h]
\centering
\begin{subfigure}[b]{0.35\textwidth}
	\includegraphics[width=1\textwidth]{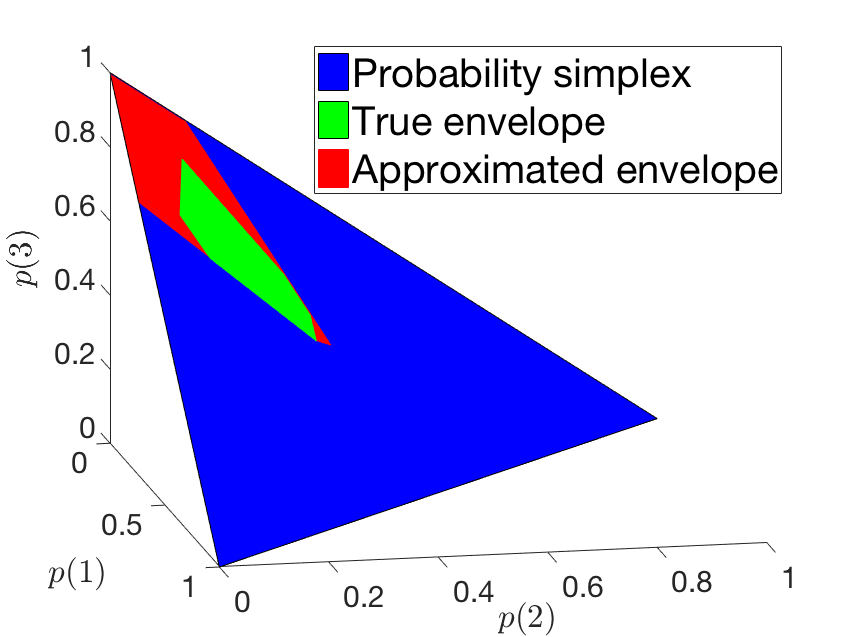}
	\caption{5 data points}
	\label{fig:poly_5} 
\end{subfigure}
\begin{subfigure}[b]{0.35\textwidth}
	\includegraphics[width=1\textwidth]{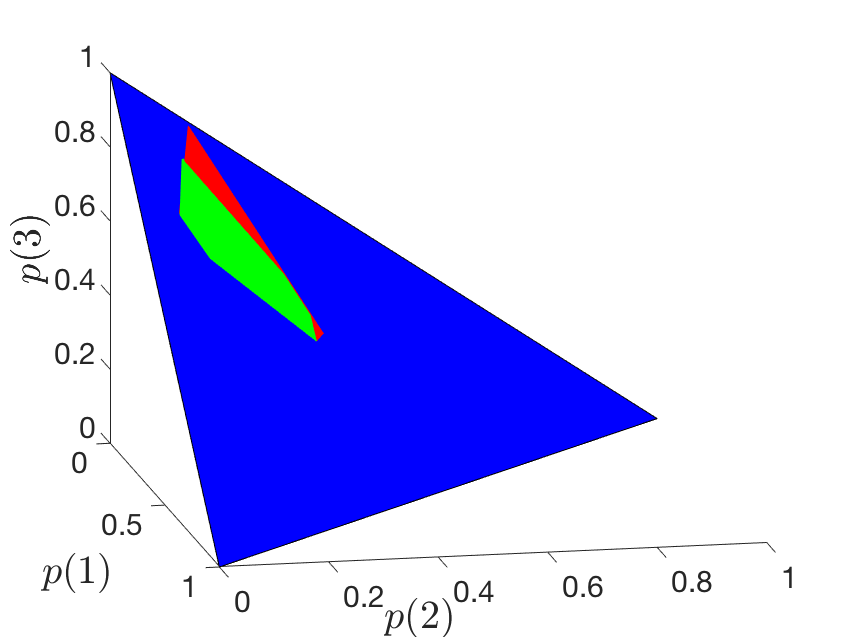}
	\caption{10 data points}
	\label{fig:poly_10} 
\end{subfigure}
\begin{subfigure}[b]{0.35\textwidth}
	\includegraphics[width=1\textwidth]{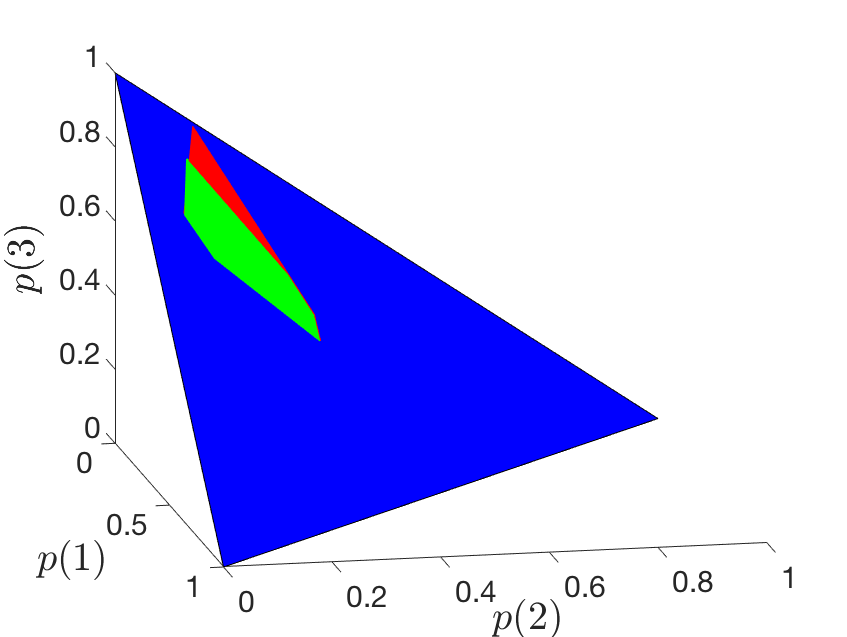}
	\caption{15 data points}
	\label{fig:poly_15} 
\end{subfigure}
\begin{subfigure}[b]{0.35\textwidth}
	\includegraphics[width=1\textwidth]{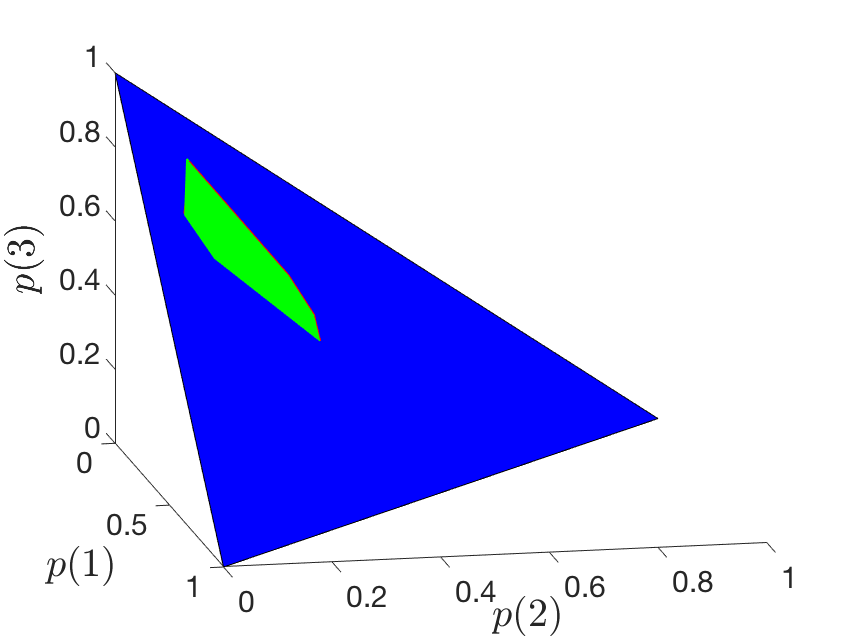}
	\caption{20 data points}
	\label{fig:poly_20} 
\end{subfigure}
\caption{Rapid convergence of the outer approximation of the risk envelope.}
\label{fig:lq_polytopes}
\end{figure}

\begin{figure}[H]
\centering
\includegraphics[trim = 0mm 0mm 5mm 5mm, clip, width=0.5\textwidth]{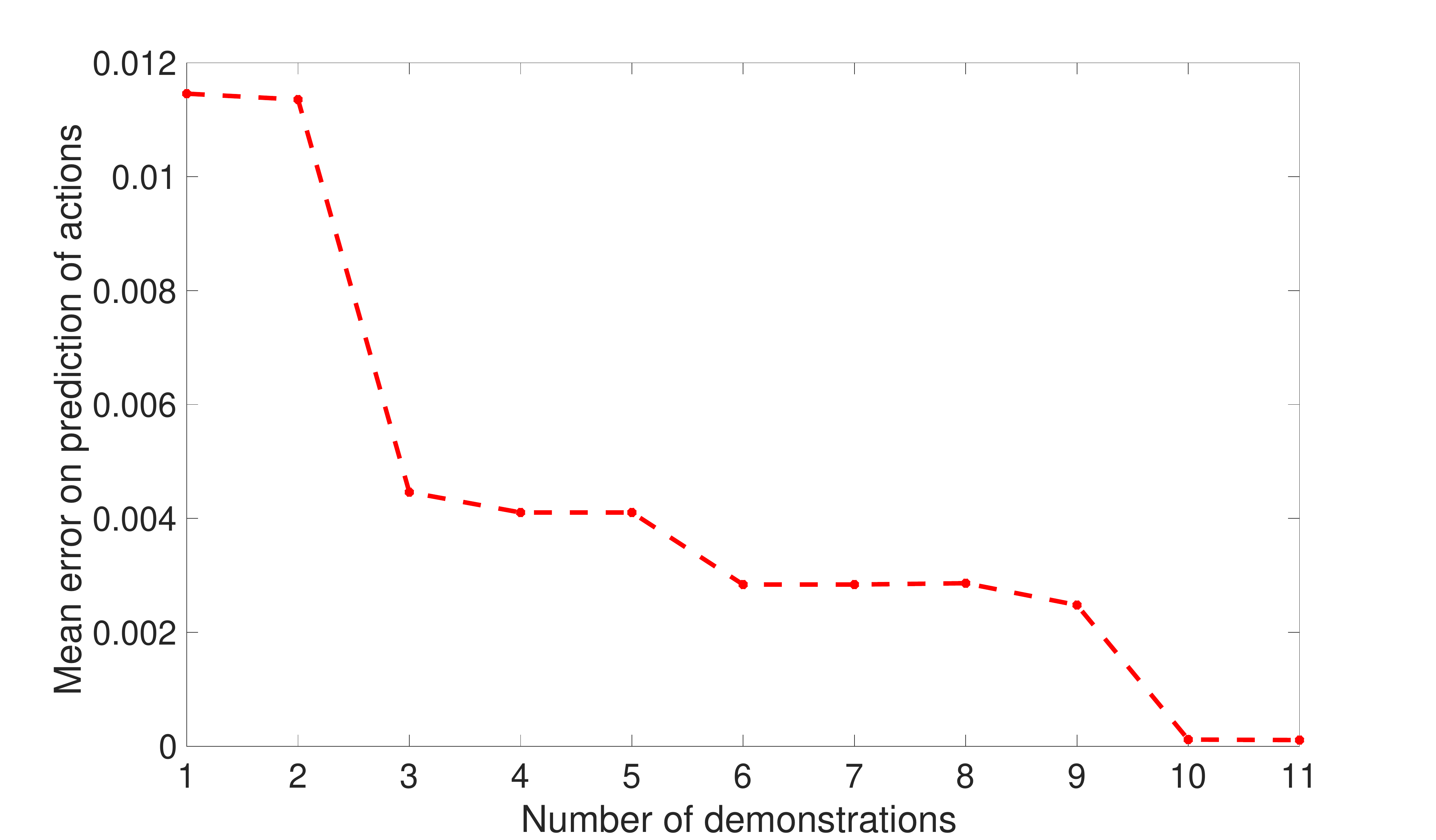}
\caption[]{Rapid decrease of the mean squared error between predicted and expert's actions \revision{on an independent test set}, as a function of the number of \revision{training} demonstrations.
\label{fig:predicted_actions}}
\end{figure}

\subsection{Unknown Cost Function}
\label{sec:single period unknown cost}

We \revision{next} consider the more general case where both the expert's cost function $C(x_1,u_0)$ and risk measure are unknown. We parameterize the cost function as a linearly weighted combination of cost features, \revision{i.e., 
\[
	C(x_1, u_0) = c ^T \phi(x_1,u_0),
\]
where $c \in \RR^{H}_{\geq 0}$ is a vector of \emph{unknown} weights and $\phi: \RR^n \times \U \rightarrow \RR^{H}$ denotes the cost feature mapping from state and control input to an $H-$dimensional real vector. Since the solution of problem \eqref{eq:static problem} solved by the expert is invariant to (i) constant shifts (Axiom A2 in Definition~\ref{def:coherent}) and one can thus absorb negative signs into the cost features, and (ii) positive scalings (Axiom A3 in Definition~\ref{def:coherent}), one can assume without loss of generality that the feature weights $c$ are nonnegative and sum to one. Extending the notation in~\eqref{eq:static_g0}, let  $\phi^{[j]}$ denote the feature vector when disturbance $w_0 = w^{[j]}$ is realized so that 
\begin{equation} 
g(x_0,u_0)(j) =  c ^T \phi^{[j]} (x_0,u_0), \quad  j = 1,\ldots,L.
\label{cost_features}
\end{equation}
Thus, problem~\eqref{eq:static problem tau} now takes the form:
\begin{alignat*}{3}
\min_{u_0 \in \U, \tau}& & \quad &\tau    \\
s.t.& & &\tau \geq \sum_{j=1}^{L} \sum_{h=1}^{H} v_i(j) c(h) \, \phi^{[j]}_h (x_0,u_0), \quad  i \in \{1,\dots,N_V\}.
\end{alignat*}
}
With this cost structure, we see that the KKT conditions derived in Section \ref{sec:single period known cost} now involve \emph{products} of the feature weights $c$ and the vertices $v_i$ of $\Pp$. Thus, an analogous version of optimization problem \eqref{eq:tau bound known cost} can be used to bound the optimal value. This problem will now contain products of the feature weights $c$ and \revision{probability vector} $v$. The key idea here is to introduce new \emph{matrix} decision variables $z$ that replace each product $v(j) c(h)$ by a new variable $z_{jh}$ which allows us to re-write problem~\eqref{eq:tau bound known cost} as an LP in $(z, \sigma_{+}, \sigma_{-})$, with the addition of the following two simple constraints: $0 \leq z_{jh} \leq 1$, for all $j,h$, and $\sum_{j,h} z_{jh} = 1.$
In a manner analogous to Theorem \ref{thm:bound polytope}, this optimization problem allows us to obtain bounding hyperplanes \emph{in the space of product variables} $z$ which can then be aggregated as in Algorithm \ref{a:bound polytope}. 
Denoting this polytope as $\Pp_z$, we can then proceed to solve the ``forward" problem (i.e., computing actions at a given state $x$) by solving the following optimization problem:
\begin{equation}
\label{eq:forward problem}
\min_{u_0 \in \U} \max_{z \in \Pp_z} \ \sum_{j,h} z_{jh} \phi^{[j]}_{h}(x_0,u_0).
\end{equation}
This problem can be solved by enumerating the vertices of the polytope $\Pp_z$ in a manner similar to problem \eqref{eq:static problem tau}. 
Similar to the case where the cost function is known, this provides us with a way to conservatively approximate the expert's decision-making process (in the sense that we are considering a larger risk envelope).


\subsubsection{Approximate Recovery of Cost and Risk Measure}

While the procedure described above operates in the space of product variables $z$ and does not require explicitly recovering the cost function and risk envelope separately, it may nevertheless be useful to do so for two reasons. First, the number of vertices of $\Pp_z$ may be quite large (since the space of product variables may be high dimensional) and thus solving the forward problem \eqref{eq:forward problem} may be computationally expensive. Recovering the cost and risk envelope separately allows us to solve a smaller optimization problem (since the risk envelope is lower dimensional in this case). Second, recovering the cost and risk measure separately may provide additional intuition and insights into the expert's decision-making process and may also allow us to make useful predictions in novel settings (e.g., where we expect the expert's risk measure to be the same but not the cost function or vice versa). 

Here we describe a procedure for approximately recovering the feature weights and the risk envelope from the polytope $\Pp_z$. The key observation that makes this possible is to note that the matrix $z$ containing the variables $z_{jh}$ is equal to the outer product $v c^T$ by definition. Hence, for $h=1,\dots,H$, we have:
\begin{equation} 
\label{eq:recover weights}
\sum_{j=1}^L z_{jh} = \sum_{j=1}^L v(j) c(h) = c(h) \sum_ {j=1}^Lv(j) = c(h).
\end{equation}
The last equality follows from the fact $v$ is a probability vector and sums to 1. Similarly, for $j=1,\dots,L$, we have:
\begin{equation} 
\label{eq:recover probability}
\sum_{h=1}^H z_{jh} = \sum_{h=1}^H v(j) c(h) = v(j) \sum_ {h=1}^H c(h) = v(j).
\end{equation}
The last equality follows from the fact that we assumed without loss of generality that the feature weights sum to 1.

Let $\{\hat{z}_i\}$ be the set of (matrix-valued) vertices of the polytope $\Pp_z$. Then, by applying equations \eqref{eq:recover weights} and \eqref{eq:recover probability} to each vertex $\hat{z}_i$, \revision{we obtain a set of estimates of the feature weight vector $c$ and a set of vectors in the probability simplex $\Delta^L$, the convex hull of which gives an approximation of the risk envelope $\Pp$. If we have exactly recovered the polytope $\Pp_z$ in the space of product variables and the vertices $\hat{z}_i$ each have rank one, then it follows from problem~\eqref{eq:forward problem} that the feature weight estimates will coincide and the convex hull of the probability vectors extracted from the vertices $\hat{z}_i$ will match the true risk envelope $\Pp$. In general however, this will not be the case since (i) there is no guarantee of exactly recovering the product polytope $\Pp_z$ (similar to how there is no guarantee of recovering the true risk envelope $\Pp$ in Algorithm~\ref{a:bound polytope}), and (ii)  $z = v c^T$ is a non-convex rank constraint that is not enforced in the KKT-based LP.}

In light of these limitations, it is important to be able to gauge the quality of the estimates we obtain from the procedure above. We can do this in two ways. First, if the estimates of the weight vector are tightly clustered, this is a good indication that we have an accurate recovery. Second, if each vertex $\hat{z}_i$ of the polytope is close to a rank one matrix, then this is again a good indication (since the true product variables $z$ equal $vc^T$). 

\subsubsection{Example: Linear-Quadratic System}
 \label{sec:linear quadratic unknown cost}
 
Consider the same system as \revision{in Section} \ref{sec:linear quadratic known cost}, but now we assume that the cost function is unknown. We take the cost function as the weighted sum of three quadratic features (i.e., $H = 3$). The quadratic features are generated randomly by taking them to be equal to $SS^T$, where the elements of $S$ are sampled from the standard normal distribution. The corresponding weights are drawn uniformly between 0 and 1 and are normalized to sum to 1. 

Figure~\ref{fig:lq_unknown} a) illustrates the tightness of the approximate envelope as compared with the true polytope, while Figure~\ref{fig:lq_unknown} b) is a scatter plot of the first two feature weights (the third is uniquely determined given the first two) as recovered from applying eq.~\eqref{eq:recover weights} to each vertex of the compound polytope $\Pp_z$. Notice that the cost feature weight estimates are tightly clustered \revision{near the true weights}. 

\begin{figure}[H]
\centering
\begin{subfigure}[b]{0.4\textwidth}
	\includegraphics[width=1\textwidth]{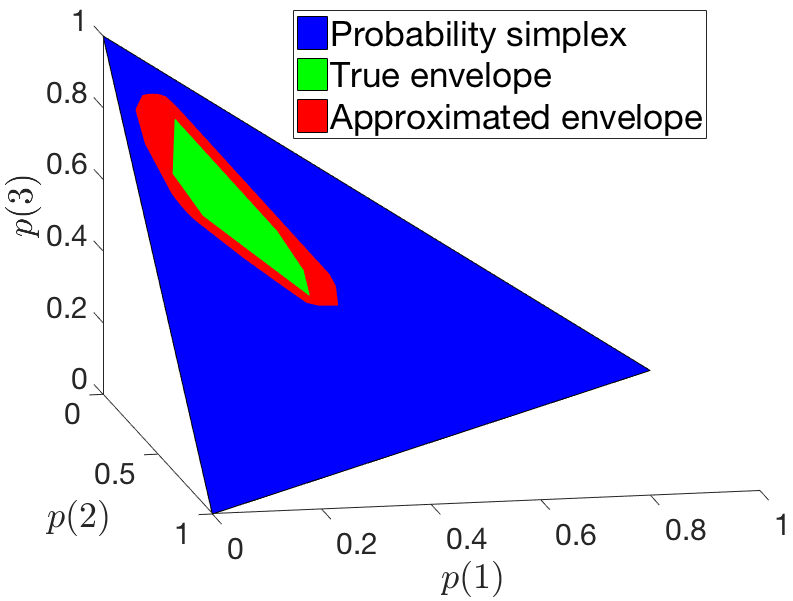}
	\caption{Polytope estimate.}
	\label{fig:poly_200} 
\end{subfigure}
\begin{subfigure}[b]{0.4\textwidth}
	\includegraphics[width=1\textwidth]{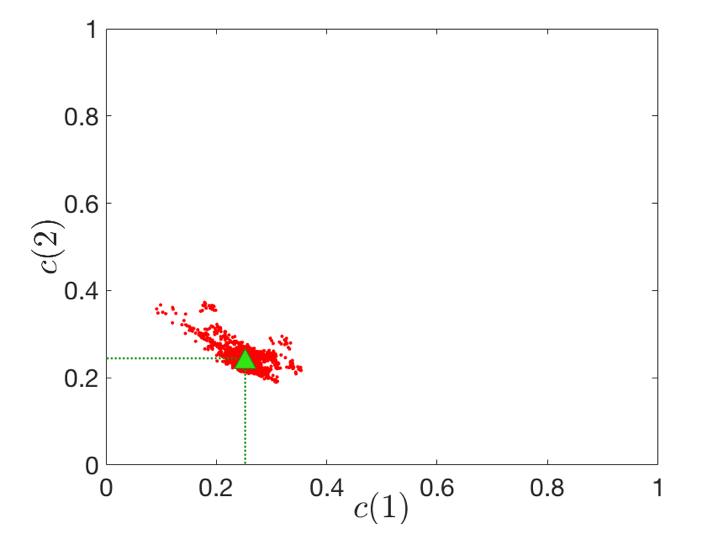}
	\caption{Feature weight estimates \revision{(red), true weights marked by green triangle.}}
	\label{fig:weights_200} 
\end{subfigure}
\caption{Approximated risk envelope and cost feature weights from 200 state-control pair demonstrations. }
\label{fig:lq_unknown}
\end{figure}

\section{Risk-sensitive IRL: Multi-step case}
\label{sec:multi period}

\revision{We now} generalize the one-step decision problem to the multi-step setting. We consider a model where the disturbance $w_k$ is sampled every $N>1$ time-steps and held constant in the interim. Such a model generalizes settings where disturbances are sampled i.i.d. at every time-step (corresponding to $N=1$ in our model) and it allows us to model delays in the expert's reaction to changing disturbances. We model the expert as planning in a \emph{receding horizon} manner by looking ahead for a finite horizon \revision{longer than $N$ steps, executing the computed policy for $N$ steps, and iterating}. Owing to the need to account for future disturbances, the multi-step finite-horizon problem is a search over control \emph{policies} (i.e., the executed control inputs depend on which disturbance is realized). 

\subsection{Prepare-React Model: Preliminaries}
In this section we reprise the ``prepare'' -- ``react'' model introduced in~\citep{MajumdarSinghEtAl2017}, and depicted below in Figure~\ref{fig:multi_look}. 
\begin{figure}[h]
\centering
	\includegraphics[width=0.4\textwidth]{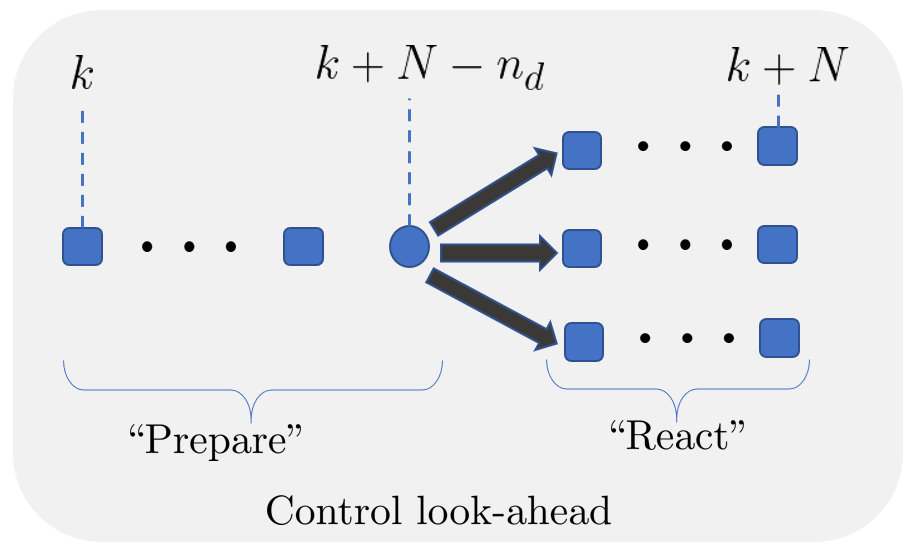}
		\caption{\revision{Scenario tree centered on a disturbance sampled at time $k+N-n_d$, where $N>n_d$. The ``prepare'' phase precedes each disturbance realization by $N-n_d$ steps, while the ``react'' phase follows it by $n_d$ steps; in total, an individual planning stage last $N$ steps. Note that during the prepare and react phases the dynamics are deterministic, as we assume that the disturbances stay constant for $N$ steps in between sampling times. As we model the expert as a receding horizon planner, the expert's planning problem at time $k$ would account for nested re-planning stages at time $k + N, k + 2N, \ldots$ up to some look-ahead horizon. The expert would then execute their policy for the first $N$ steps and then resolve the planning problem at time $k+N$ with a receded horizon.}}
	\label{fig:multi_look} 
\end{figure}
The expert's policy is decomposed into two phases (shown in Figure~\ref{fig:multi_look}), referred to as ``prepare" and ``react." Intuitively, this model captures the idea that in the period preceding a disturbance (i.e., the ``prepare" phase) the expert controls the system to a state from which he/she can recover well (in the ``react" phase) once a disturbance is realized. Studies showing that humans have a relatively short look-ahead horizon in uncertain decision-making settings lend credence to such a model~\citep{CartonNitschEtAl2016}. As in~\citep{MajumdarSinghEtAl2017}, the delay parameter $n_d$ would be learned directly from the demonstrations.
\revision{To account for nested re-planning stages, we need to define a notion} of \emph{dynamic} risk measures, used to assess risk over sequential realizations of uncertainty.

\subsection{Dynamic Risk Measures}

Consider a discrete-time stochastic cost sequence $\{Z_t\}$, where  $Z_t \in \Z_t$ the space of real-valued random variables at stage $t$. Let $\Z_{t:t'} := \Z_t \times \cdots \times \Z_{t'}$ where $ t < t'$. A \emph{dynamic risk measure} is a \emph{sequence} of risk measures $\rho_{t:t'} : \Z_{t:t'} \rightarrow \Z_t, t= 0,\ldots, t'$, each mapping a future stream of random costs into a risk assessment at stage $t$ and satisfying the monotonicity property $\rho_{t:t'}(Z_{t:t'}) \leq \rho_{t:t'}(Y_{t:t'})$ for all $Z_{t:t'},Y_{t:t'} \in \Z_{t:t'}$ such that $Z_{t:t'} \leq Y_{t:t'}$. The monotonicity property is an intuitive extension of the monotonicity property for single-step risk assessments, and an arguably defensible axiom for all risk assessments.

To give dynamic risk measures a concrete functional form, we need to generalize the CRM axioms presented in Definition~\ref{def:coherent} to the dynamic case. 

\begin{definition}[Coherent One-Step Conditional Risk Measures]
\label{def:cond_coherent}
 A \emph{coherent one-step conditional risk measure} is a mapping $\rho_t : \Z_{t+1} \rightarrow \Z_{t}$, for all $t \in \mathbb{N}$, that obeys the following four axioms. For all $Z_{t+1},Y_{t+1} \in \Z_{t+1}$ and $Z_t \in \Z_{t}$:

{\bf A1. Monotonicity:} $Z_{t+1} \leq Y_{t+1} \Rightarrow \rho_t(Z_{t+1}) \leq \rho_t(Y_{t+1})$.
\vspace{0.25em}

{\bf A2. Translation invariance:} $\rho_t(Z_{t+1} + Z_t) = \rho_t(Z_{t+1})+Z_t$.
\vspace{0.25em}

{\bf A3. Positive homogeneity:} $\forall \lambda  \geq 0$, $\rho_t(\lambda Z_{t+1}) = \lambda \rho_t(Z_{t+1})$.
\vspace{0.25em}

{\bf A4. Subadditivity:} $\rho_t(Z_{t+1} + Y_{t+1}) \leq \rho_t(Z_{t+1}) + \rho_t(Y_{t+1})$.
\end{definition}
Note that each $\rho_t$ is a random variable on the space $\Z_t$ and given the discrete underlying probability space, each component of $\rho_t$ is uniquely identified by the sequence of disturbances preceding stage $t$ (hence the term \emph{conditional}). Furthermore, it is readily observed that a mapping $\rho_t: \Z_{t+1} \rightarrow \Z_{t}$ is a coherent one-step conditional risk measure if and only if each component of $\rho_t$ is a CRM. 

As investigated in~\citep{Ruszczynski2010}, in order for dynamic risk assessments to satisfy the intuitive monotonicity condition and to ensure rationality of evaluations over time, a dynamic risk measure must have \revision{a} compositional form:
\begin{equation}
\begin{split}
	\rho_{t:t'} (Z_{t:t'}) :=& Z_t + \rho_t\left(Z_{t+1} + \rho_{t+1}\left(Z_{t+2} + \ldots + \rho_{t'-1}(Z_{t'}) \cdots \right) \right) \\
					=& \rho_t \circ \cdots \circ \rho_{t'-1} (Z_t + \cdots + Z_{t'}),
\end{split}	
\label{dyn_risk}
\end{equation}
where each $\rho_t$ is a coherent one-step conditional risk measure, \revision{and the second equality follows by the translational invariance property}. Figure~\ref{fig:dyn_metric} provides a helpful visualization of \revision{such a} compounded functional form. 
\begin{figure}[h]
\centering
	\includegraphics[width=0.4\textwidth]{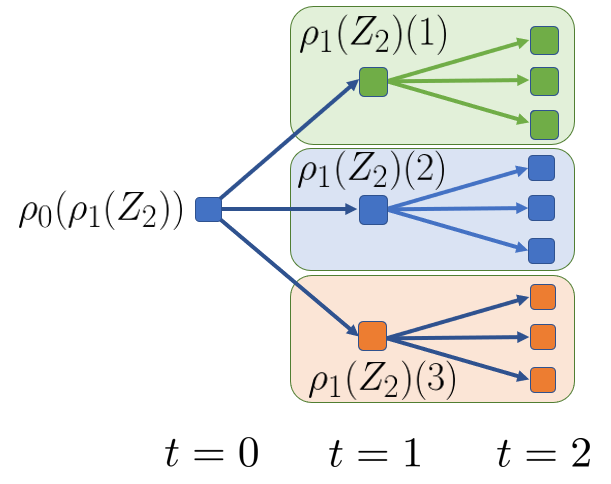}
		\caption{A scenario tree with three uncertain outcomes at each stage. The one-step risk mapping $\rho_1(Z_2) \in \Z_1$ maps the random cost $Z_2 \in \Z_2$ to a risk assessment at stage 1, i.e., is a random variable on $\Z_1$ and is thus isomorphic to the space $\RR_{\geq 0}^3$. Here, each component $j$ of $\rho_1(Z_2)$, i.e., $\rho_1(Z_2)(j)$, associated with node $j$ at stage 1 \revision{(e.g., for $j=1$, the green node)}, is a CRM over the children of node $j$ at stage 2. The mapping $\rho_0(\rho_1(Z_2))$ subsequently maps the risk-assessments at stage 1 (i.e., $\rho_1(\cdot)$) back to stage 0. }
	\label{fig:dyn_metric}
\end{figure}

\subsection{Prepare-React Model: Formal Definition}

We are now ready to \revision{formally} define the expert's multi-step problem, \revision{from the perspective} of time-step $k$, with look-ahead horizon $T\, N$ steps, where $T \in \mathbb{N}_{\geq 1}$ denotes the number of {\em branching events}  \revision{(i.e., disturbance samples)}  within the prediction horizon. \revision{According to} the prepare-react model introduced earlier, we assume that the disturbance mode for the first $N-n_d$ steps starting at time-step $k$ corresponds to $w_{k-n_d}$ \revision{(i.e., the disturbance mode realized at the last sampling event)}, following which the disturbance is re-sampled every $N$ steps. \revision{An illustration of the nested prepare-react model with a  look-ahead horizon $T=2$} is provided in Figure~\ref{fig:multi_pr_look}.
\begin{figure}[h]
\centering
	\includegraphics[width=0.6\textwidth]{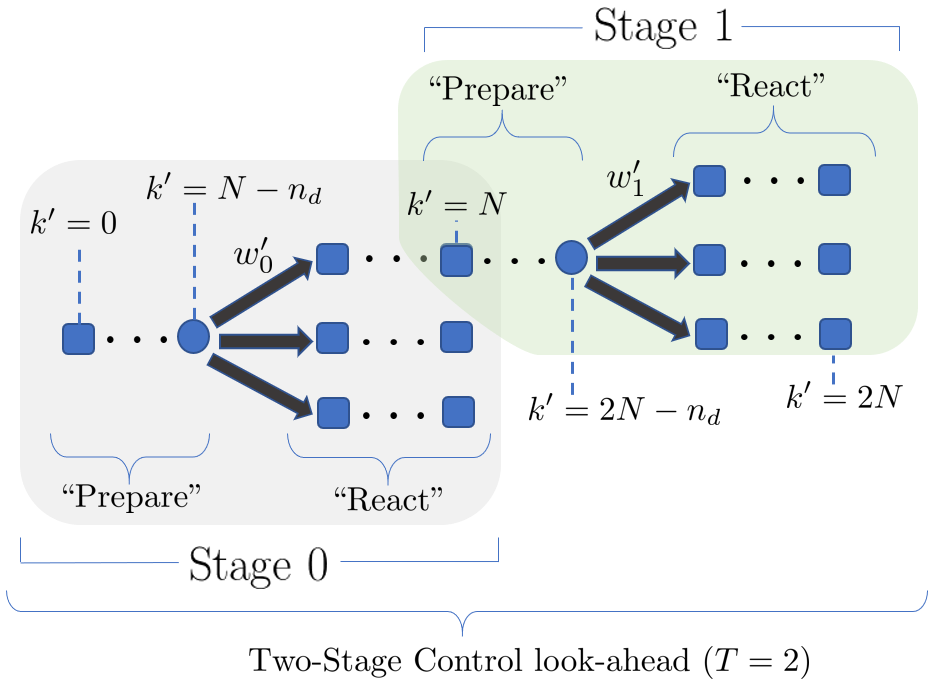}
		\caption{Multi-stage scenario tree for the prepare-react multi-step problem at time $k$ (indexed internally using $k'$ for simplicity). The disturbance is sampled every $N$ steps. The control look-ahead consists of multiple nested branches of ``prepare" and ``react" sequences \revision{(indexed as ``stages'' in the figure above)}; shaded green in the figure above is one such nested branch corresponding to $w_0' = w^{[1]}$. To evaluate costs over stage 1, it is assumed that the expert is \revision{using} the conditional CRM $\rho_1(\cdot)$ where for each realization of $x_{N|k}$ (identified uniquely by the observed disturbance branch at stage $0$), the expert uses the static CRM $\rho(\cdot)$ over the nested outcomes (shown in green for one possible realization of $x_{N|k}$). The observed control sequence is the beginning ``prepare'' -- ``react'' sequence corresponding to the actual realized disturbance $w^*_{0|k}$. }	\label{fig:multi_pr_look} 
\end{figure}

Let $x_{k'|k}$ denote the predicted state for time-step $k+k'$, where $k' \in [0,\, TN-1]$, as predicted at time-step $k$ within the expert's multi-step optimization problem. Similarly, let $\{w'_t\}, t \in [0,T-1]$ represent the predicted disturbance sequence, \revision{where each $w_{t}' \in \{1,\ldots,L\}$}. \revision{Let $\hat{\pi}_t(\bm{\omega}_{t-1}, w'_{t}),\ t\in [0,T-1]$ denote the expert's {\em ``prepare'' -- ``react'' control policy} for stage $t$ (corresponding to time-steps $k' \in [tN, \, (t+1)N-1]$), \revision{where we make explicit the dependency on} the partial predicted disturbance history $\bm{\omega}_{t-1} := \{w_{-1}', w_0', \ldots, w_{t-1}' \}$ and the next predicted disturbance mode $w'_{t}$, while we omit in the interest of brevity the \revision{dependency} on $x_k$ and partial policy history $\{\hat{\pi}_0,\ldots,\hat{\pi}_{t-1}\}$}. We take $w'_{-1} =: w^*_{-1|k}$ to represent the actual disturbance mode in progress at the time of solving the multi-stage optimization problem at time-step $k$. Note that, by causality, only the ``react'' portion of $\hat{\pi}_t$ may be a function of $w_{t}'$ but not the ``prepare'' portion. Finally, \revision{let} $C_{tN:(t+1)N-1}(x_{tN|k},\hat{\pi}_t(\cdot))$ \revision{denote} the accumulated cost \revision{(i.e., a random variable adapted to the filtration generated by the sequence $\{w_{t}'\}$)} over time-steps $k' \in [tN, (t+1)N-1]$ given the ``prepare'' -- ``react'' control policy $\hat{\pi}_t$ \revision{at stage $t$}, \revision{that is, 
\[
	C_{tN:(t+1)N-1}(x_{tN|k},\hat{\pi}_t(\cdot)) = \sum_{k'=tN}^{(t+1)N-1} C( x_{k'|k}, \hat{\pi}_t(x_{k'|k})).
\]
As in Section~\ref{sec:single period unknown cost}, we assume that the time-step cost $C(x,u)$ is represented by a linear combination $c^T \phi(x,u)$ of features $\phi(x,u)$.} The expert's multi-step optimization problem is then given as: 
\begin{equation}
\min\limits_{\substack{\hat{\pi}_t \\ t \in [0,T-1]}}  \rho_0\bigg(C_{0:N-1}(\cdot,\hat{\pi}_0) + \rho_1 \big( C_{N:2N-1}(\cdot, \hat{\pi}_1) + \cdots + \rho_{T-1}( C_{(T-1)N:TN-1}( \cdot,\hat{\pi}_{T-1}) ) \big) \cdots \bigg),
\label{dyn_cost_pr}
\end{equation}
where each $\rho_t,\ t = 0,\ldots,T-1$ is a coherent one-step \emph{conditional} risk measure such that each \emph{component} of $\rho_t$ is a CRM $\rho(\cdot)$ with respect to the probability space $(\W, 2^{\W}, p)$ and characterized by \emph{the fixed risk envelope} $\Pp \subseteq \Delta^L$. Leveraging the translational invariance property, the objective may be equivalently re-written as 
\[\small
\begin{split}
& \varrho(C_{0:TN-1}) := \rho_0 \circ \cdots \circ \rho_{T-1} ( C_{0:TN-1} ) \\
= &C_{0:N-n_d} + \rho_0\bigg( C_{N-n_d+1:N-1} + C_{N:2N-n_d} + \rho_1 \big( C_{2N-n_d+1:2N-1} + \cdots + \rho_{T-1}\left( C_{TN-n_d+1:TN-1} \right) \cdots \big) \bigg).
 \end{split}
\]
One should notice that (1) for each $t \in \{0,\ldots, T-1\}$, the cost sequence $C_{tN:(t+1)N-1}$ is split across the risk operator $\rho_t$ due to the ``prepare''--``react'' structure and the translational invariance property, (2) the risk mapping is over accumulated costs, \revision{i.e., in the notation of eq.~\eqref{dyn_risk}, the stage $t$ random cost $Z_t$ corresponds to the accumulated cost $C_{tN:(t+1)N-1}$} since disturbances are sampled every $N$ steps, \revision{and (3) since problem~\eqref{dyn_cost_pr} is solved in receding horizon fashion and thus $x_k$ is \emph{known} at time $k$, $\varrho(\cdot)$ is a real valued \emph{function}}. The observed input from the expert is the \revision{first stage optimal} ``prepare'' -- ``react'' control policy $\hat{\pi}^*_0 (w^*_{-1|k}, w^*_{0|k})$ where $w^*_{0|k}$ represents the actual disturbance mode sampled after time-step $k$, following which the expert re-solves the problem at time $k+N$, \revision{with a receded horizon up to time-step $k+N+TN$}.

Notice that by setting $T=1$, one recovers the single-stage ``prepare -- react" model presented in~\citep{MajumdarSinghEtAl2017}. \revision{The strategy in~\citep{MajumdarSinghEtAl2017} is to reduce} the multi-step inference problem \revision{to a mathematically equivalent  single-step problem} by \revision{estimating} the (un-observed) control policies of the human agent, corresponding to the \emph{un-realized} disturbance branches. \revision{Specifically}, consider the scenario tree decomposition in Figure~\ref{fig:multi_look}. If disturbance $w^{[3]}$ is realized at time-step $k+N-n_d$, then we only observe the ``react'' control sequence corresponding to the third branch. The algorithm in~\citep{MajumdarSinghEtAl2017} proceeded by first inferring the ``react'' control sequences for the un-observed branches and then constructing a bounding hyperplane using a similar version of problem~\eqref{eq:tau bound known cost}. In a multiple-stage setting, however, it is exceedingly difficult to exactly infer (or approximate) the unobserved control policies as each of these policies involves an \emph{unobserved} nested optimization over future branching events. Consequently, the optimality conditions of an observed control policy are defined by equalities that are non-linear in the unobserved variables. Therefore, extending the use of KKT conditions to infer an outer approximation of the risk envelope in the style of Theorem~\ref{thm:bound polytope} leads to an \emph{intractable} non-convex optimization problem. To address this fundamental observability issue, we introduce a \emph{semi-parametric} representation of the CRM, discussed next. 

\subsubsection{Semi-Parametric CRM}
Fix a set of $M$ normal vectors $a_j \in \RR^{L}, j = 1,\ldots, M$. \revision{Let $\Pp_r$ denote the polytope defined by the halfspace constraints}
\begin{equation}
	\Pp_r := \{ v \in \Delta^L  \ | \  a_j^T v \leq b(j) - r(j),\quad  j = 1,\ldots,M\},
\label{P_param}
\end{equation}
where for each $j$, $b(j):= \max_{v\in \Delta^L} a_j^T v$, \revision{and $r$ is a parameter vector in $\RR^{M}$}. The CRM \revision{with risk envelope $\Pp_r$ is denoted as $\rho^r(\cdot)$; explicitly,}
\begin{equation}
	\rho^r(Z) = \max_{v\in  \Pp_r} \mathbb{E}_v [Z],
\label{rho_sp}
\end{equation}
where $Z \in \RR^L$ is a discrete random variable with $L$ possible realizations (see Figure~\ref{fig:semi_CRM}). 
\begin{figure}[h]
\centering
	\includegraphics[width=0.5\textwidth]{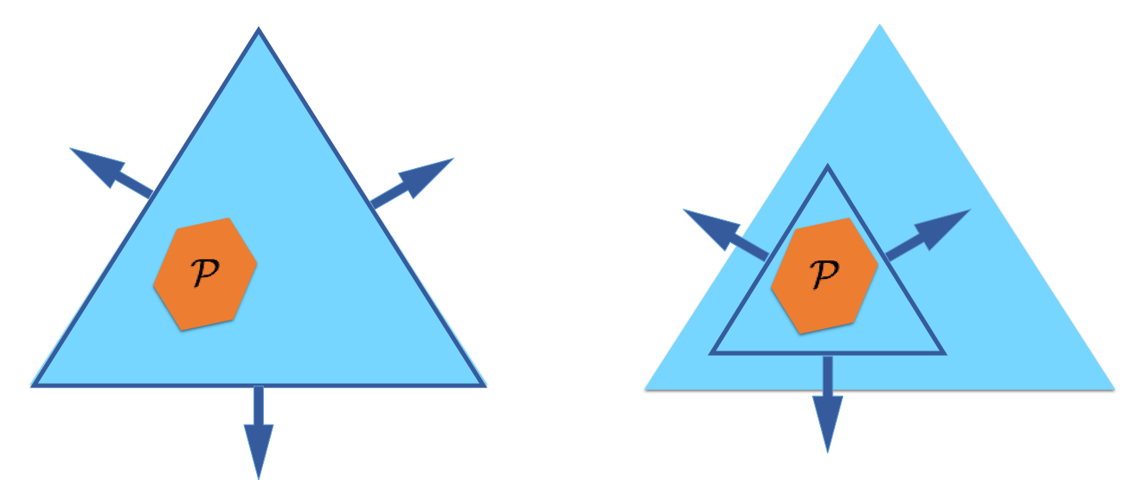}
	\caption{Schematic illustration of a semi-parametric CRM for a 3-scenario outcome space. The true risk envelope $\Pp$ is shown in orange; the \revision{boundary of the approximation polytope $\Pp_r$} is shown in dark blue. Left: the polytope $\Pp_r$ with $r = 0$. Right: the polytope $\Pp_r$ for some $r \in \RR_{>0}^M$. The arrows denote the a priori fixed normal vectors $\{a_j\}_{j=1}^3$. }
\label{fig:semi_CRM}
\end{figure}

This induced CRM is termed \emph{semi}-parametric since unlike methods where one seeks to find the parameters defining a \emph{fixed} disutility function (e.g., \cite{ShenTobiaEtAl2014,RatliffMazumdar2017}), here we \emph{do not} assume a fixed chosen risk measure. Instead, by parameterizing the risk measure in the dual space (via its risk envelope characterization), we retain the generality to recover any polytopic CRM, given a sufficient number of normal vectors $a_j$. A potential method to choose the normal vectors $a_j$ is to take the halfplane normals from the multi-step KKT method described in~\citep{MajumdarSinghEtAl2017}.

\revision{In order to ensure that the polytope $\Pp_r$ defined in eq.~\eqref{P_param} is non-empty, we} define the extended polytope
\[
	\tilde{\Pp}_r := \{ (v,r) \in \Delta^L \times \RR^{M}  \ | \   a^T_j v  + r(j) \leq b(j), \quad j = 1,\ldots,M\}.
\]
Define $\mathcal{R} := \mathrm{proj}_{r} \tilde{\Pp}_r$ as the projection of \revision{polytope $\tilde{\Pp}_r$}  along the $r$ variables. Then, $r\in \mathcal{R}$ ensures that the polytope $\Pp_r$ is non-empty. It is readily observed that the set $\mathcal{R}$ is also a polytope.

\subsubsection{Constrained Maximum Likelihood}

Given the semi-parametric representation of the CRM in~\eqref{rho_sp}, the RS-IRL problem reduces to inference over the offset vector $r$, and cost feature weight vector $c$. We will perform this inference using a constrained maximum likelihood model. Consider, first, \revision{a likelihood model inspired by the MaxEnt IRL framework~\citep{ZiebartMaasEtAl2008} where we assume
\begin{equation}
	\mathrm{Pr} (\hat{\pi}_0 (w^{*}_{-1|k}, w^{*}_{0|k})) \propto \exp\left( - \tau [w^*_{-1|k}, w^*_{0|k}] \right),
\label{like_cont}
\end{equation}
where, $\tau [w^*_{-1|k}, w^*_{0|k}] $ is the optimal value of~\eqref{dyn_cost_pr}}, computed using the semi-parametric CRM defined in~\eqref{rho_sp} and  \emph{conditioned} on $w'_0 = w^*_{0|k}$ and $\hat{\pi}_0 = \hat{\pi}_0 (w^{*}_{-1|k}, w^*_{0|k})$ \revision{(see Appendix~\ref{app:full_like} for a detailed derivation of this distribution)}. While the original MaxEnt IRL model is motivated by finding the maximum entropy distribution subject to an expected feature matching constraint, the robust performance of MaxEnt IRL even in the absence of such a statistical motivation has been extensively observed and leveraged in the IRL literature, \revision{particularly in the context of noisy or suboptimal demonstrations}.

A key limitation of the MaxEnt model, however, \revision{lies in the complexity of estimating the partition function (normalization factor for the distribution in eq.~\eqref{like_cont})} and its gradients. \revision{The likelihood model in eq.~\eqref{like_cont} represents} a distribution over all possible $N$-length \emph{policies}. This makes sampling-based approximations intractable, as similarly observed in~\citep{KretzschmarSpiesEtAl2016}, and Laplace integral-based approximations as used in~\citep{LevineKoltun2012} too imprecise. 

In order to construct a tractable algorithm, we employ the simplification whereby at the beginning of any ``prepare'' stage (see Figure~\ref{fig:multi_pr_look}), the expert \revision{can} only choose an \emph{open-loop} control \emph{trajectory} $\hat{u}$ of length $N$ from a finite set of such trajectories $\prod$, thereby eliminating the notion of a ``react'' \emph{policy} and replacing it with an open-loop sequence spanning the entire ``prepare'' -- ``react'' stage.  The set $\prod$ of these trajectories can be chosen for instance by running the K-Means clustering algorithm on the raw input trajectories. This simplification allows us to interpret problem~\eqref{dyn_cost_pr} as a game between the expert with action set $\prod$ and nature with action set $\W$, and uniquely identify any predicted state $x_{tN|k}, t \in [1,T-1]$ using the predicted game history, i.e., disturbance history $\bm{\omega}_{t-1}$ and control history $\bm{u}_{t-1} := \{\hat{u}_0,\ldots,\hat{u}_{t-1}\}$.  Leveraging such a discrete representation and dynamic programming, \revision{one can then} construct the optimal solution to the expert's multi-stage optimization problem using a ``risk-sensitive" Bellman recursion, defined below.
\vspace{\baselineskip}

\noindent \emph{Terminal Stage:} For all possible game histories at stage $T-1$, define
\begin{align*}
	\tau[\bm{u}_{T-2}, \bm{\omega}_{T-2}](\hat{u}) &:= \rho \left(C_{(T-1)N:TN-1}(x_{(T-1)N|k}, \hat{u}) \right)  \\ 
	\hat{\pi}^*_{T-1} [\bm{u}_{T-2}, \bm{\omega}_{T-2}] &:= \argmin_{\hat{u} \in \prod} \tau[\bm{u}_{T-2}, \bm{\omega}_{T-2}](\hat{u}).
\end{align*}

\noindent \emph{Recursion:} For all possible game histories at stage $t$, for $t = T-2, \ldots, 1$:
\begin{align*}
	\tau[\bm{u}_{t-1}, \bm{\omega}_{t-1}](\hat{u}) &:= \rho \left(C_{tN:(t+1)N-1}(x_{tN|k},\hat{u}) +  \min_{\hat{u}'\in \prod} \tau [\{\bm{u}_{t-1},\hat{u}\}, \{\bm{\omega}_{t-1}, w_{t}' \}](\hat{u}') \right) \\ 
	\hat{\pi}^*_{t} [\bm{u}_{t-1}, \bm{\omega}_{t-1}]  &:= \argmin_{\hat{u} \in \prod} \tau[\bm{u}_{t-1}, \bm{\omega}_{t-1}](\hat{u}).
\end{align*}

\noindent \emph{First Stage:}
\begin{align*}
	\tau[w^*_{-1|k}](\hat{u}) &:= \rho\left(C_{0:N-1}(x_{k},\hat{u}) + \min_{\hat{u}' \in \prod} \tau [\{\hat{u}\}, \{w^*_{-1|k}, w_{0}' \}](\hat{u}') \right) \\ 
	\hat{\pi}^*_{0} [w^*_{-1|k}]  &:= \argmin_{\hat{u} \in \prod} \tau[w^*_{-1|k}]](\hat{u}).  
\end{align*}
The value $\min_{\hat{u} \in \prod} \tau[w^*_{-1|k}](\hat{u})$ is the optimal value of problem~\eqref{dyn_cost_pr}. In the equations above, it is understood that for each $t \in [0,T-1]$, the accumulated cost $C_{tN:(t+1)N-1}$ is evaluated based on the previous disturbance mode $w_{t-1}'$ for the first $N-n_d$ steps, followed by $w_{t}'$ for the remaining $n_d$ steps. 

Given the structure of the optimal solution of problem~\eqref{dyn_cost_pr}, presented in Bellman form above using the true CRM $\rho(\cdot)$, we now construct a \emph{computationally tractable} likelihood model for the parameters $r$ and $c$ by defining the \emph{soft} risk-sensitive Bellman recursion using the semi-parametric CRM $\rho^r(\cdot)$. For the terminal stage, define
\begin{equation}
	\tilde{\tau}[\bm{u}_{T-2}, \bm{\omega}_{T-2}](\hat{u}) := \rho^r \left(C_{(T-1)N:TN-1}(x_{(T-1)N|k}, \hat{u}) \right);
\label{tau_term_soft}
\end{equation}
for all $t = T-2,\ldots,1$, \revision{define}:
\begin{equation}
	\tilde{\tau}[\bm{u}_{t-1}, \bm{\omega}_{t-1}](\hat{u}) := \rho^r \left(C_{tN:(t+1)N-1}(x_{tN|k},\hat{u}) +  \softmin_{\hat{u}' \in \prod} \tilde{\tau} [\{\bm{u}_{t-1},\hat{u}\}, \{\bm{\omega}_{t-1}, w_{t}' \}](\hat{u}') \right);\label{tau_int_soft} 
\end{equation}
finally, for the first stage, \revision{define}:
\begin{equation}
	\tilde{\tau}[w^*_{-1|k}](\hat{u}) := \rho^r \left(C_{0:N-1}(x_{k},\hat{u}) + \softmin_{\hat{u}' \in \prod} \tilde{\tau} [\{\hat{u}\}, \{ w^*_{-1|k}, w_{0}' \}](\hat{u}') \right), \label{tau_soft} 
\end{equation}
where $\softmin_{x} f(x):= -\log \sum_x \exp(-f(x) )$. 

Let $\hat{u}^*_t$ be the closest (in $\mathcal{L}_2$ norm) trajectory in $\prod$ to the observed control sequence over time-steps $[tN,(t+1)N-1]$\footnote{We use $t$ here for notational consistency between the stage-wise decomposition of the multi-step problem and demonstrated action trajectories.}. Similar to the MaxEnt IRL approach, we allow for imperfect human demonstrations by postulating that lower risk-sensitive cost actions (i.e., $\tilde{\tau}[w_{-1|tN}^*](\hat{u})$) are exponentially preferred, i.e., 
\begin{equation}
	\mathrm{Pr} (\hat{u}) \propto \exp\left( -\beta \tilde{\tau}[w_{-1|tN}^*](\hat{u}) \right),
\label{boltz_def}
\end{equation}
where $\beta >0$ is an inverse temperature parameter. Thus, the likelihood of parameters $r,c$ is given by:
\begin{equation}
	l(r,c | \hat{u}^*_t) := \dfrac{\exp \left(-\beta \tilde{\tau}[w_{-1|tN}^*] (\hat{u}^*_t)\right)}{\sum_{\hat{u}} \exp\left(-\beta  \tilde{\tau}[w_{-1|tN}^*] (\hat{u})\right)}.
\label{like_discrete}
\end{equation}
As the expert is assumed to solve the problem in a receding horizon fashion, we may treat each $(w^*_{-1|tN}, \hat{u}^*_t)$ tuple in the demonstrated trajectory $\mathcal{T}^*$ independently. Consequently, the log likelihood given the entire trajectory is simply
\begin{equation}
	l(r,c | \mathcal{T}^*) = \dfrac{1}{|\mathcal{T}^*|} \sum_{\hat{u}^*_t \in \mathcal{T}^*} -\beta\tilde{\tau}[w_{-1|tN}^*] (\hat{u}^*_t) + \softmin_{\hat{u}} \beta \tilde{\tau}[w_{-1|tN}^*](\hat{u}),
\label{like_traj}
\end{equation}
where $|\mathcal{T}^*|$ is the number of $N$-step demonstrations in the trajectory $\mathcal{T}^*$. The inference problem is \revision{then}
\begin{equation}
	\{r^*, c^*\} := \argmax\limits_{c \in \Delta^H,\  r \in \mathcal{R}} \ l(r,c | \mathcal{T}^*),
\label{ML_prob}
\end{equation}
where, as before, we assume that the cost weights are non-negative and sum to one (and thus lie in the simplex $\Delta^H$). We solve the problem using projected gradient descent on $r$ and entropic mirror descent on $c$. The gradient formulas are derived by propagating gradients of the $\tilde{\tau}$ variables in recursive fashion from the terminal to the first stage (similar to \revision{the} computation of $\tilde{\tau}$ itself) and leveraging LP sensitivity results. For ease of exposition, we provide these recursive formulas in Appendix~\ref{app:gradients}.

\begin{remark}
While we lose the outer-approximation of the risk envelope and convergence guarantees associated with the KKT method, in its place we obtain a tractable algorithm that enables us to accommodate a substantially larger class of dynamic decision-making inference problems. Experimental results, as discussed in Section~\ref{sec:results}, confirm that the method works well in approximating a wide range of risk profiles.
\end{remark}

\section{Example: Driving Game Scenario}
\label{sec:results}

\revision{In this section} we apply our RS-IRL framework on a simulated driving game (Figure \ref{fig:visualization})  with ten human participants to demonstrate that our approach is able to infer individuals' varying attitudes toward risk and mimic the resulting driving styles. In particular, we note that the experimental setting here constitutes a significantly more challenging and dynamic testbed than typical benchmark examples such as grid-world or sequential investment tasks.

\subsection{Experimental Setting}

The setting consists of a leader car and a follower car, simulated in the commercial driving simulator Vires VTD~\citep{VIRESSGH}. Participants controlled the follower car with the Logitech G29 control suite, consisting of a steering wheel and pedals (Figure \ref{fig:visualization}). The follower car is modeled using the simple car model with states: $x_f$ (along-track position), $y_f$ (lateral position), $v_f$ (speed), $\theta_f$ (yaw angle) and $\delta_f$ (steering angle). The dynamics are given by:
\begin{equation}
\label{eq: follower's dynamics}
\dot{x}_f = v_f \cos(\theta_f), \quad \dot{y}_f = v_f\sin(\theta_f), \quad \dot{v}_f = u_a, \quad \dot{\theta}_f = -\frac{v_f}{l} \tan(\delta_f), \quad \dot{\delta_f} = u_s,
\end{equation}
where $u_a$ and $u_s$ are, respectively, the longitudinal acceleration and the steering rate inputs, and $l = 3.476$ m. The leader car plays the role of an ``erratic driver" and is modeled with double integrator dynamics along-track and triple integrator dynamics in the lateral direction to mimic continuous steering inputs. The state of the leader's car is described by $x_l$ (along-track position), $y_l$ (lateral position), $v_{x,l}$ (forward speed), $v_{y,l}$ (lateral speed) and $a_y$ (lateral acceleration). The dynamics are given by:
\begin{equation}
\label{eq: leader's dynamics}
\dot{x} = v_x, \quad \dot{v}_x = w_x, \quad \dot{y} = v_y, \quad \dot{v}_y = a_y, \quad \dot{a}_y = w_y,
\end{equation}
where $[w_x,w_y]^T$ is the leader's control input. We simulate this system in discrete time at $60$ Hz and analyze the data with a time step $\Delta t = 0.1$ s. 

In this setting, a disturbance $w^{[i]}$ corresponds to a sequence of control inputs executed by the leader car $w^{[i]} := \{ (w_x, w_y)^{[i]}_k \}_{k=1}^N$ over $N$ time steps. Each disturbance is sampled from a finite set $\W = \{w^{[1]},\ldots,w^{[L]}\}$ with $L = 4$. These ``disturbance" realizations correspond to different maneuvers for the leader (doing nothing, accelerating, decelerating, and swapping lanes) and are generated randomly according to the pmf $p = [0.3, 0.3, 0.3, 0.1]$.

The disturbance is sampled every $N = 15$ time steps. Thus, the leader car can be viewed as executing a random \emph{maneuver} every 1.5 seconds. The whole system is described by the state: 
\begin{equation}
\label{eq: system's state}
\xi := \left[ x_f, y_f, \theta_f, v_f, \delta_f, x_l, v_{x,l}, y_l, v_{y,l}, a_{y,l}\right]^T.
\end{equation}
Participants in the study were informed that their goal was to follow the leader car (described as an ``erratic driver"), as closely as possible in the $x$ and $y$ directions, while staying behind the leader and avoiding any collision. The leader car's four maneuvers were described to participants, along with the fact that these sequences of actions are generated every 1.5 s, \emph{independent} of (as opposed to an interactive game)  the participant's actions and position. 

The experimental protocol for each participant consisted of three phases. The first phase ($\sim 1$ minute) was meant for the participant to familiarize themselves with the simulator. The second and third phases (one minute each) involved the leader car acting according to the model described above (with actions being sampled according to the pmf $p$). The data collected during the second phase was used to train the model and data collected during the third phase was used to test it (the second and third phase disturbance sequences were kept same for all participants). 

Note that the pmf $p$ is \emph{not} shared with the participants. This experimental setting may thus be considered \emph{ambiguous}. However, since participants are exposed to a training phase where they may build a mental model of disturbances, the setting may also be interpreted as one involving risk.

While the ``game'' setting is identical to the one introduced in our earlier work in~\citep{MajumdarSinghEtAl2017}, the use of non-linear dynamics and a realistic driving simulator as opposed to the first-order integrator MATLAB game in~\citep{MajumdarSinghEtAl2017} lends the experiment more realism. All data, algorithm, and plotting code is made available at \url{https://github.com/StanfordASL/RSIRL}.

\subsection{Modeling and Implementation}

We modeled participants' behavior using the multi-stage ``prepare''--``react" framework presented in Section \ref{sec:multi period} with the ``prepare"  phase starting 0.7 seconds before the leader's action is sampled. The ``react" phase thus extends to 0.8 seconds after the disturbance. This parameter was chosen as being roughly reflective of observed participant behavior during the \emph{training} phase. We use $T=2$ decision stages to model the participants. Hence, the planning horizon is $NT=30$, which involves planning over two \revision{disturbance branching events} in a receding horizon fashion.

We represent the cost function as a linear combination of the following features (with unknown weights):
\begin{itemize}
\item $\phi_1 = \textbf{1}_{x_{\text{rel}} < 2.5}[\log(1+e^{-r_1 (x_{\text{rel}}-2.5)}) - \log(2)]$, 
\item $\phi_2 = \textbf{1}_{x_{\text{rel}} > 2.5} [\log(1+e^{r_2 (x_{\mathrm{rel}}-2.5)}) - \log(2)]$, 
\item $\phi_3 = \log(1+e^{r_3 |v_{x,\text{rel}}|}) - \log(2)$, 
\item $\phi_4 = r_4 \sum_{k=2}^N (u_{a,k} - u_{a,k-1})^2$, 
\item $\phi_5 = \log(1 + e^{r_5 |y_{\mathrm{rel}}|}) - \log(2)$,
\item $\phi_6 = \textbf{1}_{y_{f} > 2} [\log(1+e^{r_6 (y_{f}-2)}) - \log(2)] +  \textbf{1}_{y_{f} < -2} [\log(1+e^{-r_6 (y_{f}+2)}) - \log(2)]$,
\end{itemize}
where $x_{\text{rel}}$, $y_{\text{rel}}$, and $v_{x,\text{rel}}$ are respectively the relative along-track position, lateral position, and along-track velocity between the leader and the follower. Hence, the first feature translates the instruction of staying behind the leader; the second, third, and fifth features penalize the relative distance and velocity between the leader and follower; the fourth feature penalizes change in longitudinal acceleration (effectively jerk of the trajectory); the sixth feature penalizes crossing the road boundaries. We use $r_1=1$, $r_2=0.05$, $r_3=.1$, $r_4 = 1.0$, $r_5 = 0.1$, and $r_6 = 0.5$. These values were chosen to ensure that the costs were well conditioned over the usual range of relative states observed during the experiments.

In order to optimize the model by maximum likelihood estimation as described in Section \ref{sec:multi period}, we discretized the control space $[u_a, u_s]$ to generate the participant action space. For each participant, we ran the K-Means clustering algorithm on the training control inputs, and chose $15$ control trajectories for the first decision stage and $5$ trajectories for the second stage. Since we model each participant as planning over a receding horizon with two decision stages, it is reasonable to assume that the plan for the second stage is not as fine-grained as over the first one. In addition to reducing the computational burden, this concept of mixing coarse-fine planning is a feature also described in~\citep{CartonNitschEtAl2016} to model human locomotion. We observe that using $15$ control trajectories for the first stage and $5$ trajectories for the second stage was sufficient to generate a diverse expert action set in terms of accelerations (Figure \ref{fig:kmeans 15 acceleration}, \ref{fig:kmeans 5 acceleration}) and steering rates (Figure \ref{fig:kmeans 15 steering rate}, \ref{fig:kmeans 5 steering rate}); the span of all $x/y$ traces resulting from the combination of these first and second stage control trajectories is shown in Figure~\ref{fig:traj_spread}.

\begin{figure}[H]
\centering
\begin{subfigure}[t]{0.45\textwidth}
	\includegraphics[width=1\textwidth]{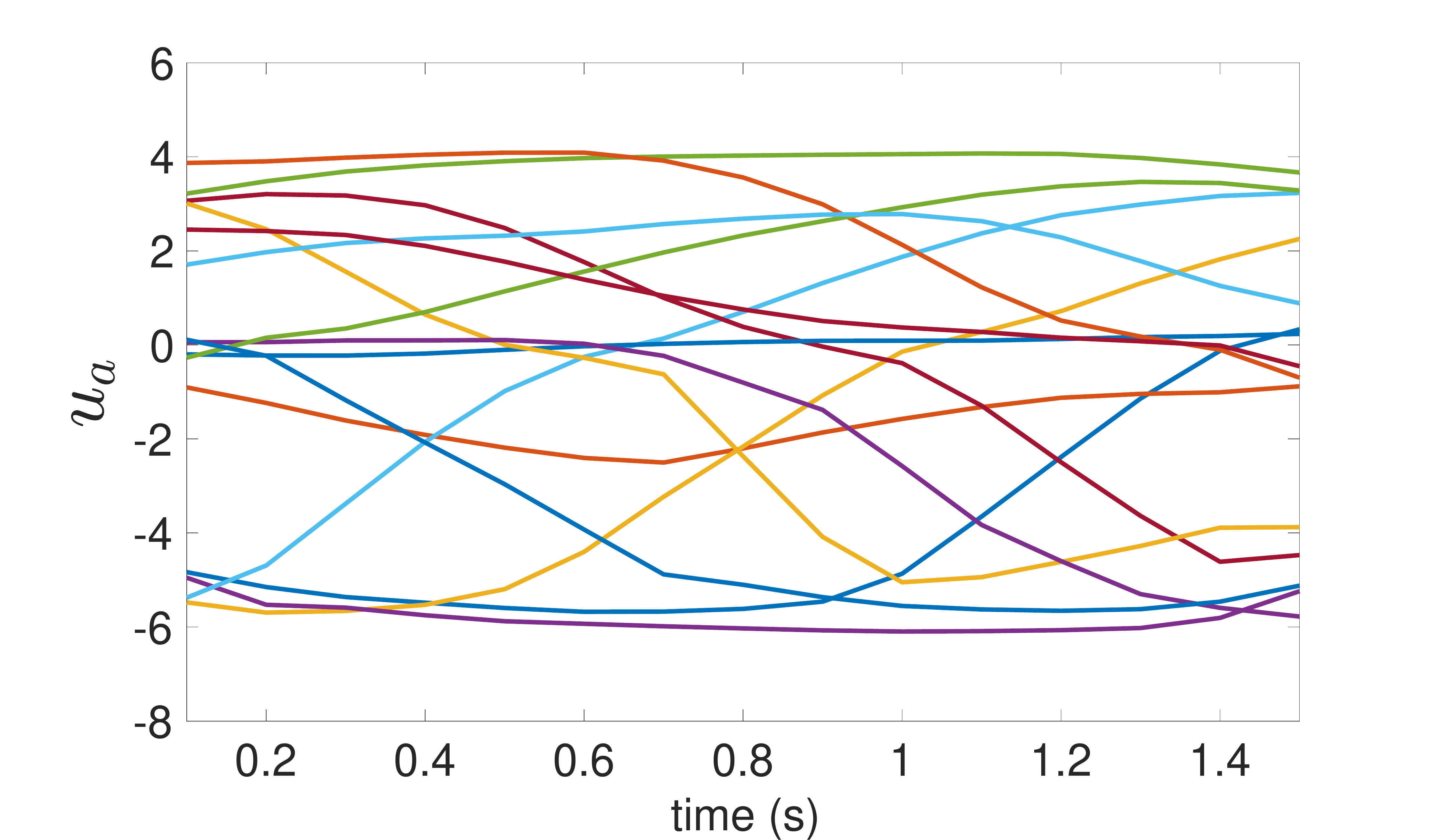}
	\caption{Acceleration}
	\label{fig:kmeans 15 acceleration} 
\end{subfigure}
\begin{subfigure}[t]{0.45\textwidth}
	\includegraphics[width=1\textwidth]{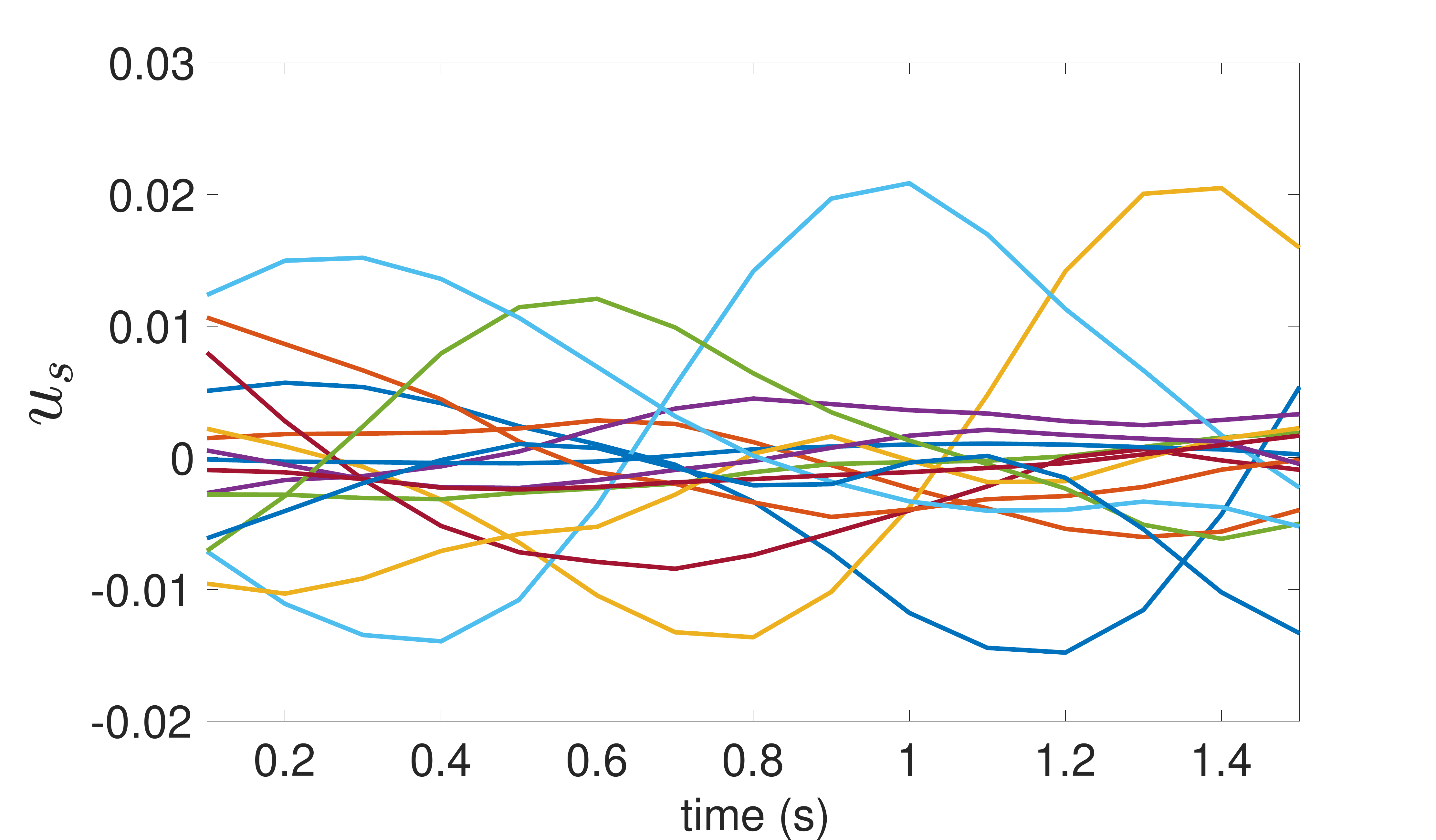}
	\caption{Steering rate}
	\label{fig:kmeans 15 steering rate} 
\end{subfigure}
\begin{subfigure}[t]{0.5\textwidth}
	\includegraphics[width=1\textwidth]{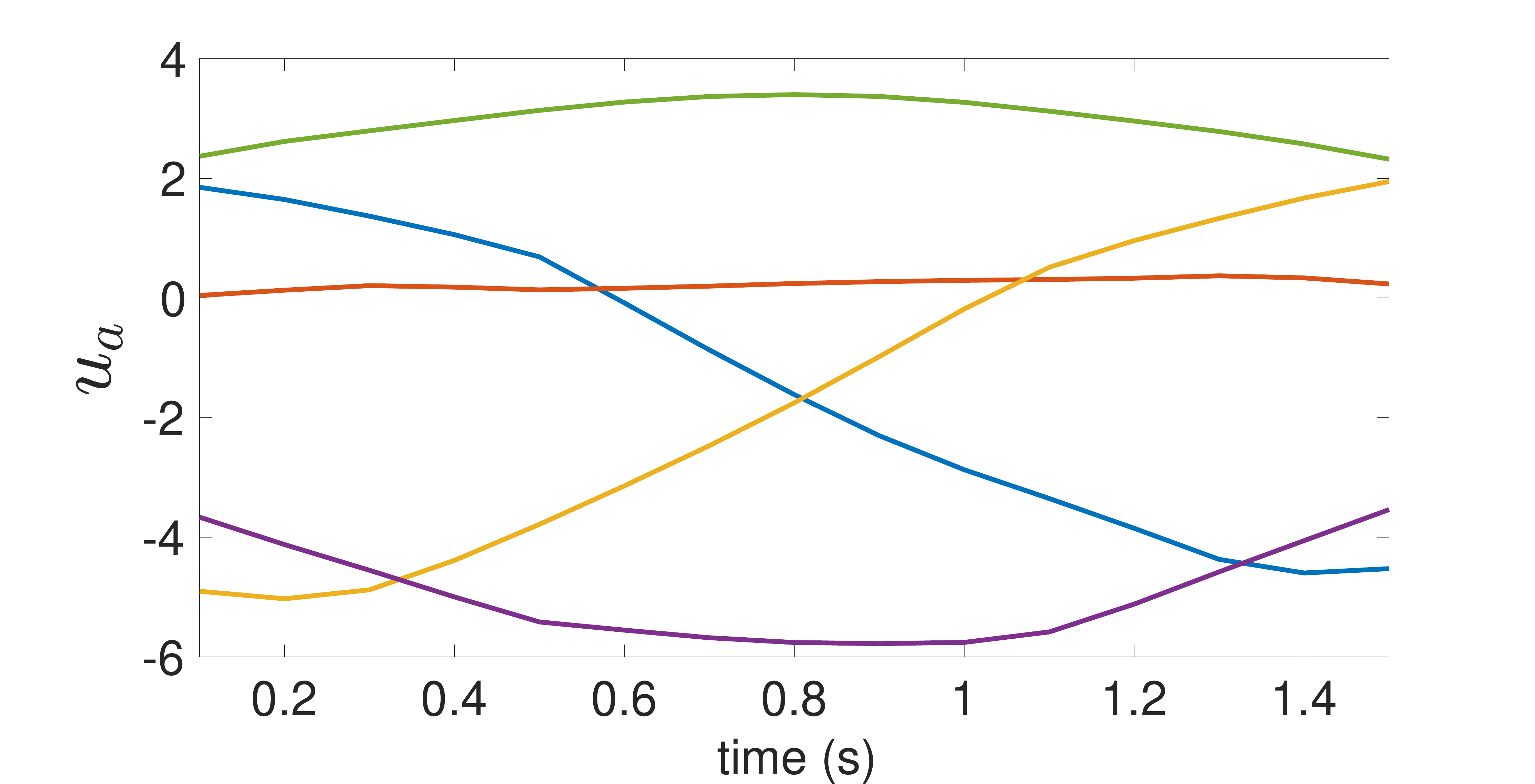}
	\caption{Acceleration}
	\label{fig:kmeans 5 acceleration} 
\end{subfigure}
\begin{subfigure}[t]{0.45\textwidth}
	\includegraphics[width=1\textwidth]{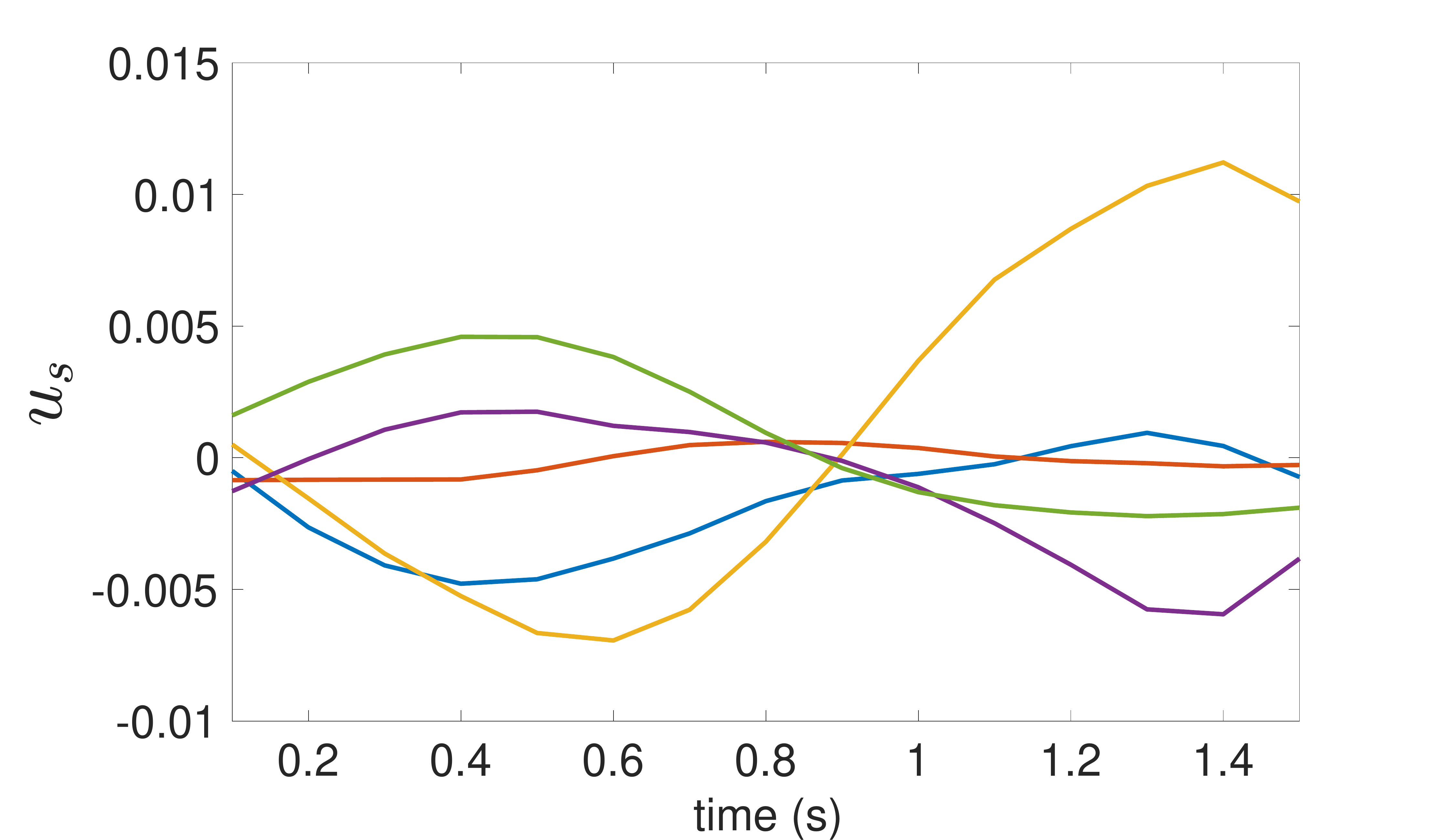}
	\caption{Steering rate}
	\label{fig:kmeans 5 steering rate} 
\end{subfigure}
\caption{Example of control trajectories computed by the K-Means algorithm using 15 centroids ((a), (b)) and 5 centroids ((c), (d)). Acceleration in m/s$^2$ and steering rate in rad/s.}
 \label{fig:kmeans second stage} 
\end{figure}

\begin{figure}[H]
\centering
	\includegraphics[width=0.8\textwidth]{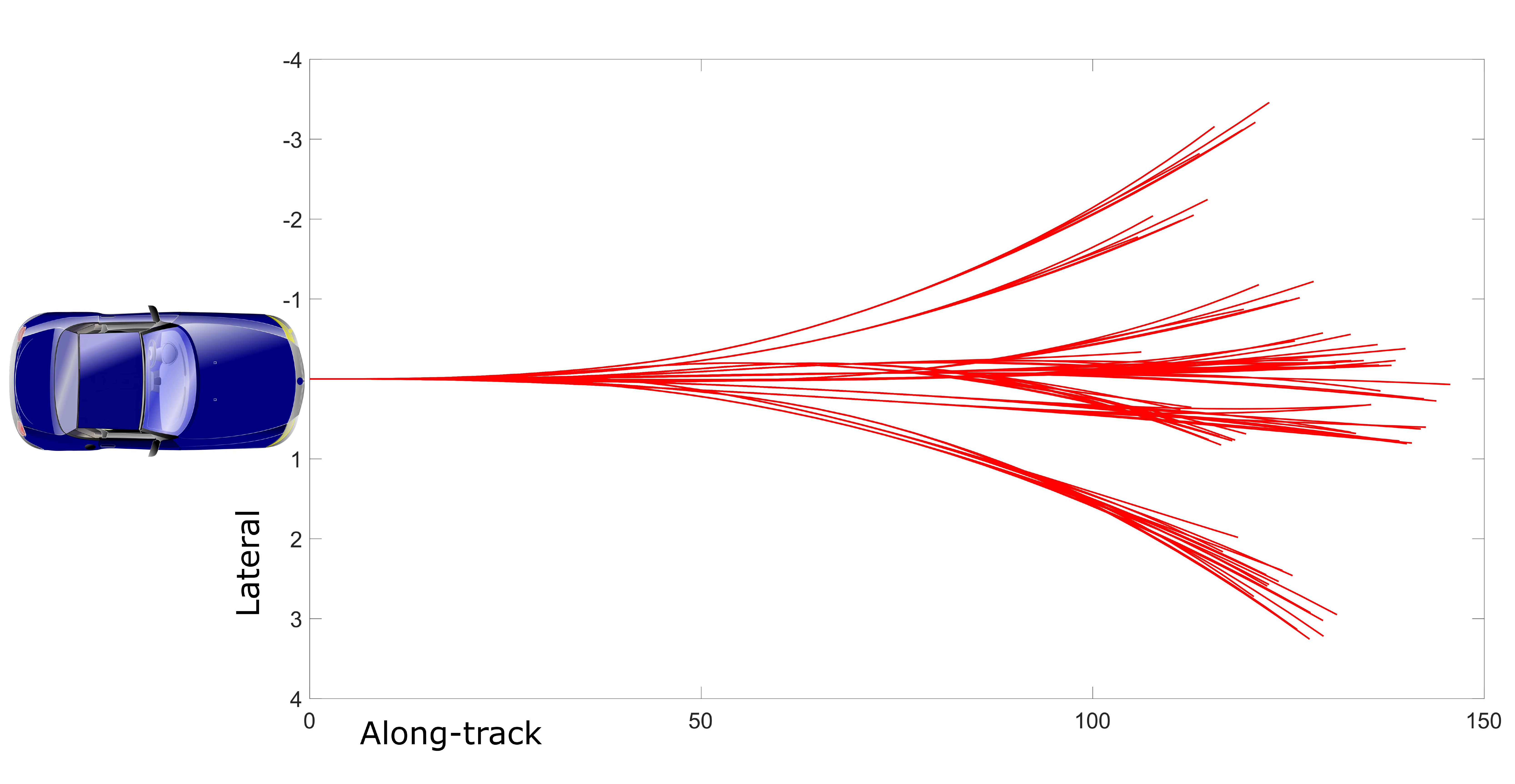}
\caption{Span of all possible along-track/lateral ($x/y$)  3 second trajectories encapsulated within a single 2-stage optimization problem with the $15/5$ discrete control trajectory space. All lengths in meters.}
 \label{fig:traj_spread} 
\end{figure}

The polytope $\mathcal{P}_r$ (alternatively, the semi-parameterized CRM $\rho^r(\cdot)$) was parametrized with 8 normal vectors $\{a_j\}_{j=1}^8$ corresponding to the positive and negative standard basis vectors in $\RR^4$ (i.e., $\{(e_i, -e_i)\}_{i=1}^{4}$). Our MATLAB implementation uses the parser YALMIP~\citep{Loefberg2004} and the solver Mosek~\citep{ApS2017}.

\subsection{Results}

Interestingly, our simulated driving scenario was rich enough to elicit a wide variety of qualitative behaviors from the ten participants. In particular, we observed two extreme policies. One extreme involved the driver following the leader very closely with a small separation (Figure~\ref{fig:full_x_ed}). Another extreme was to follow the leader with a distance large enough to avoid any collision, often decelerating preemptively to avoid such an event (Figure~\ref{fig:full_x_nina}). These two extremes can be interpreted as reflecting varying attitudes towards risk. The first policy corresponds to risk-neutral behavior, where the perceived (as captured by the inferred risk measure) probability of collision is lower than for highly risk-averse participants who were more sensitive to the worst-case eventuality (leader slowing down). We also observed a range of behaviors that lie between these two extremes. 

We compare the RS-IRL approach with one where the expert is modeled as minimizing the expected value of his/her cost function computed with respect to the pmf $p$, in \revision{a receding horizon fashion with} two decision stages. Similar to eq.~\eqref{boltz_def}, we assume that the risk-neutral stochastic policy is given by the Boltzmann distribution induced by the risk-neutral costs, thereby coinciding with the standard, MaxEnt IRL model and representing an important benchmark for comparison. We refer to this approach as risk-neutral IRL (RN-IRL). The analog of the recursion equations~\eqref{tau_term_soft}--\eqref{tau_soft} for the risk-neutral model are obtained by simply replacing the conditional risk measures with the expected value with respect to the pmf $p$. 

Since the expert is assumed to plan his/her decisions every 1.5 s in a receding horizon fashion, we evaluate RS-IRL and RN-IRL predictions based on their errors with respect to each 1.5 s observed demonstration in the test trajectory. In particular, we define the following error metric for predictions in $x_{\text{rel}} = x_l - x_f$:
\begin{equation}
	\Delta x_{\text{rel}, t} := \mathbb{E}\left[\sqrt{\sum_{k} (x_{\text{rel},k|t}^{\mathrm{predicted}} - x_{\text{rel},k|t}^{\mathrm{human}})^2 }\right],
	\label{err_metric}
\end{equation}
where $x_{\text{rel},k|t}^{\mathrm{predicted}}$ and $x_{\text{rel},k|t}^{\mathrm{human}}$ are, respectively, the predicted and actual $x_{\text{rel}}$ trajectories at time $k \in [tN,(t+1)N]$ corresponding to the $t^{\text{th}}$ 1.5 s segment in the demonstrated trajectory. The expectation is taken with respect to the stochastic policy (i.e., Boltzmann distribution) induced by the RS or RN costs (see eq.~\eqref{boltz_def} and~\eqref{boltz}). The errors $\Delta y_{\text{rel}}$, $\Delta v_{x,\text{rel}}$ and $\Delta v_{y,\text{rel}}$ are computed similarly. As RN-IRL consistently performed better with $T=2$ decision stages, we only present comparison results between RS-IRL and RN-IRL for $T=2$. To get a scale for the values reported in this section, the figure below illustrates the two cars almost colliding ($x_{\text{rel}} \approx 2.5$ m) and when they are $5$ m apart.

\begin{figure}[H]
\centering
\begin{subfigure}[t]{0.25\textwidth}
	\includegraphics[width=1\textwidth]{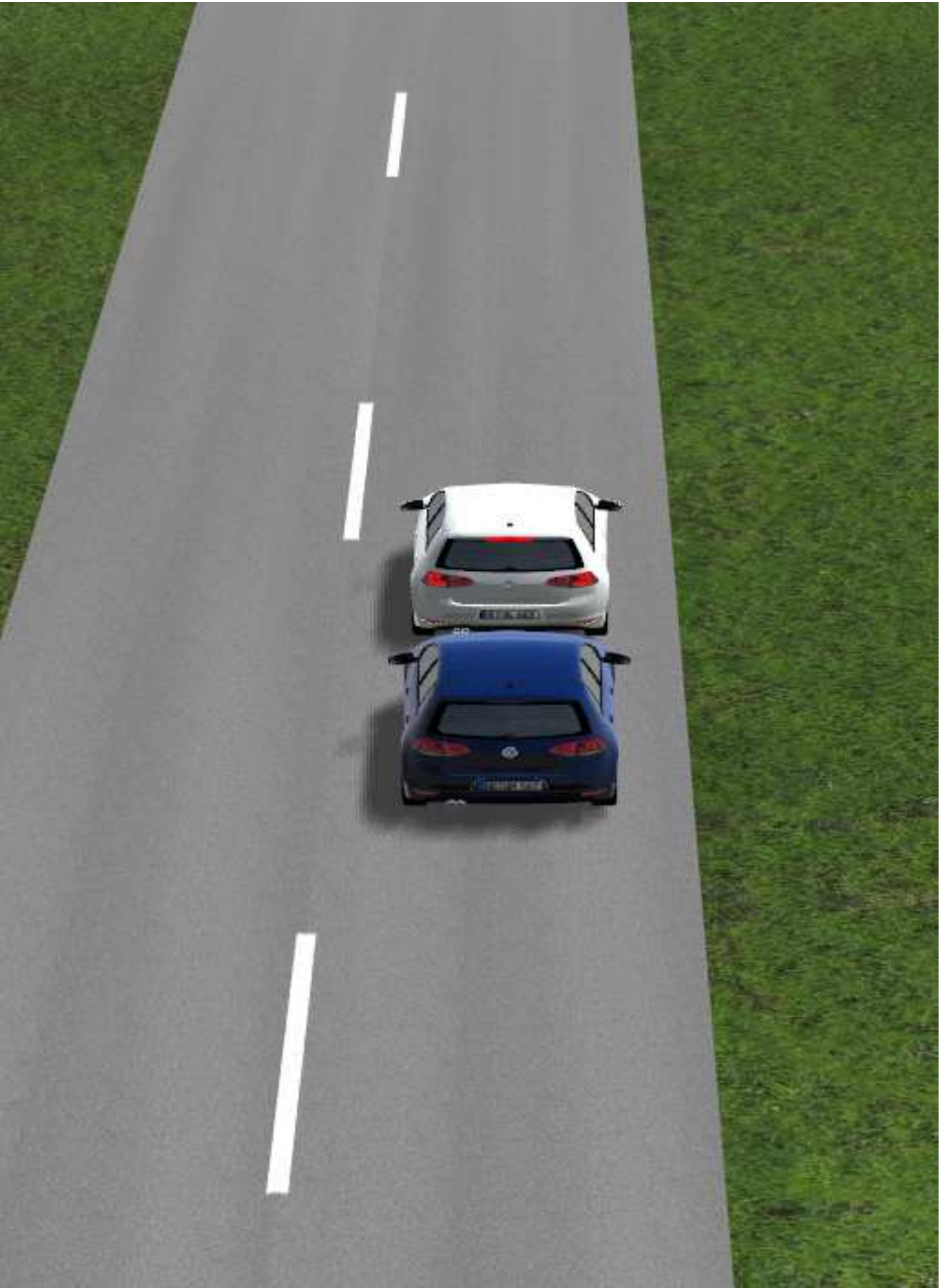}
		\label{fig:collision}
\end{subfigure}
\begin{subfigure}[t]{0.25\textwidth}
	\includegraphics[width=1\textwidth]{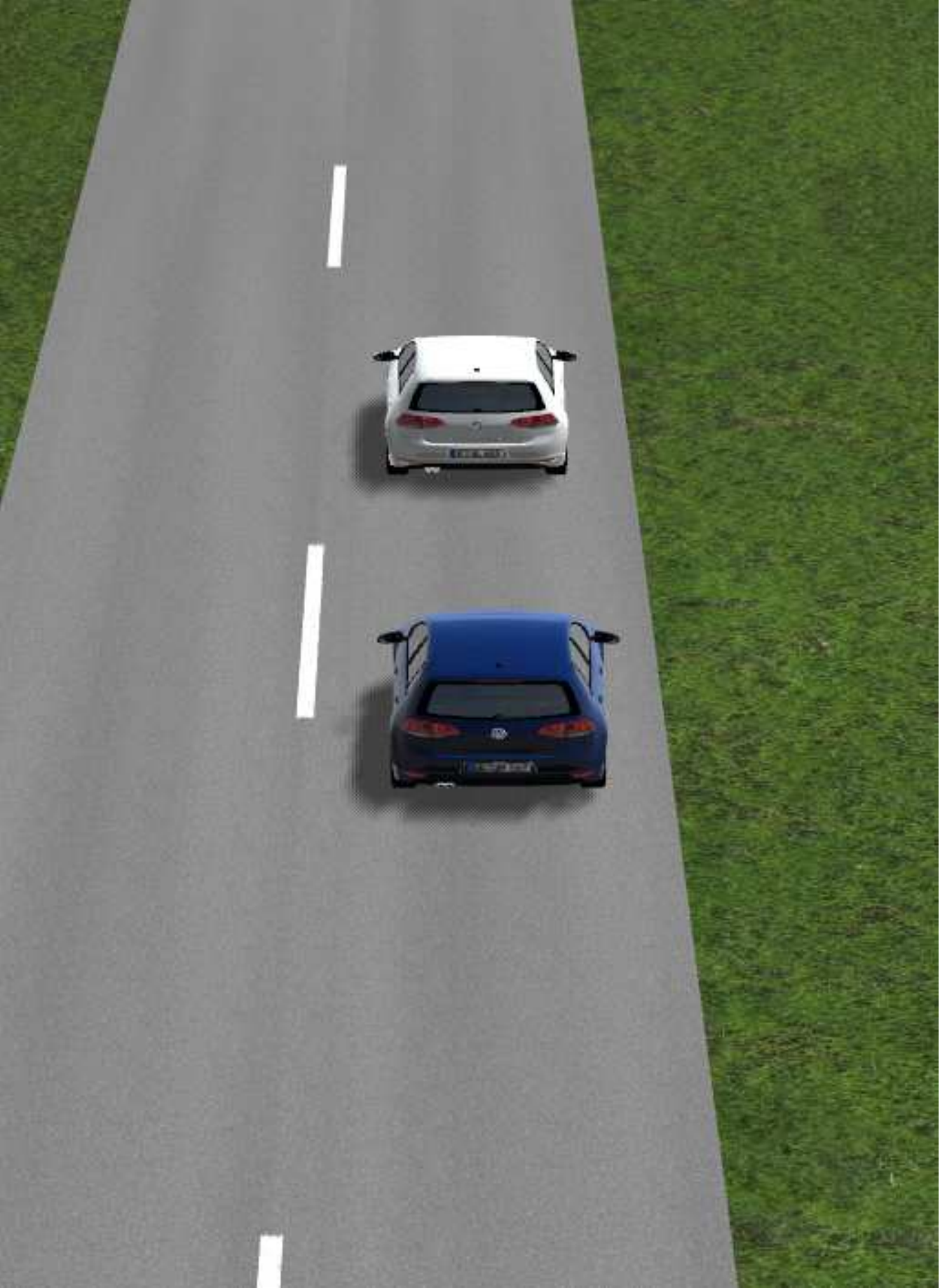}
	\label{fig:risk-averse typical distance}
\end{subfigure}
\caption{Left: Simulator visual when cars  are almost at collision distance ($x_{\text{rel}} \approx 2.5$ m); Right: Simulator visual of a participant driving 5 meters behind the leader. Lane-width: 3 m. }
\label{fig:risk averse participant}
\end{figure}

\subsubsection{Case Study \# 1: Risk-Averse Participant}

Figure~\ref{fig:full_x_nina} plots the $x_{\text{rel}}$ trajectory (normalized by car length $\approx 4.2$ m) during the third (test) phase for a \revision{participant inferred to be highly risk-averse}. On average, the along-track relative distance is quite large ($\approx 6$ car-lengths). The expected prediction errors $\Delta x_{\text{rel},t}$ (normalized by car-length) from RS-IRL and RN-IRL for each of the 51 1.5 s demonstrations comprising the test phase are plotted in Figure~\ref{fig:absolute_x_nina} as absolute errors, and in Figure~\ref{fig:difference_x_nina} as percentage differences with positive values indicating an improvement of RS-IRL over RN-IRL. A similar error plot is shown in Figure~\ref{fig:absolute_vx_nina} and~\ref{fig:difference_vx_nina} for along-track velocity error $\Delta v_{x,\text{rel},t}$.  

\begin{figure}[H] 
\centering
\includegraphics[width=0.55\textwidth]{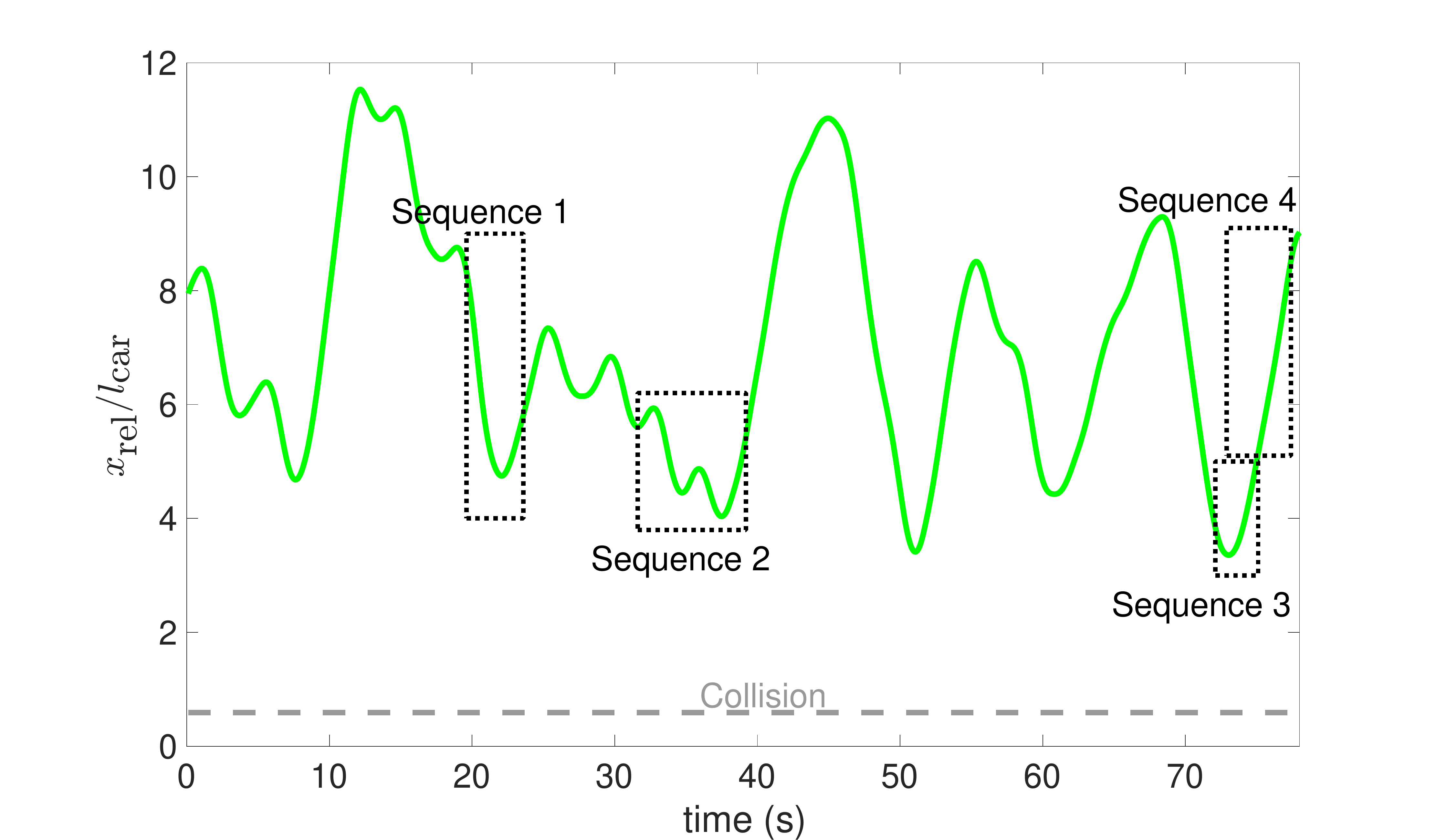}
\caption{Full $x_{\text{rel}}$ (longitudinal distance) trajectory (normalized by car length) for a highly risk-averse participant. On average, the relative distance is quite large ($\approx 6$ car-lengths). The boxed sections are discussed in further detail below.}
\label{fig:full_x_nina}
\end{figure}

\begin{figure}[H]
\centering
\begin{subfigure}[t]{0.49\textwidth}
\includegraphics[width=1\textwidth]{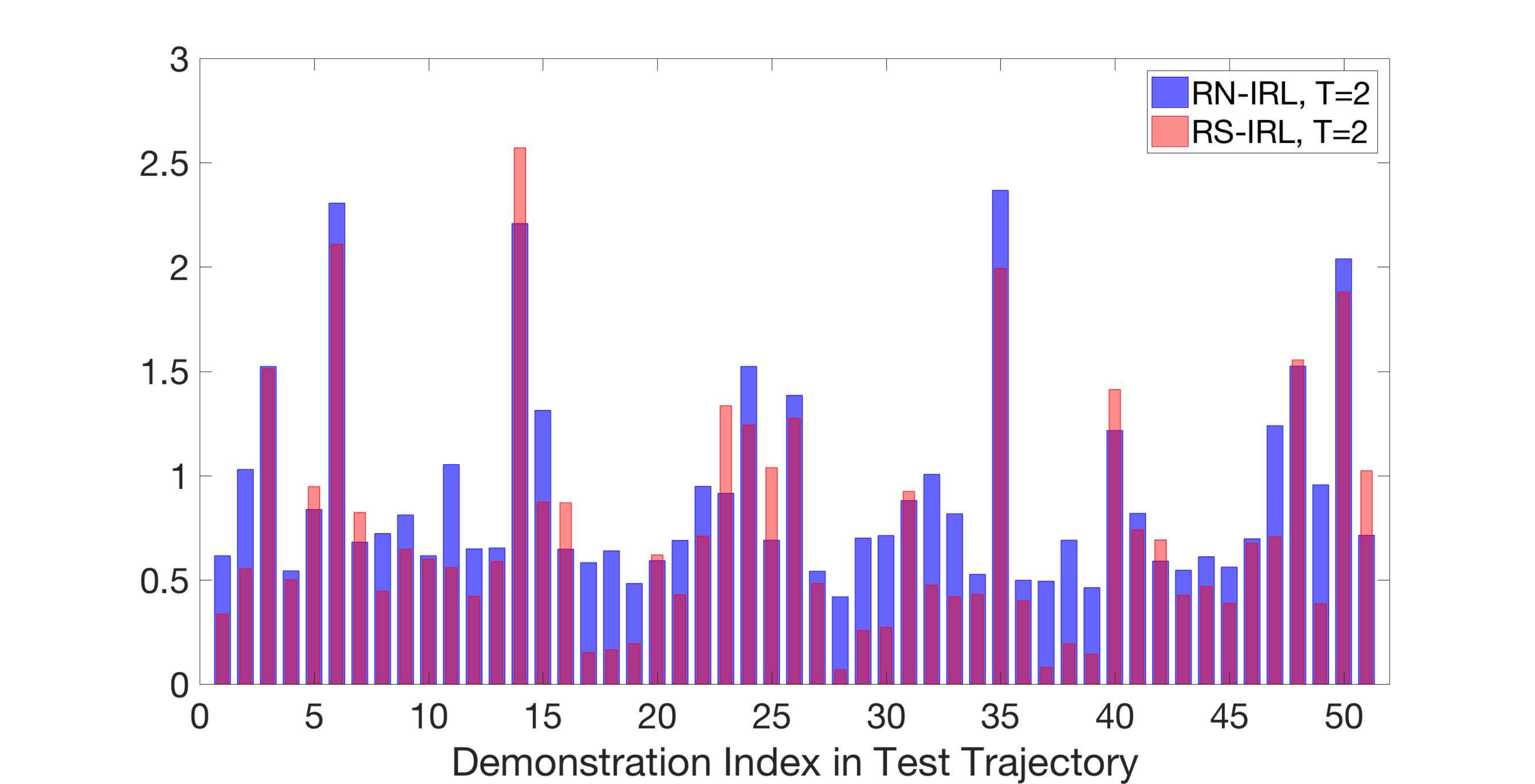}
\caption{Expected (w.r.t. stochastic policy) prediction errors $\Delta x_{\text{rel},t}$ from RS-IRL and RN-IRL for each 1.5 s trajectory segment.}
\label{fig:absolute_x_nina}
\end{subfigure}\ \ 
\begin{subfigure}[t]{0.49\textwidth}
\includegraphics[width=1\textwidth]{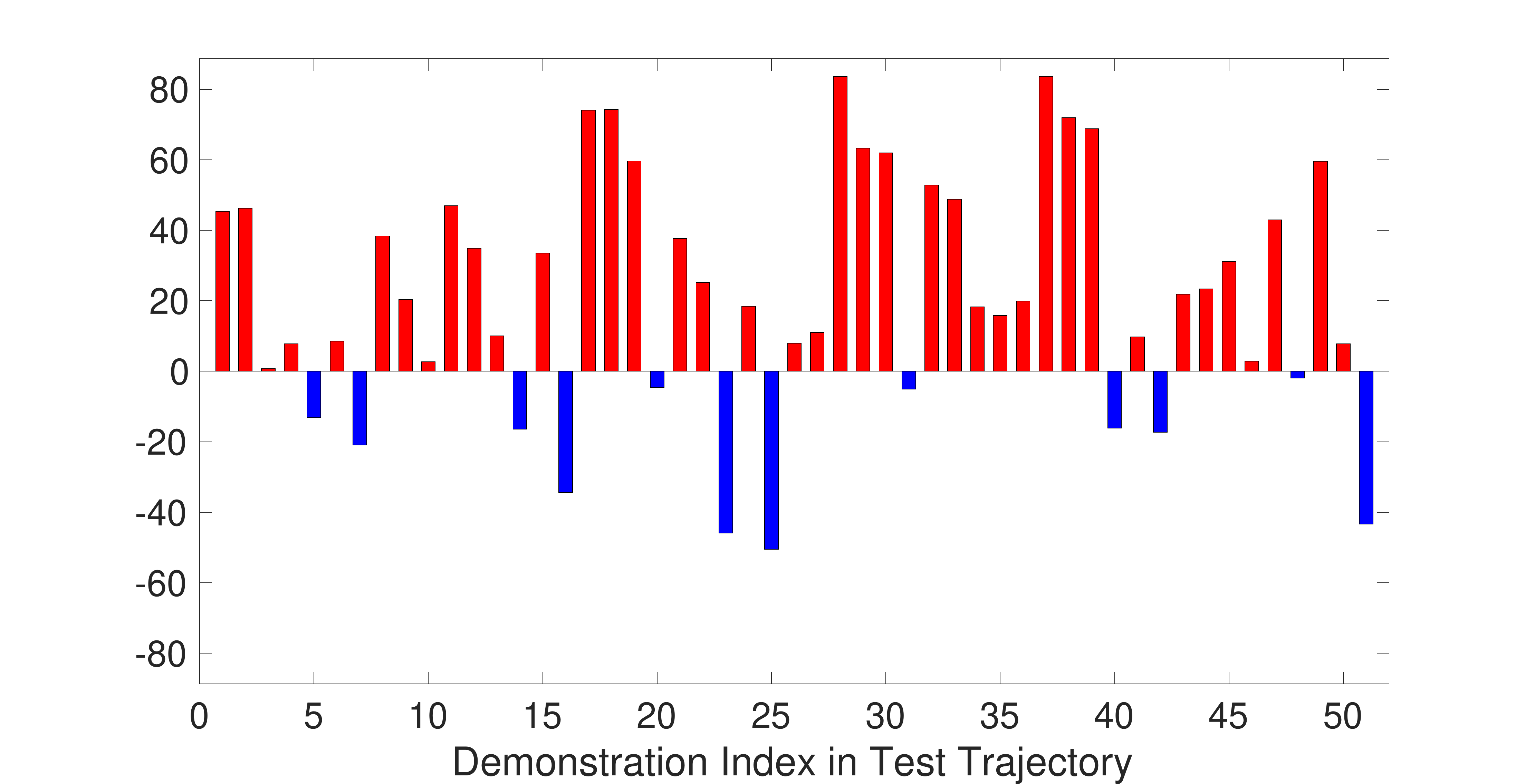}
\caption{Percentage improvement in $\Delta x_{\text{rel},t}$ for the RS-IRL model over RN-IRL for each 1.5 s trajectory segment. }
\label{fig:difference_x_nina}
\end{subfigure}
\begin{subfigure}[t]{0.48\textwidth}
\includegraphics[width=1\textwidth]{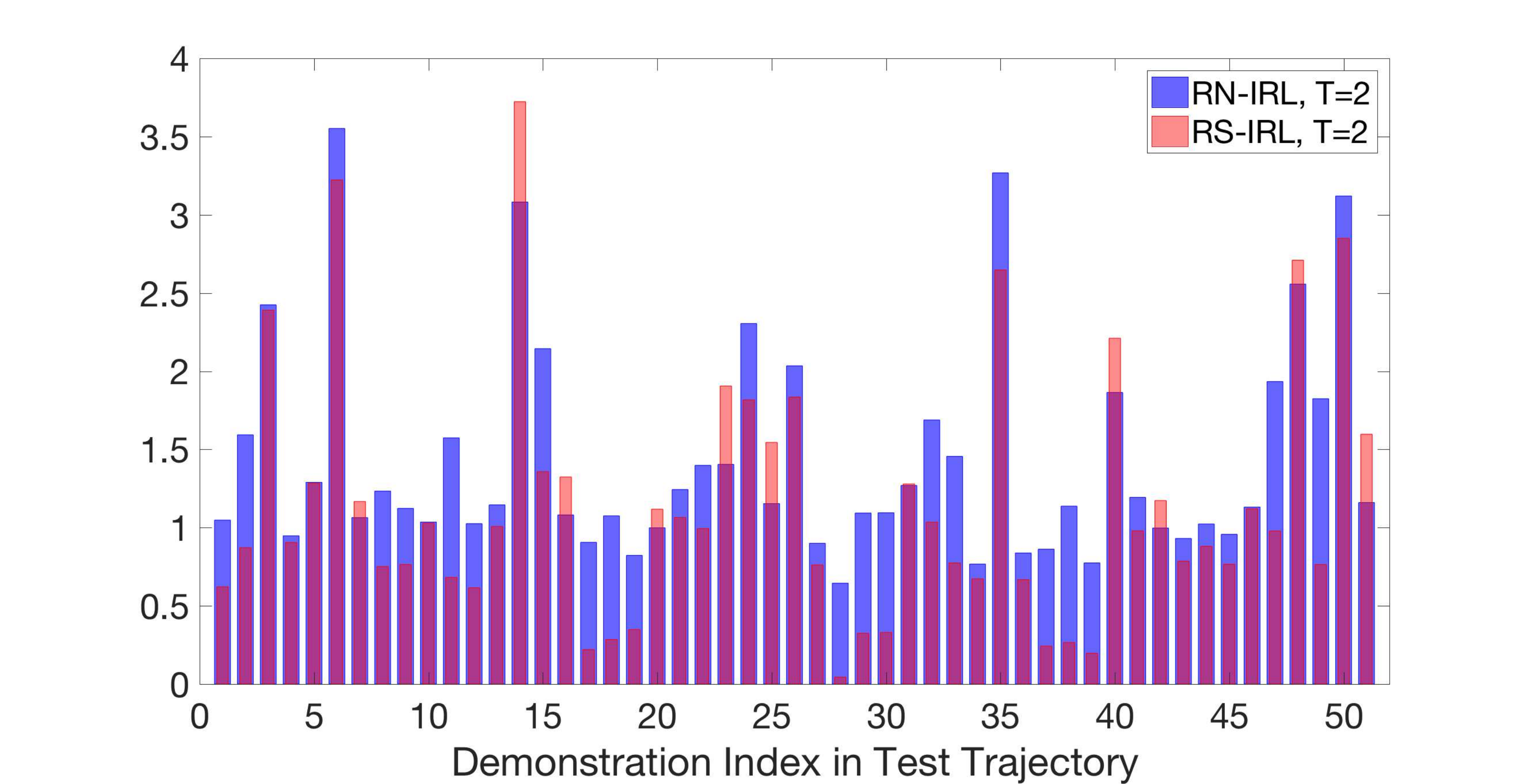}
\caption{Expected (w.r.t. stochastic policy) prediction errors $\Delta v_{x,\text{rel},t}$ from RS-IRL and RN-IRL for each 1.5 s trajectory segment.}
\label{fig:absolute_vx_nina}
\end{subfigure}\ \ 
\begin{subfigure}[t]{0.48\textwidth}
\includegraphics[width=1\textwidth]{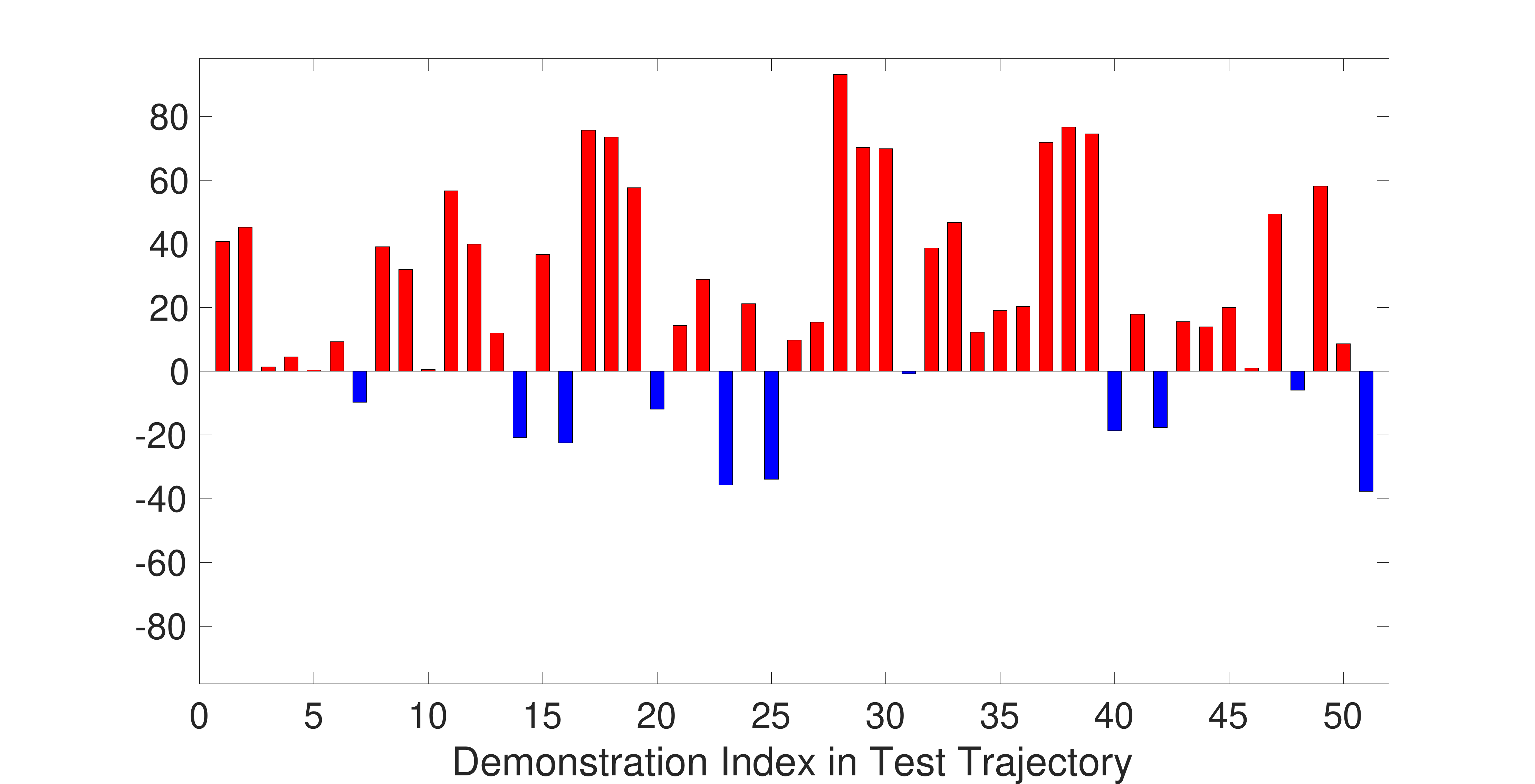}
\caption{Percentage improvement in $\Delta v_{x,\text{rel},t}$ for the RS-IRL model over RN-IRL for each 1.5 s trajectory segment.}
\label{fig:difference_vx_nina}
\end{subfigure}
\caption{Comparison of the $\Delta x_{\text{rel},t}$ and $\Delta v_{x,\text{rel},t}$ prediction errors (normalized by car length) for the RS-IRL and RN-IRL models for a highly risk-averse participant. The RS-IRL model almost always outperforms RN-IRL, on average providing about 22\% improvement. }
\end{figure}
\begin{figure}[H] 
\centering
\begin{subfigure}[t]{0.47\textwidth}
\includegraphics[width=\textwidth]{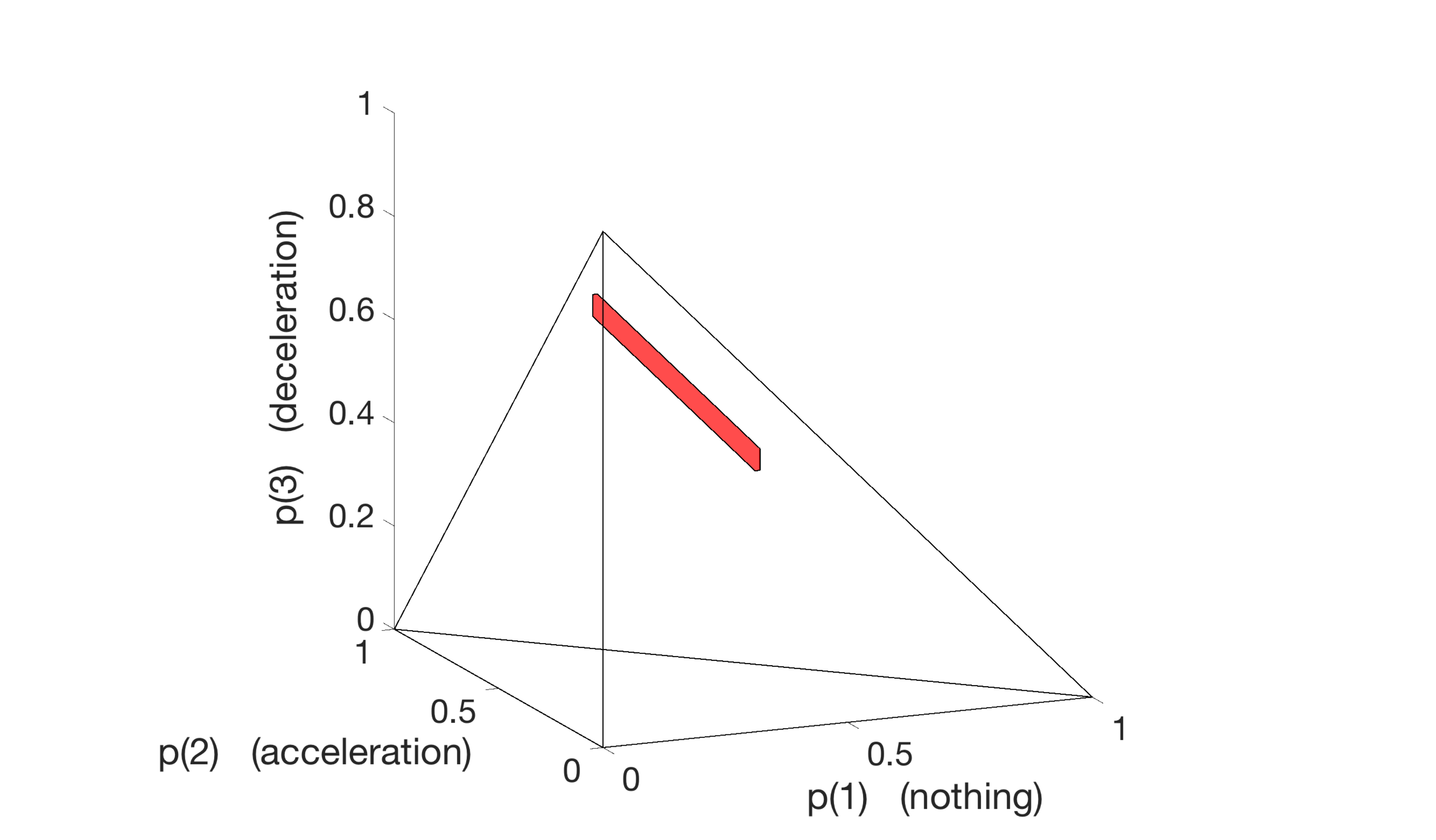}
\caption{Projection of the risk-averse participant's risk envelope along the first three dimensions, corresponding to the $\{ \text{nothing, accelerate, decelerate} \}$ maneuvers by the leader. }
\label{fig:polytope_2_nina}
\end{subfigure}\ \ 
\begin{subfigure}[t]{0.47\textwidth}
\includegraphics[width=\textwidth]{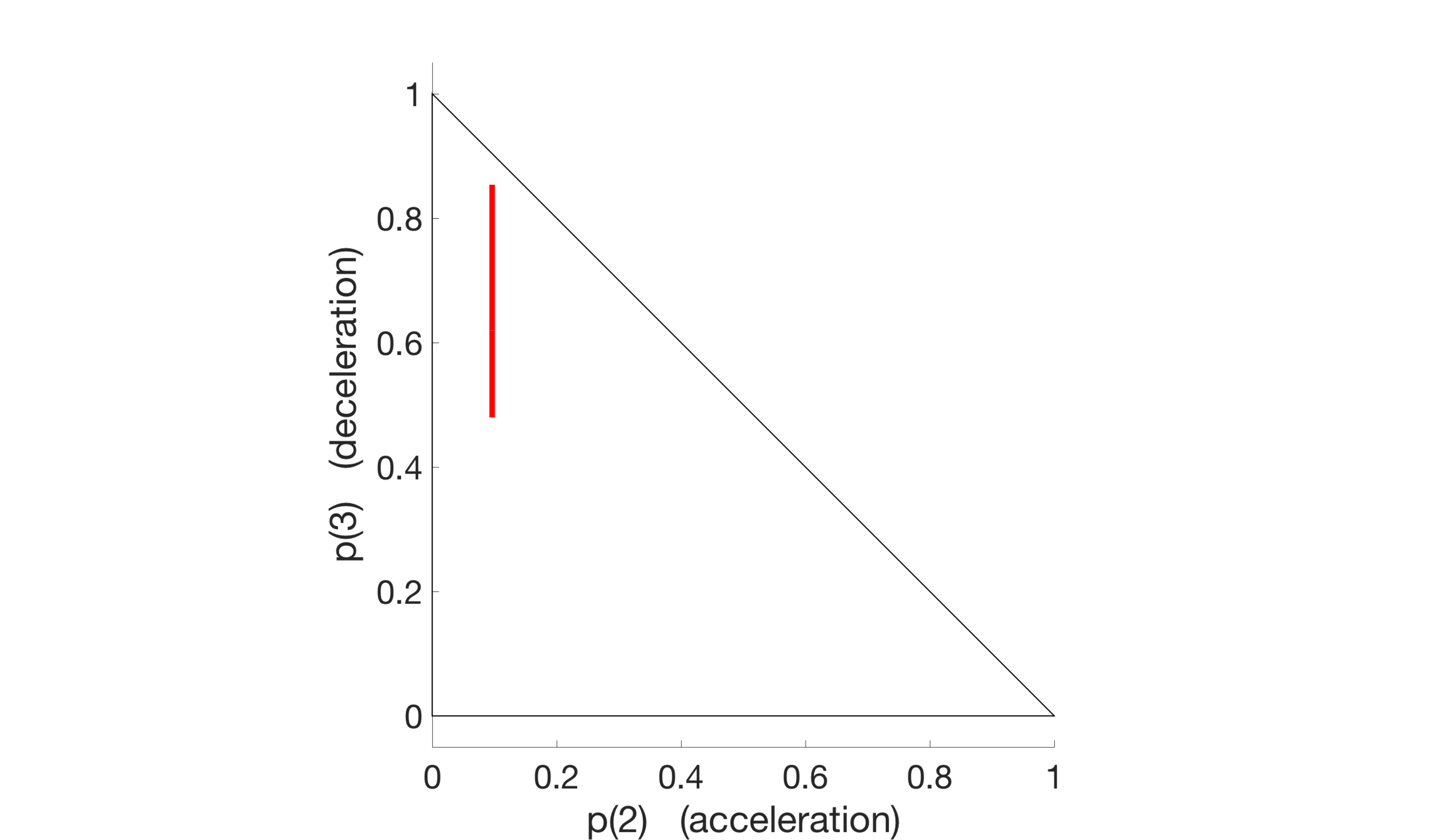}
\caption{Projection of the risk-averse participant's risk envelope along the second two dimensions, corresponding to the $\{\text{accelerate, decelerate} \}$ maneuvers by the leader.}
\label{fig:polytope_1_nina}
\end{subfigure}
\caption{Inferred risk envelope for risk-averse participant. Notice the overweighting (and ambiguity therein) of probabilities associated with the leader's deceleration maneuver and the underweighting of the leader's acceleration. The wire-frame pyramid is the projection of the probability simplex $\Delta^4$ onto the first three dimensions.}
\label{fig:polytope risk averse participant}
\end{figure}
It can be observed that the RS-IRL model almost always outperforms RN-IRL, on average providing an  improvement of about 22\%. Notice however, the four prominent negative peaks in Figure~\ref{fig:difference_x_nina} corresponding to the RN-IRL model outperforming RS-IRL on these instances. All four peaks can be explained by considering the inferred risk envelope $\Pp_r$, whose projections are plotted in Figures~\ref{fig:polytope_2_nina} and~\ref{fig:polytope_1_nina}. Notice how the participant's polytope is significantly biased towards high probabilities in the third dimension (corresponding to the decelerate maneuver for the leader), and furthermore, spans a wide range in probabilities ($\approx [0.5,0.85]$) for this event. Thus, not only does the participant overweight the true probability of this outcome (0.3), he/she is also \emph{ambiguous} about the true probability. On the other hand, the envelope suggests a very low and narrow range of probabilities for the leader's accelerate maneuver. These two properties are what result in the four negative peaks in Figure~\ref{fig:difference_x_nina}. 

Consider the first negative peak (corresponding to Sequence 1 in Figure~\ref{fig:full_x_nina}). At this decision stage, the true distance decreases due to the participant accelerating. As the learned RS-IRL model assumes significant ambiguity in the leader's deceleration yet small probability in acceleration, the most probable trajectories from the RS-IRL model prefer weaker acceleration. Thus, this particular artifact may be attributed to a slight overfitting of the parameterized polytope offsets $r$. This overfitting manifests itself again at the very end (Sequence 4) where the cars are far apart (6 car-lengths) and the RS-IRL model predicts a weaker deceleration than that observed. Essentially the high bias associated with the leader's deceleration as encoded in the inferred polytope along with the already large separation between the cars (recall that the second feature penalizes large separations) leads to a less aggressive deceleration prediction. The remaining peaks (Sequence 2) are shown in detail in Figure~\ref{fig:close_up_nina_2}. At both these decision stages, the RS-IRL model \emph{preempts}, by about one decision stage, a deceleration maneuver for the participant. Again this is a result of the deceleration ambiguity implied by the RS-IRL model and suggests that additional training data is needed to further refine (decrease) this ambiguity. This would naturally also alleviate the prediction errors in Sequences 1 and 4.  

Notice however that all of the spurious predictions by the RS-IRL model as discussed above occur when the cars are quite far apart, on the order of 5--6 car-lengths. In contrast, when the two cars are quite close (e.g., less than 4 car-lengths during Sequence 3), there is now a significant collision risk in the event that the leader decelerates. Figure~\ref{fig:close_up_nina_1} illustrates the RS-IRL model correctly predicting the participant braking to increase the relative separation, and with high probability ($\approx 0.89$). In contrast, the RN-IRL induced stochastic policy is not only of high entropy but also suggests quite absurd trajectories that lead to a further decrease in the relative separation. Thus, these results suggest that when RS-IRL underperforms RN-IRL, it does so at non-critical stages characterized by low risk. 

\begin{figure}[H]
\begin{subfigure}[t]{0.5\textwidth}
\includegraphics[width=1\textwidth]{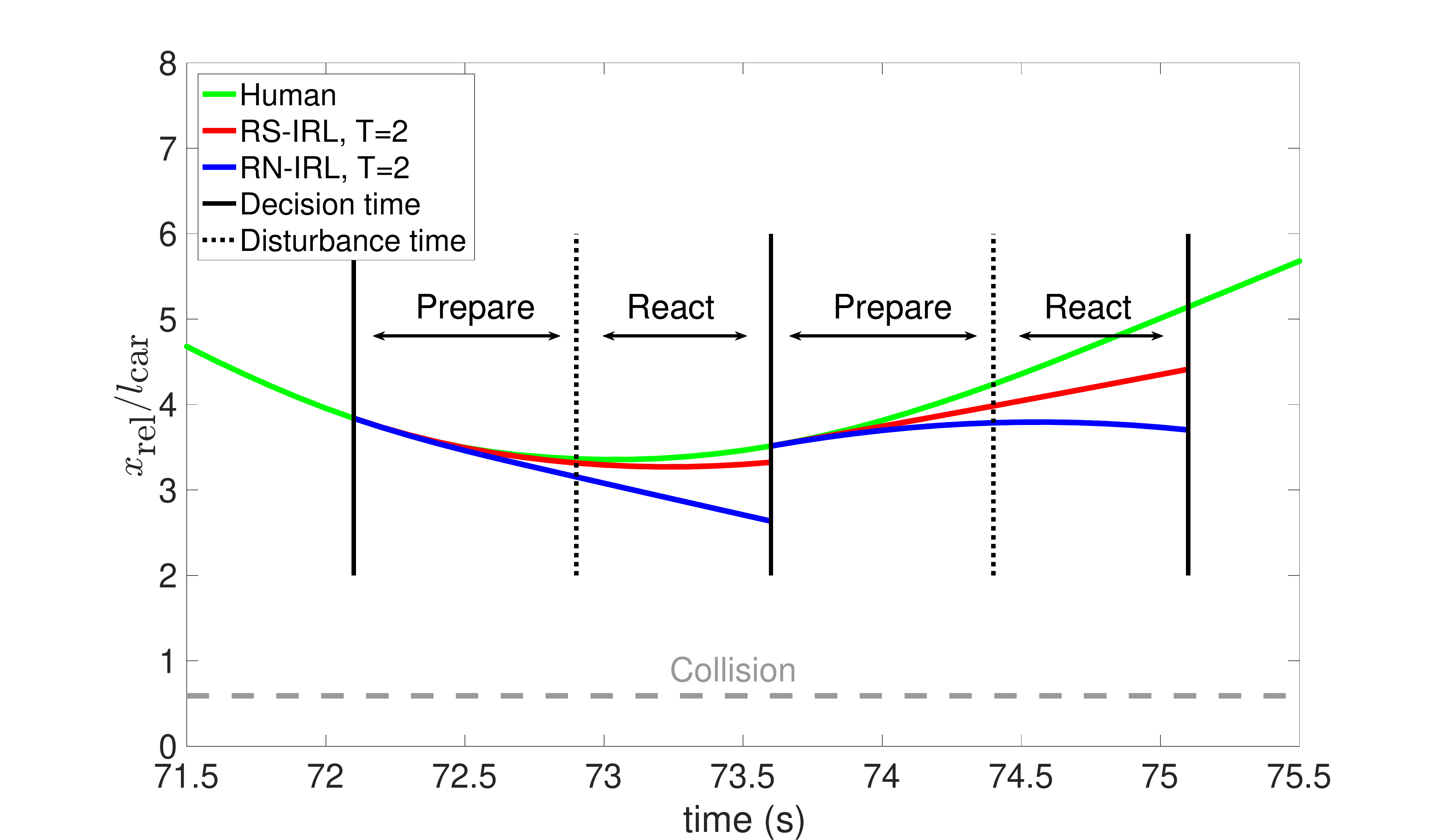}
\caption{Close-up of Sequence 3 in Figure~\ref{fig:full_x_nina} demonstrating the accuracy of RS-IRL over RN-IRL at a critical decision stage when cars are quite close. }
\label{fig:close_up_nina_1}
\end{subfigure}\ \ 
\begin{subfigure}[t]{0.5\textwidth}
\includegraphics[width=1\textwidth]{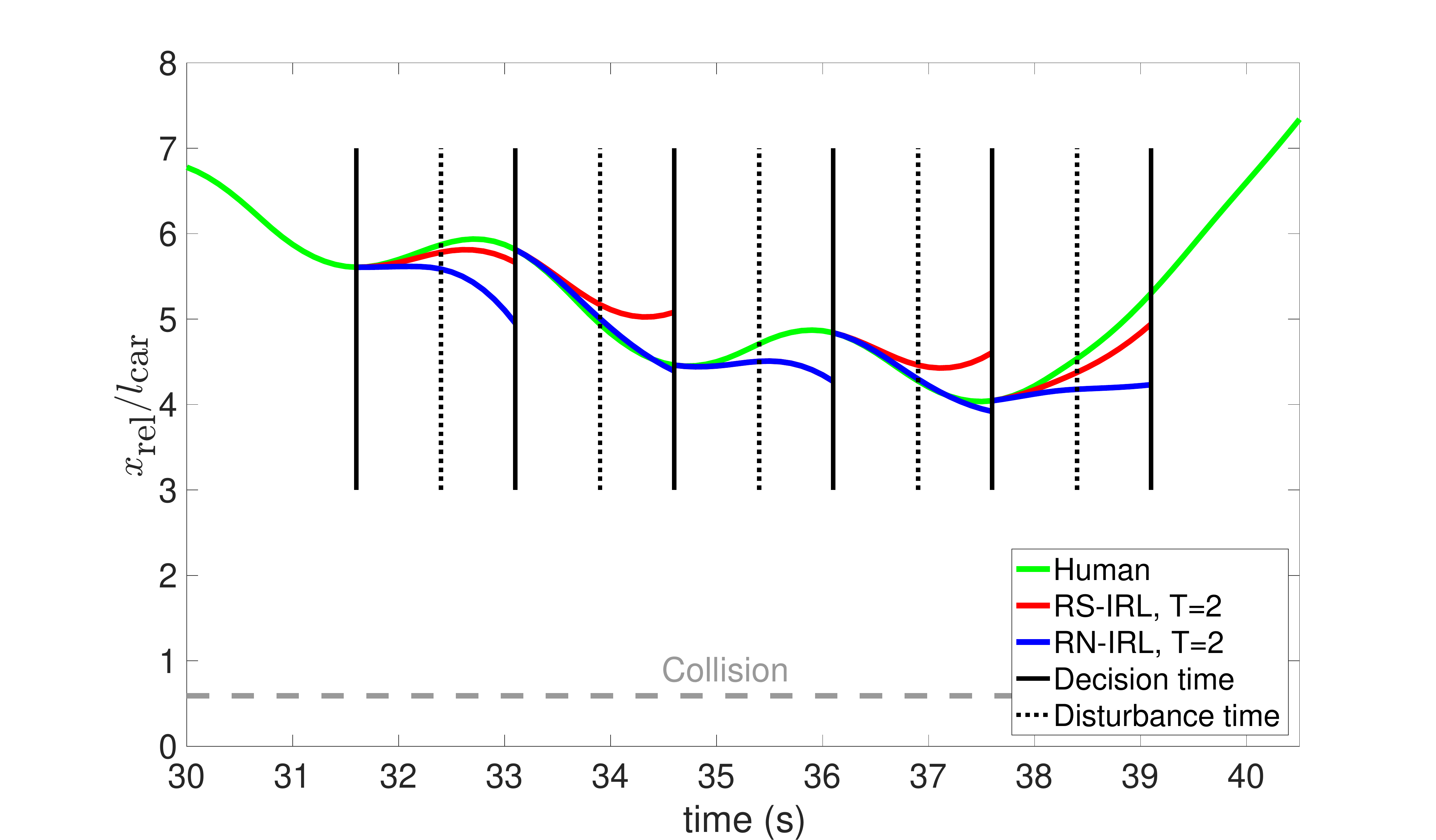}
\caption{Close-up of Sequence 2 in Figure~\ref{fig:full_x_nina} demonstrating the preemptive braking maneuvers predicted by RS-IRL. Notice that the error occurs at a non-critical decision stage when cars are far apart.}
\label{fig:close_up_nina_2}
\end{subfigure}
\caption{Comparisons of the \emph{most-probable} (under the stochastic Boltzmann policy) RS-IRL and RN-IRL trajectory predictions in $x_{\text{rel}}$ as compared with the true data. }
\label{fig:risk averse participant}
\end{figure}

\subsubsection{Case Study \# 2: Risk/Ambiguity-Averse Participant}

We next discuss a participant for whom RS-IRL again significantly outperformed RN-IRL, however, due to a different manifestation of risk. Consider the inferred risk envelopes in Figure~\ref{fig:polytope_federico}.
\begin{figure}[H]
\centering
\begin{subfigure}[t]{0.48\textwidth}
\includegraphics[width=1\textwidth]{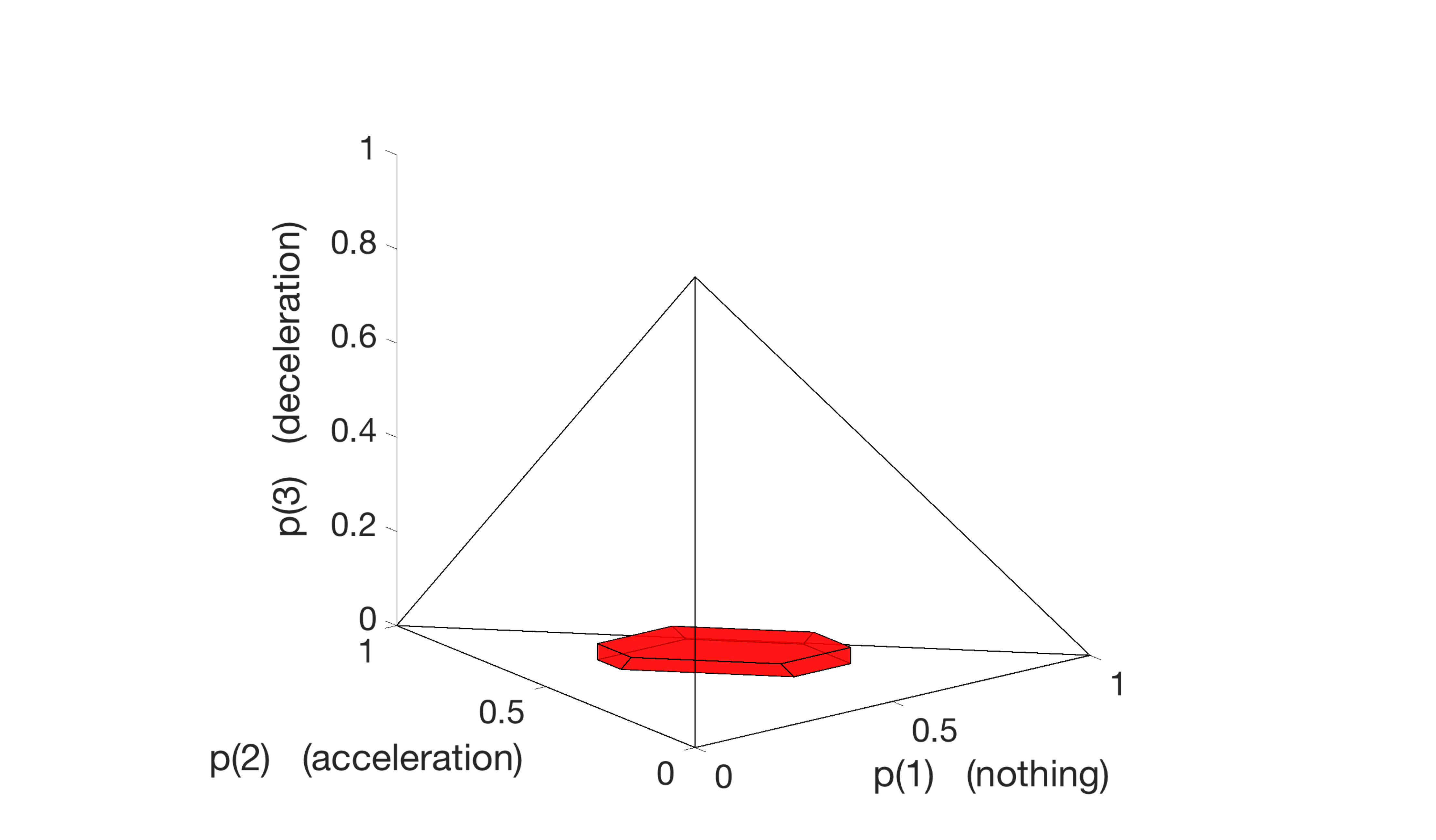}
\caption{Projection of polytope along the $\{ \text{nothing, accelerate, decelerate} \}$ ``dimensions''.}
\label{fig:polytope_1_federico}
\end{subfigure}\ \ 
\begin{subfigure}[t]{0.48\textwidth}
\includegraphics[width=1\textwidth]{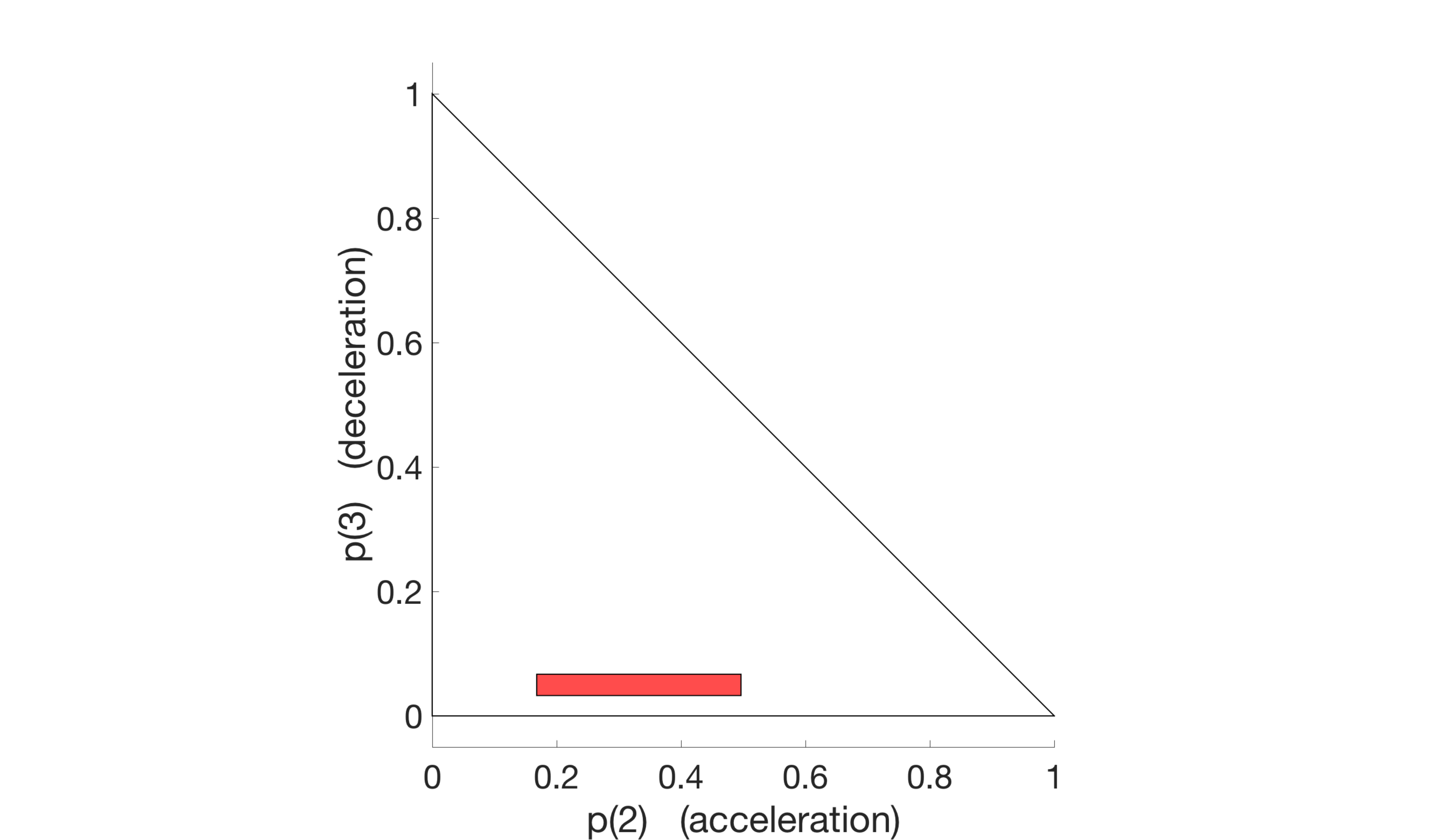}
\caption{Projection of polytope along the $\{ \text{accelerate, decelerate} \}$ ``dimensions''.}
\label{fig:polytope_2_federico}
\end{subfigure}
\caption{Inferred risk envelope for ambiguity-averse participant. Notice the drastic underweighting of probabilities associated with leader's deceleration and significant ambiguity regarding leader's acceleration. }
\label{fig:polytope_federico}
\end{figure}

Notice that the participant places very low probability on deceleration yet is ambiguous regarding the leader's acceleration (i.e., a polar opposite of the first participant). This ambiguity led to several preemptive braking maneuvers in an attempt to maintain a ``safe'' distance to the leader car. Figures~\ref{fig:absolute_x_federico}--\ref{fig:difference_vx_federico} illustrate the consistency of RS-IRL's improvement over RN-IRL in predicting these maneuvers over the entire test trajectory.

\begin{figure}[H] 
\centering
\begin{subfigure}[t]{0.48\textwidth}
\includegraphics[width=1\textwidth]{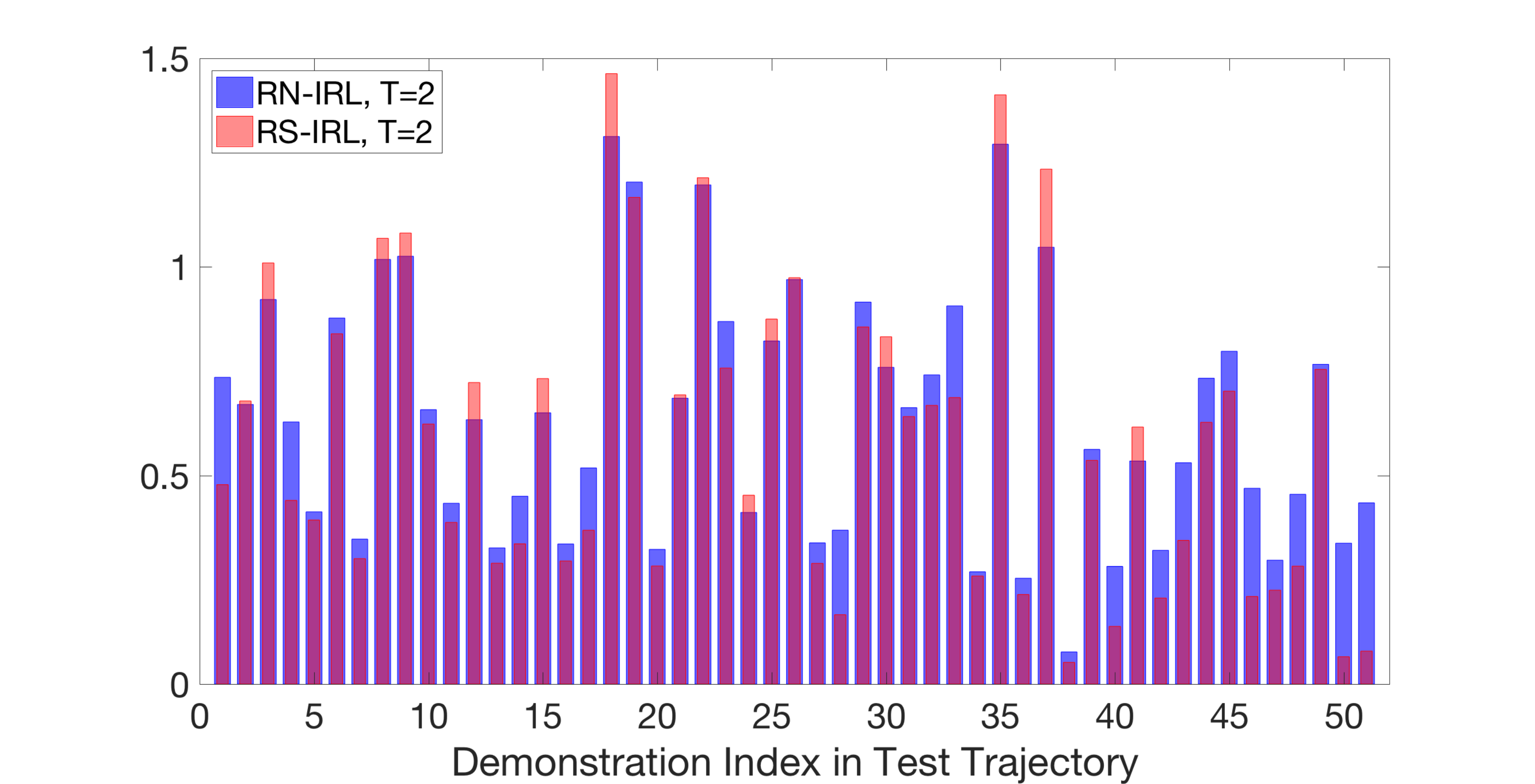}
\caption{Expected (w.r.t. stochastic policy) prediction errors $\Delta x_{\text{rel},t}$ from RS-IRL and RN-IRL for each 1.5 s trajectory segment.}
\label{fig:absolute_x_federico}
\end{subfigure}\ \ 
\begin{subfigure}[t]{0.48\textwidth}
\includegraphics[width=1\textwidth]{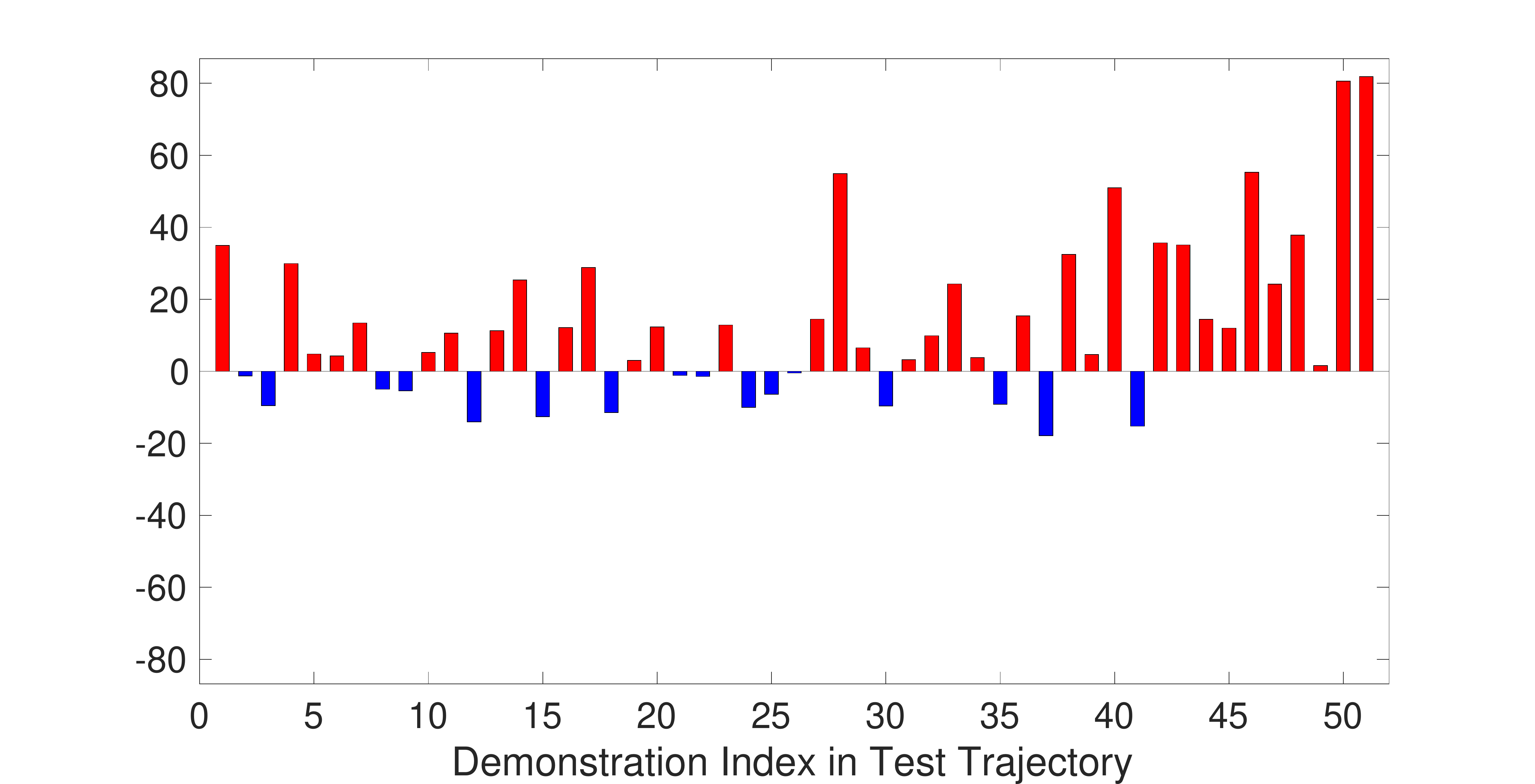}
\caption{Percentage improvement in $\Delta x_{\text{rel},t}$ for the RS-IRL model over RN-IRL for each 1.5 s trajectory segment.}
\label{fig:difference_x_federico}
\end{subfigure}
\begin{subfigure}[t]{0.48\textwidth}
\includegraphics[width=1\textwidth]{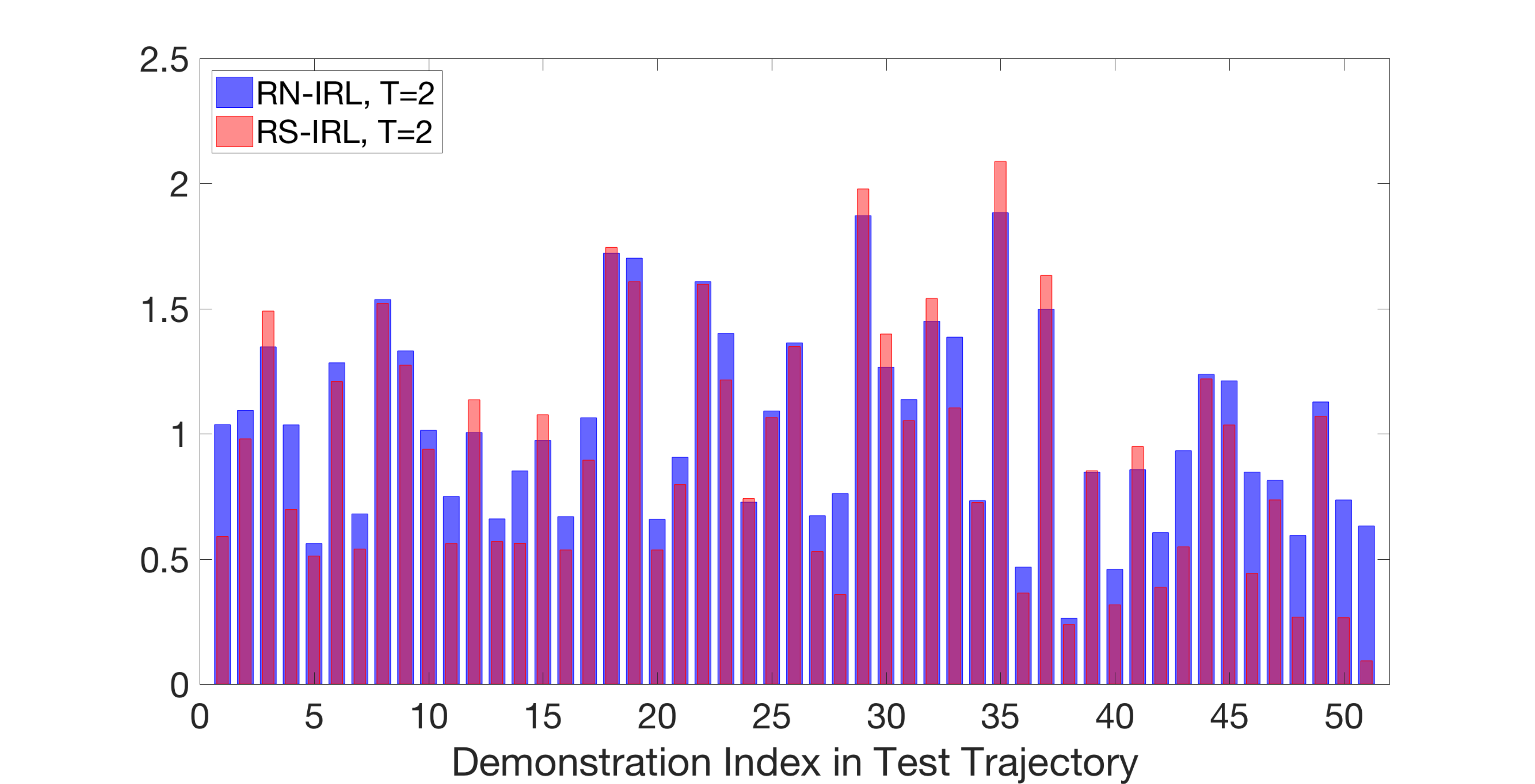}
\caption{Expected (w.r.t. stochastic policy) prediction errors $\Delta v_{x,\text{rel},t}$ from RS-IRL and RN-IRL for each 1.5 s trajectory segment.}
\label{fig:absolute_vx_federico}
\end{subfigure}\ \ 
\begin{subfigure}[t]{0.48\textwidth}
\includegraphics[width=1\textwidth]{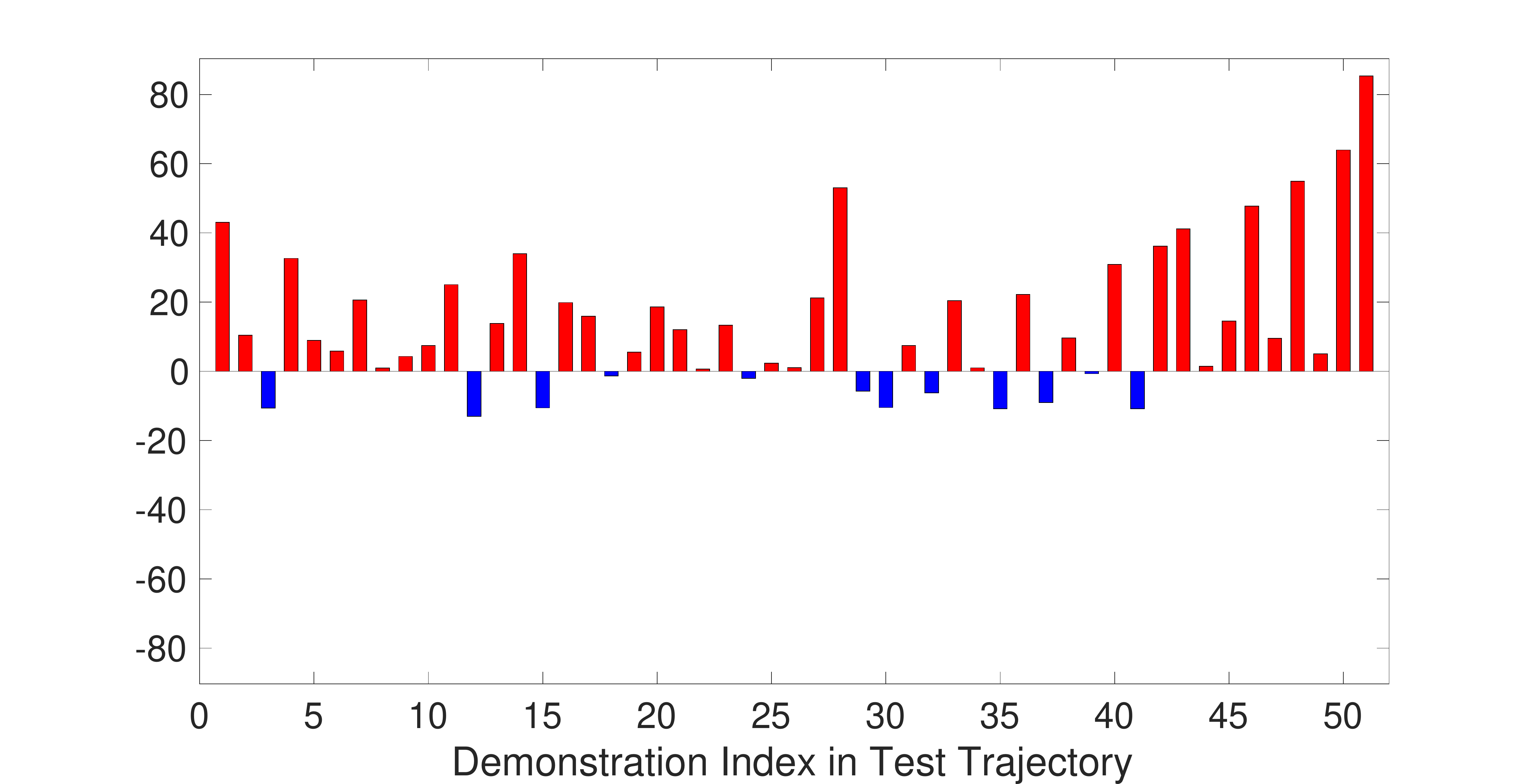}
\caption{Percentage improvement in $\Delta v_{x,\text{rel},t}$ for the RS-IRL model over RN-IRL for each 1.5 s trajectory segment.}
\label{fig:difference_vx_federico}
\end{subfigure}
\caption{Comparison of the $\Delta x_{\text{rel},t}$ and $\Delta v_{x,\text{rel},t}$ prediction errors (normalized by car length) for the RS-IRL and RN-IRL models for an ambiguity-averse participant. The RS-IRL model almost always outperforms RN-IRL, on average providing about 13--14\% improvement. }
\label{fig:risk neutral participant 5}
\end{figure}

\subsubsection{Case Study \# 3: Risk-Neutral Participant}

Finally, we consider a participant for whom both RS-IRL and RN-IRL performed on par with each other. Figure~\ref{fig:full_x_ed} plots the $x_{\text{rel}}$ trajectory while Figure~\ref{fig:polytope_ed} illustrates the inferred risk envelope. 
\begin{figure}[H]
\centering
\includegraphics[width=0.55\textwidth]{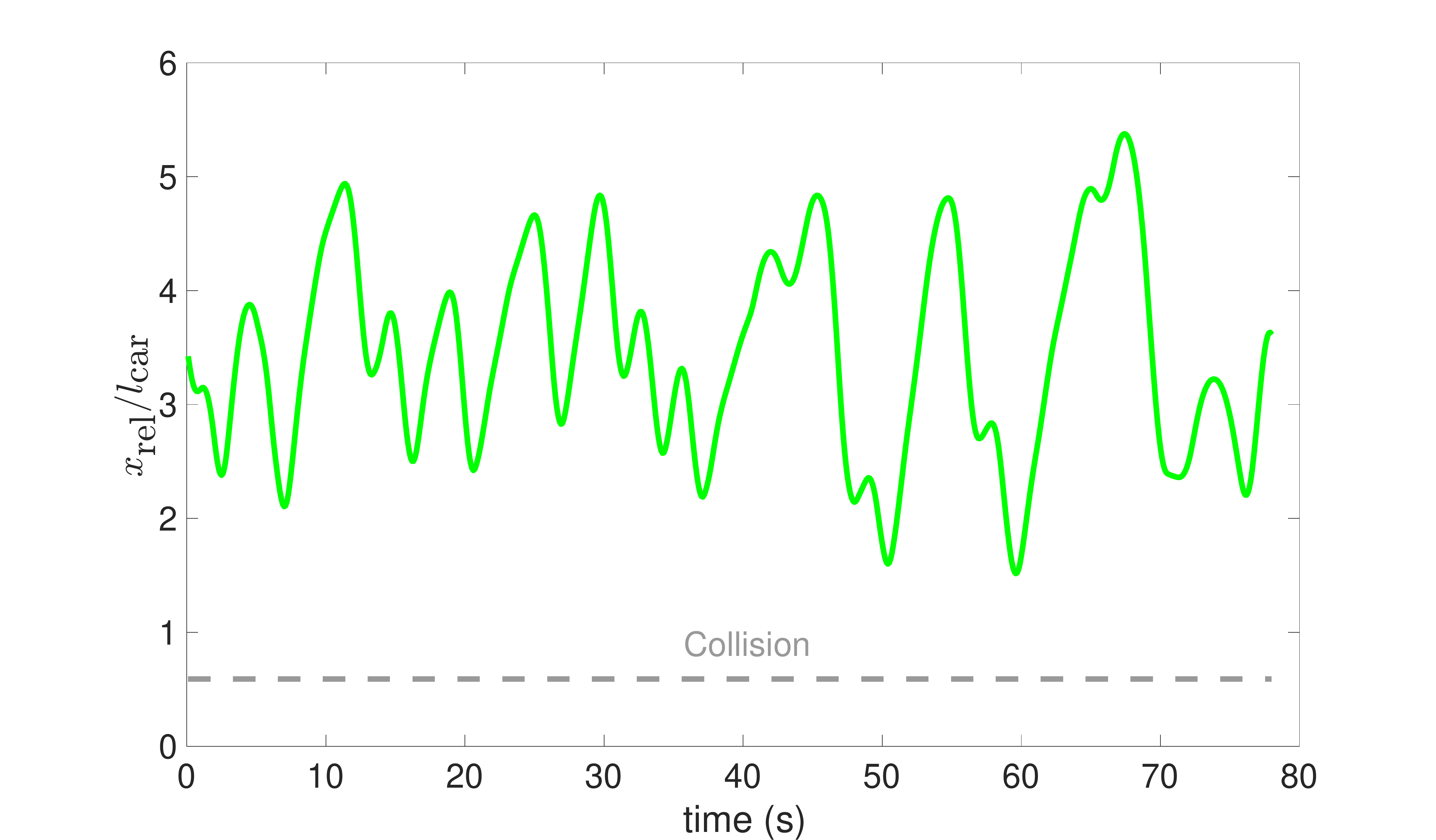}
\caption{Full $x_{\text{rel}}$ trajectory (normalized by car length) for a risk-neutral participant. On average, the relative distance is noticeably smaller, on the order of 3 car-lengths.}
\label{fig:full_x_ed}
\end{figure}

\begin{figure}[H]
\centering
\begin{subfigure}[t]{0.48\textwidth}
\includegraphics[width=1\textwidth]{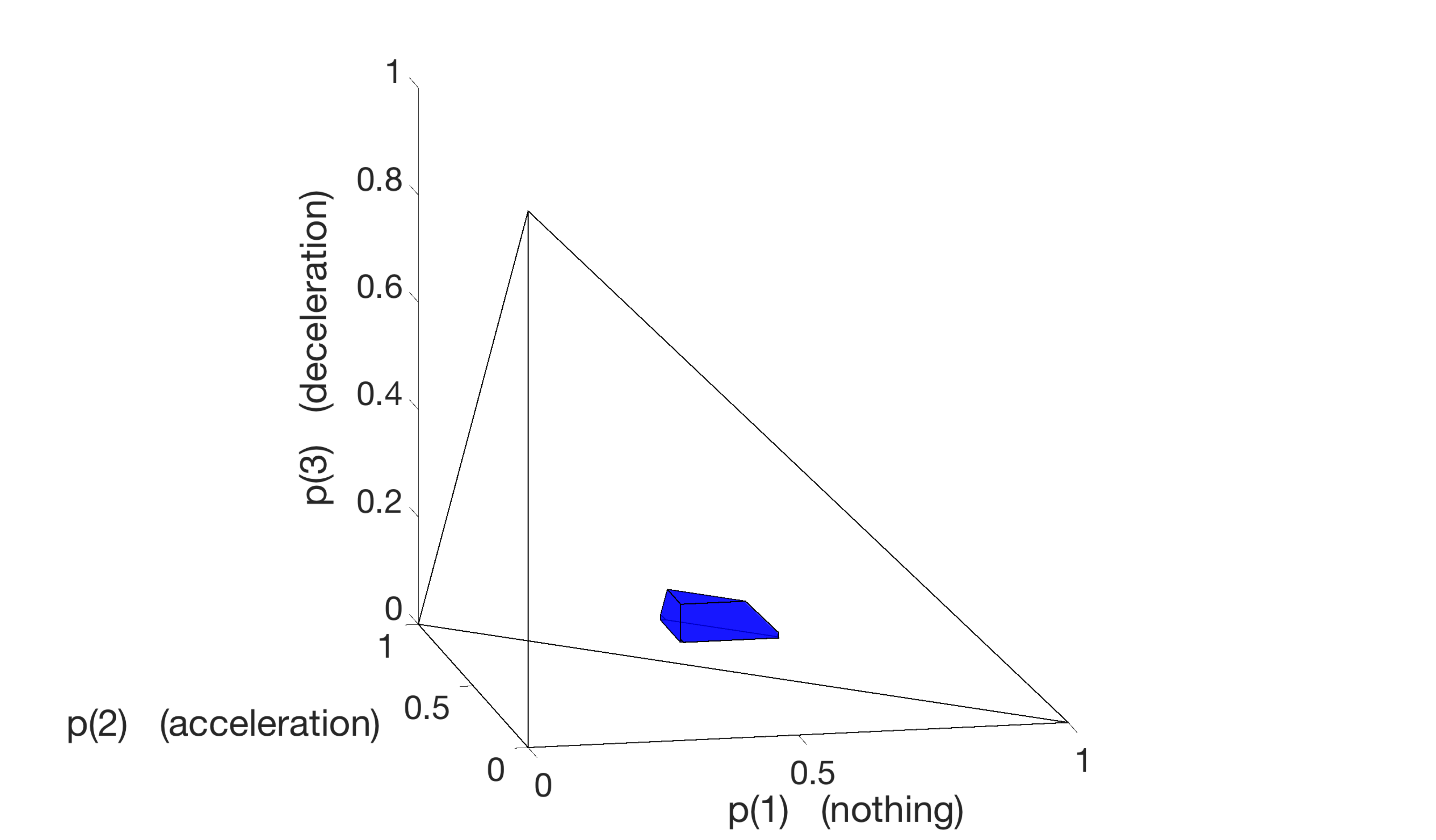}
\caption{Projection of polytope along the $\{ \text{nothing, accelerate, decelerate} \}$ ``dimensions''.}
\label{fig:polytope_1_ed}
\end{subfigure}\ \ 
\begin{subfigure}[t]{0.48\textwidth}
\includegraphics[width=1\textwidth]{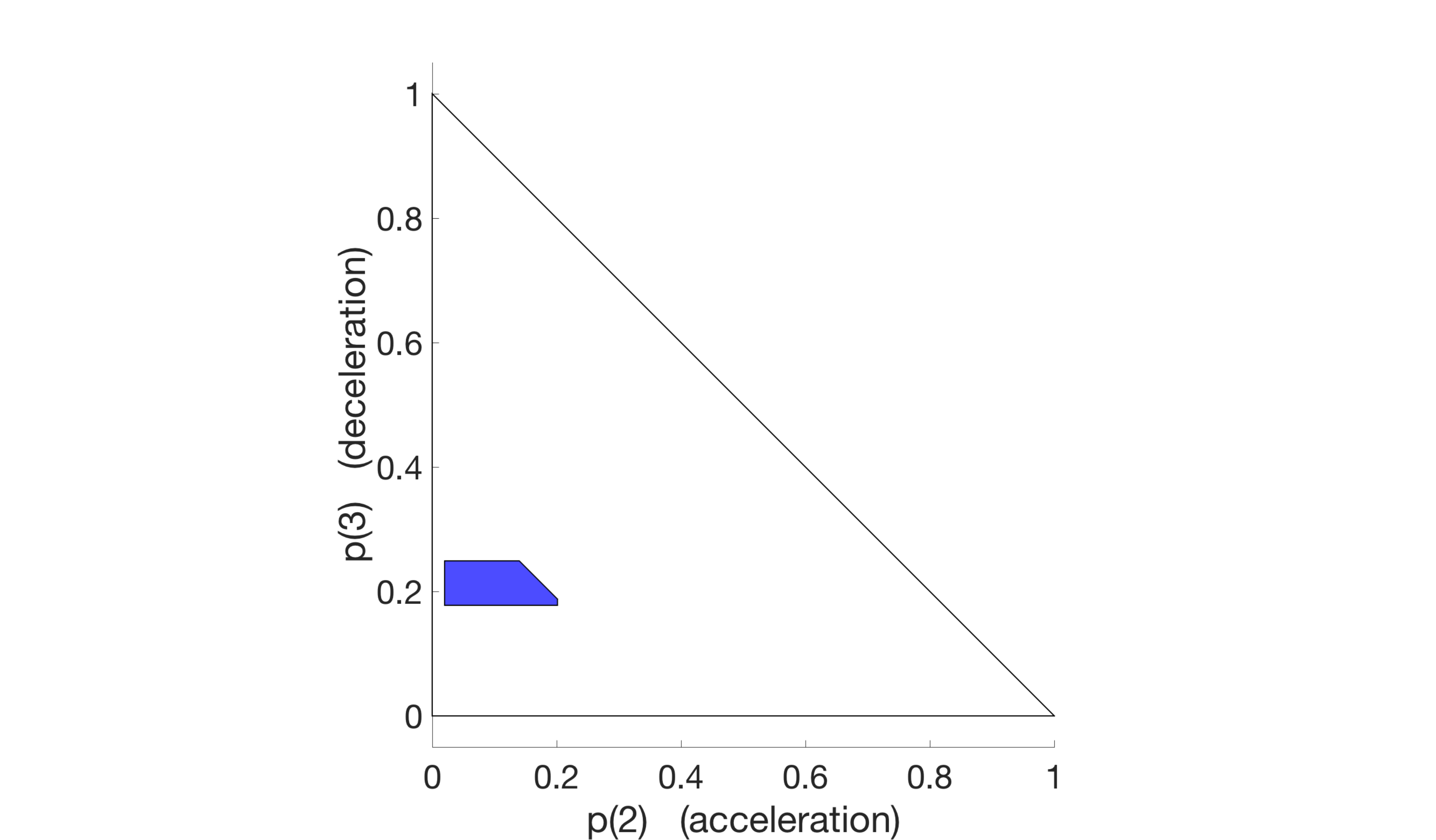}
\caption{Projection of polytope along the $\{ \text{accelerate, decelerate} \}$ ``dimensions''.}
\label{fig:polytope_2_ed}
\end{subfigure}
\caption{Inferred risk envelope for risk-neutral participant. Notice how there is no appreciable level of ambiguity nor is the polytope biased along any dimension; hence suggesting the risk-neutral categorization.}
\label{fig:polytope_ed}
\end{figure}

The inferred risk envelope features no bias along any dimension nor any appreciable level of ambiguity; thereby suggesting a risk-neutral profile for this participant.  Furthermore, notice in Figure~\ref{fig:full_x_ed} that the participant stays quite a bit closer to the leader car than the first participant, a clear indicator of his/her risk-neutral stance. Figures~\ref{fig:ed_comp_x} and~\ref{fig:ed_comp_vx} confirm the two models performing on par with each other. This, however, is to be expected for a strongly risk-neutral participant since the RS-IRL model subsumes the RN-IRL model.

\begin{figure}[H] 
\centering
\begin{subfigure}[t]{0.48\textwidth}
\includegraphics[width=1\textwidth]{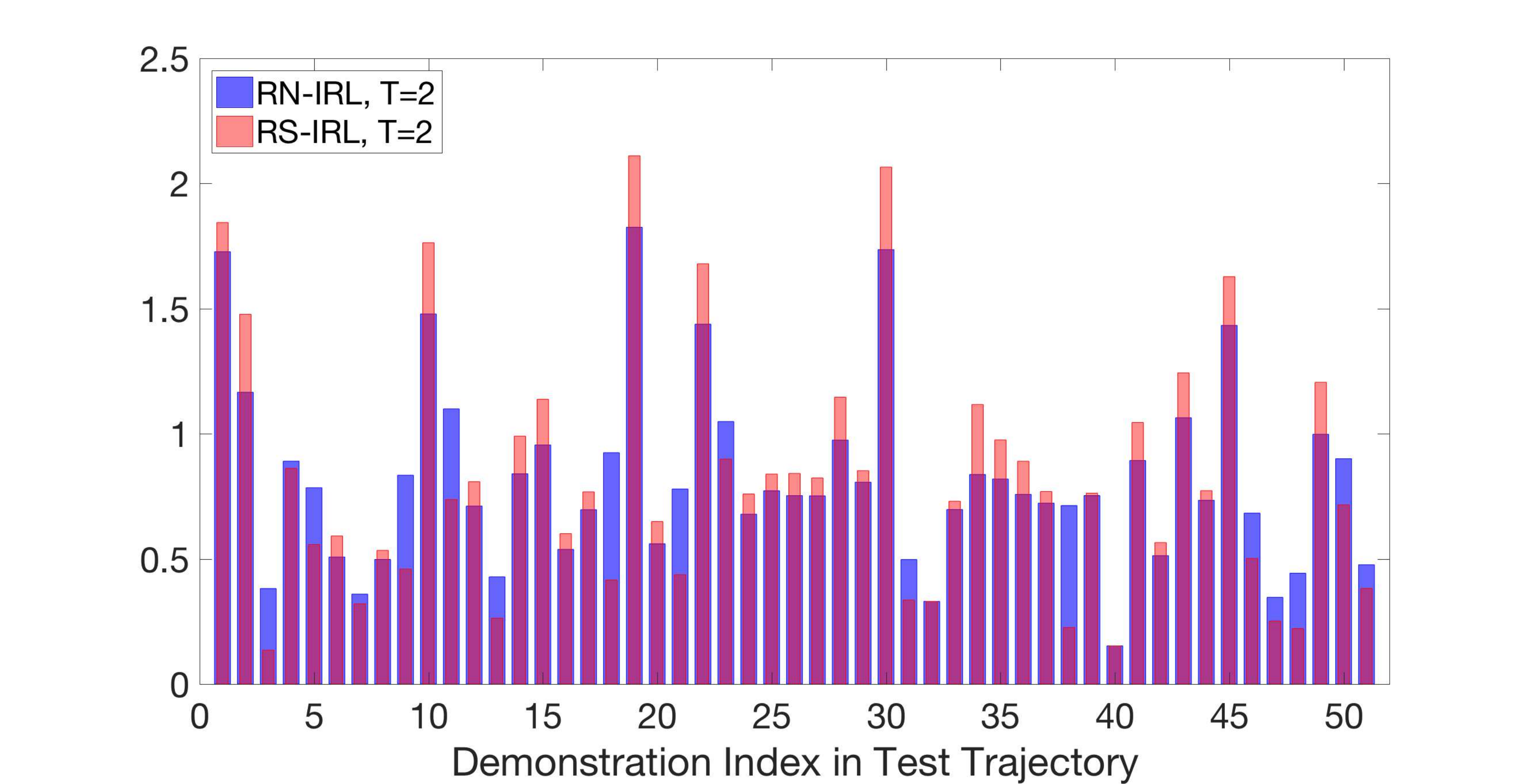}
\caption{Expected (w.r.t. stochastic policy) prediction errors $\Delta x_{\text{rel},t}$ from RS-IRL and RN-IRL for each 1.5 s trajectory segment.}
\label{fig:absolute_x_ed}
\end{subfigure}\ \ 
\begin{subfigure}[t]{0.48\textwidth}
\includegraphics[width=1\textwidth]{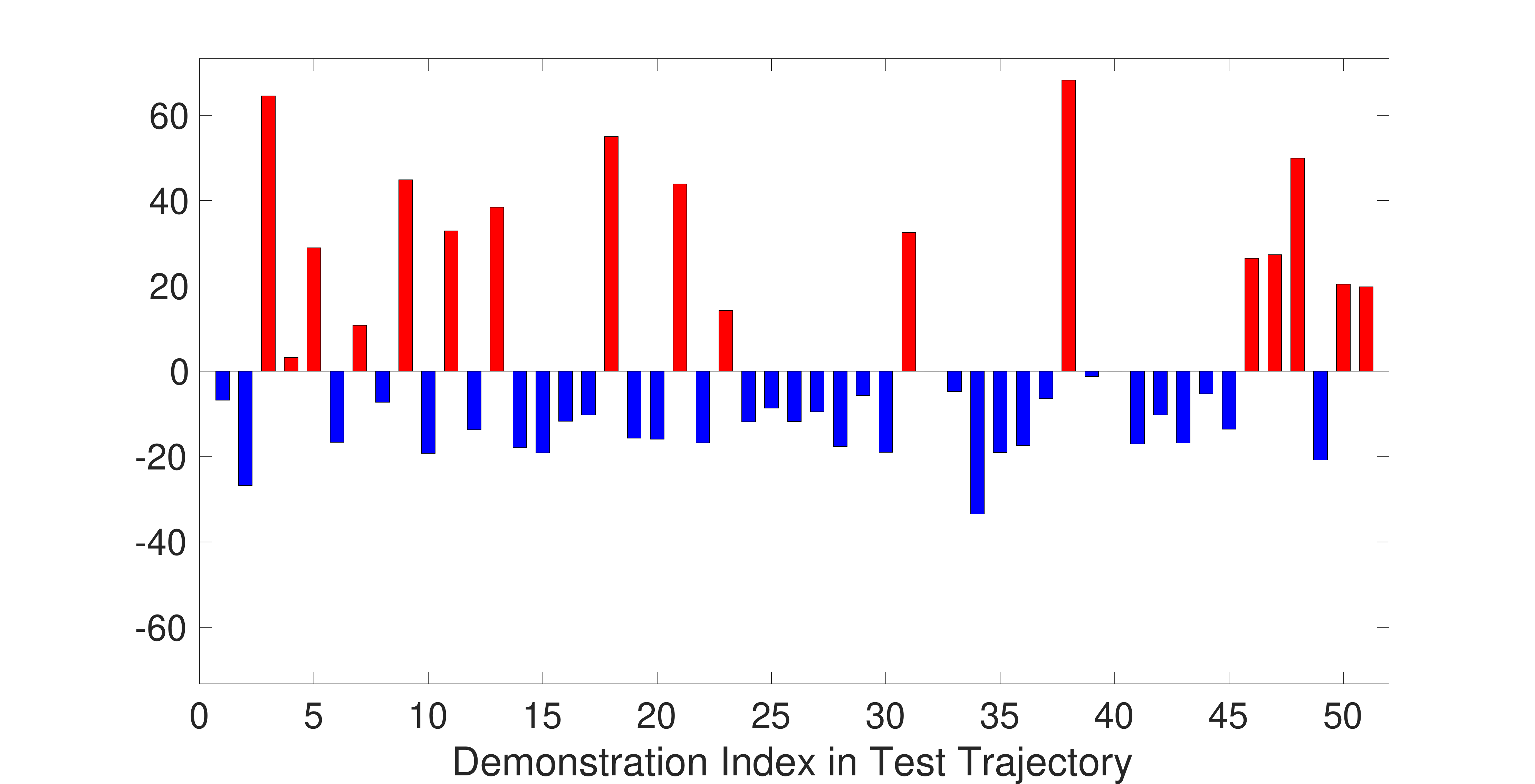}
\caption{Percentage improvement in $\Delta x_{\text{rel},t}$ for the RS-IRL model over RN-IRL for each 1.5 s trajectory segment.}
\label{fig:difference_x_ed}
\end{subfigure}
\caption{Comparison of the $\Delta x_{\text{rel},t}$ prediction errors (normalized by car length) for the RS-IRL and RN-IRL models for a risk-neutral participant. The two models perform on par with each other.}
\label{fig:ed_comp_x}
\end{figure}

\begin{figure}
\begin{subfigure}[t]{0.48\textwidth}
\includegraphics[width=1\textwidth]{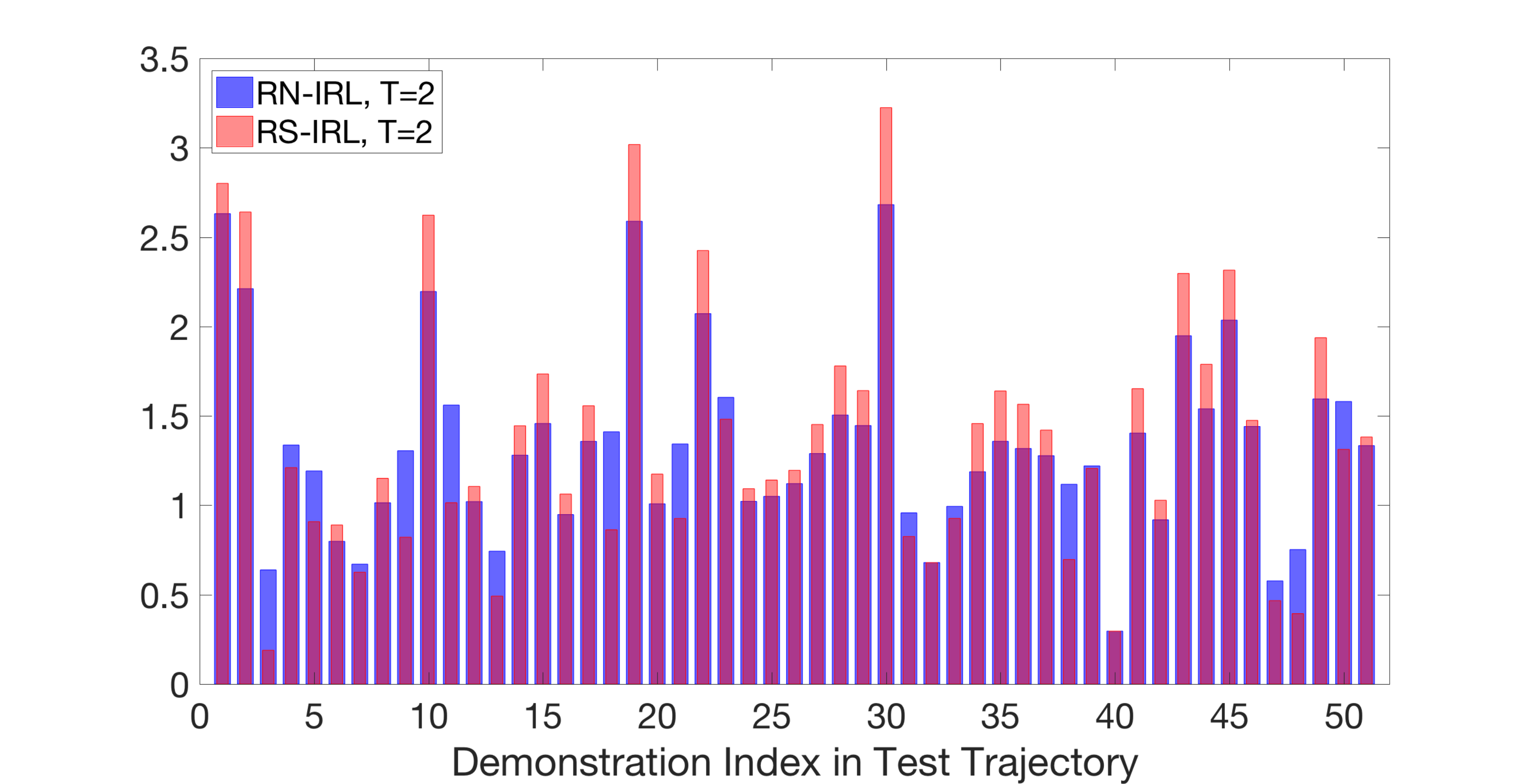}
\caption{Expected (w.r.t. stochastic policy) prediction errors $\Delta v_{x,\text{rel},t}$ from RS-IRL and RN-IRL for each 1.5 s trajectory segment.}
\label{fig:absolute_vx_ed}
\end{subfigure}\ \ 
\begin{subfigure}[t]{0.48\textwidth}
\includegraphics[width=1\textwidth]{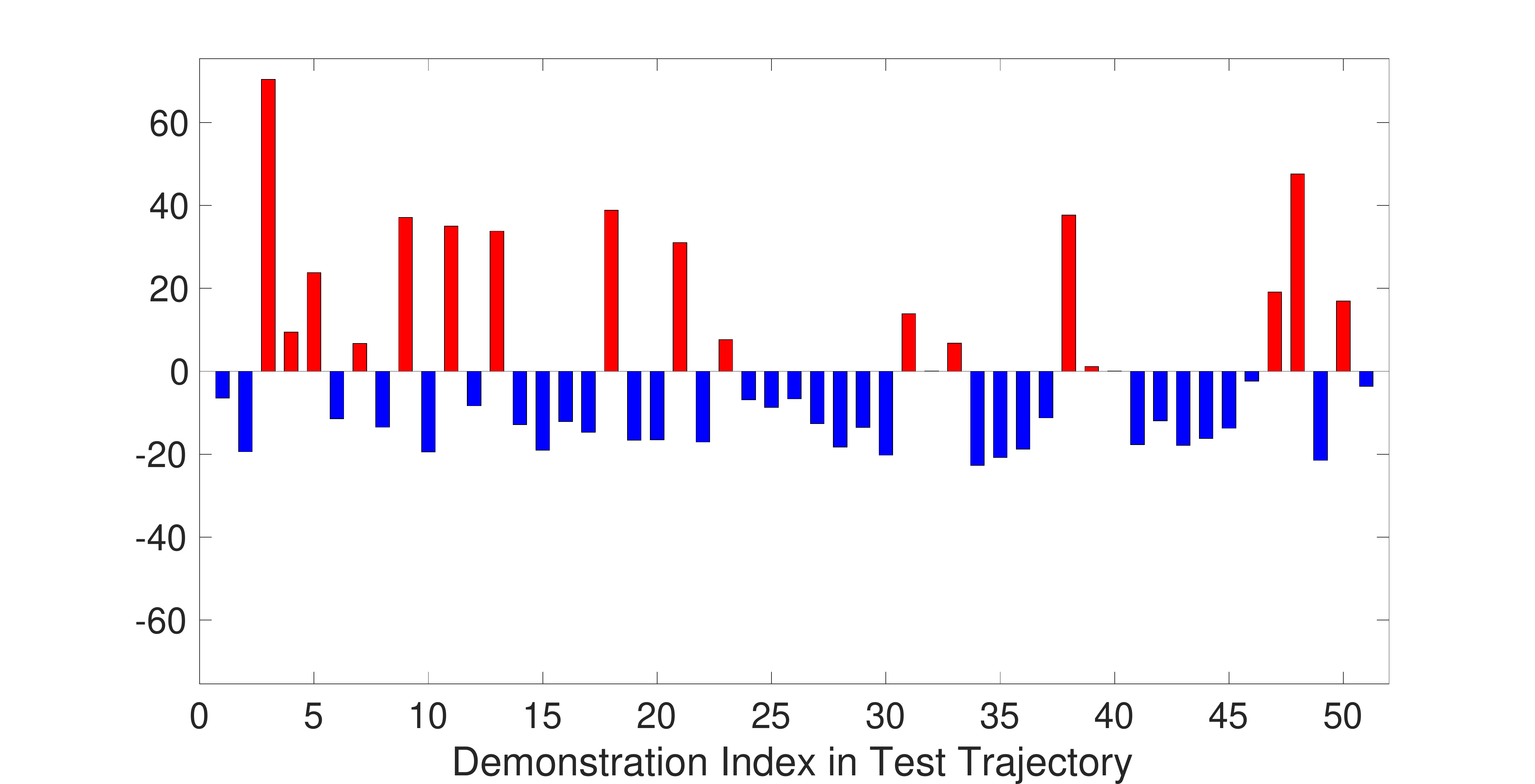}
\caption{Percentage improvement in $\Delta v_{x,\text{rel},t}$ for the RS-IRL model over RN-IRL for each 1.5 s trajectory segment.}
\label{fig:difference_vx_ed}
\end{subfigure}
\caption{Comparison of the $\Delta v_{x,\text{rel},t}$ prediction errors (normalized by car length) for the RS-IRL and RN-IRL models for a risk-neutral participant. The two models perform on par with each other.}
\label{fig:ed_comp_vx}
\end{figure}

\subsubsection{Limitations of Cost Shaping}

In this section we argue that \emph{both} the cost weights $c$ and the risk measure $\rho(\cdot)$ are \emph{necessary} to reasonably approximate diverse risk-sensitive behaviors. \revision{In Table~\ref{tab:rs_rn_comp}, we provide the learned feature weights from RS-IRL and RN-IRL for features $\{\phi_1,\phi_2,\phi_3,\phi_4\}$ for the participants in case studies \#1 and \#2.} These are the four features that dominate the along-track behavior of the participants.

\begin{table}[H]
\centering
 \begin{tabular}{  c | c | c | c | c | }
          & \multicolumn{2}{c}{Case Study \#1} &  \multicolumn{2}{c}{Case Study \#2}\\ \hline
     Feature weight & RS-IRL & RN-IRL & RS-IRL & RN-IRL \\ \hline
    $c(1)$   &   0.0918    &   0.0466   & 0.0748  &  0.1884  \\ \hline
    $c(2)$   &    0.5174   &   0.5589   & 0.6354  & 0.2313  \\ \hline
    $c(3)$   &  0.0993     & 0.1052     & 0.1864  &  0.3170  \\ \hline
    $c(4)$   &  0.2352   &   0.2330   &  0.0562  &  0.0624  \\ \hline
  \end{tabular}
 \caption{Comparison of the inferred cost weights from RS-IRL and RN-IRL for the participants in case studies \#1 and 2. }
 \label{tab:rs_rn_comp}
\end{table}

Notice that in spite of the extremely similar cost weights for the participant in case study \#1, the improvement in performance using RS-IRL is substantial. For the participant in case study \#2, a difference in the cost weights \emph{and} the use of a risk measure were needed to obtain the observed performance boost with RS-IRL. This clearly establishes the benefits of risk sensitive inference, particularly highlighting the deficiency of only using cost shaping (in general, risk-neutral algorithms) due to its inability to cope with more nuanced manifestations of uncertainty in decision-making. 

\subsubsection{Summary}

Table \ref{tab:position velocity errors N2} presents the average (over the 51 1.5 s segments in the test trajectory) percentage improvement (RN-IRL to RS-IRL) in the prediction errors, i.e., $\Delta x_{\text{rel}}$, $\Delta y_{\text{rel}}$, $\Delta v_{x,\text{rel}}$ and $\Delta v_{y,\text{rel}}$. As expected, the RS-IRL predictions are almost always better than those \revision{provided} by RN-IRL, with as much as $22.0\%$ improvement in $x_{\text{rel}}$ and $23.1\%$ improvement in $v_{x,\text{rel}}$. Regarding errors in $y_{\text{rel}}$ and $v_{y,\text{rel}}$, RS-IRL and RN-IRL perform comparably (absolute errors were in the range 0.19 -- 0.42 m) since the primary source of risk-aversion stems from the leader's  acceleration and deceleration rather than its lateral motion. In the cases with noticeable improvement, either in position or velocity (i.e., participants \#2, 5, 6, 9), the better predictions were a result of the RS-IRL model more accurately representing participants with higher levels of risk- and/or ambiguity-aversion. Indeed the first two detailed case studies presented earlier correspond to participants \#2 and \#5. In contrast, for more risk-neutral participants (e.g., the last presented case study corresponding to participant \#4), the performance improvements were either less pronounced or the two models performed comparably. 

\begin{table}[H]
\centering
 \begin{tabular}{ | c | c | c | c | c | c | c | c | c | c | c |}
    \hline 
    Participant \#  & 1 & \color{red}{2}  & 3 & 4 & \color{red}{5} & \color{red}{6} & 7 & 8 & \color{red}{9} & 10 \\ \hline
    $\Delta x_{\text{rel}}$ (T=2) & 7.5  & \color{red}{22.0} & 3.3 & 2.6 & \color{red}{13.3} & \color{red}{14.9} & 3.9 & 9.7 & \color{red}{18.6} & 10.1 \\ \hline
    $\Delta y_{\text{rel}}$ (T=2) & 10.7 & \color{red}{22.2} & -2.7 & 12.0 & \color{red}{4.8} & \color{red}{9.5} & 8.4 & 7.3 &  \color{red}{16.3}  & 21.4 \\ \hline
    \hline
    \hline
    $\Delta v_{x,\text{rel}}$ (T=2) & 12.7 & \color{red}{23.1} & 5.0 & -0.3 & \color{red}{14.3} & \color{red}{13.0} & 8.1 & 10.5 &  \color{red}{17.1} & 10.4 \\ \hline
    $\Delta v_{y,\text{rel}}$ (T=2) & 8.7 & \color{red}{22.9}& -3.0 & 13.7 & \color{red}{0.2} & \color{red}{6.0} & 9.1 & 3.0 &  \color{red}{14.9} & 22.5 \\ \hline
  \end{tabular}
 \caption{Average percentage improvement (over 51 segments in the test trajectory) in prediction errors for RS-IRL over RN-IRL. The RS-IRL predictions with $T=2$ for $x_{\text{rel}}$ and $v_{x,\text{rel}}$ are more accurate than those for RN-IRL for all but one participant, with as much as $23.1\%$ average improvement. The most pronounced improvements (highlighted in red) corresponded to participants with higher levels of risk- and/or ambiguity-aversion, some of whom are studied in detail in the case studies. The improvements in the lateral direction are less significant since the primary source of risk and ambiguity was in the longitudinal dynamics. }
 \label{tab:position velocity errors N2}
\end{table}

\section{Discussion and Conclusions}
\label{sec:conclusion}

We have presented an approach for IRL that explicitly accounts for risk \emph{and} ambiguity sensitivity in experts. We proposed a flexible modeling framework based on coherent risk measures that allows us to capture an entire spectrum of risk assessments from risk-neutral to worst-case for a rich class of static and dynamic decision-making settings. We developed efficient LP based non-parametric algorithms for static, and likelihood based semi-parametric algorithms for dynamic decision making settings. Notably, we significantly improved the modeling framework in~\citep{MajumdarSinghEtAl2017} for dynamic decision making, \revision{and verified the performance improvements from RS-IRL despite transitioning} from the exact LP iteration to semi-parametric likelihood based algorithms. The proposed inference framework was rigorously evaluated on a realistic simulated driving game with ten participants and shown to be able to infer and mimic qualitatively different driving styles ranging from risk-neutral to highly risk-averse in a data-efficient manner, while more accurately capturing participant behavior than \revision{with} a risk-neutral model. Most importantly, by performing inference using the dual representation of coherent risk measures, we retain the generality to be able to recover any risk measure within such a class of risk measures, \emph{without} assuming any a priori knowledge (e.g., fixed disutility function and/or risk measure). 

\revision{Throughout this work, we assumed a \emph{discrete} model of uncertainty for both the static and dynamic decision-making settings. While one would like to be able to address large or continuous sets of disturbances, we believe that a hierarchical representation of uncertainty is a more tractable approach. For instance, at the higher-level, one reasons about various uncertain \emph{modes} of operation (e.g., the random erratic car maneuvers). At the lower-level, conditional on a given mode (e.g., deceleration), one may consider continuous models of uncertainty (e.g., the set of all robot deceleration profiles). The overall framework thus constitutes a mixture model. At the continuous lower-level of uncertainty (e.g., due to the natural variance in demonstrations), aspects such as risk-sensitivity are less relevant. Thus, this work studies risk-sensitivity at the \emph{higher-level} hierarchy of decision making where discrete/modal models of uncertainty induce more nuanced  behavior.}

\revision{This paper opens several directions for future research.} First, as in the majority of IRL literature, we hand-picked features for the driving game. While performance on the test trajectory given $\sim$ 1 minute of training data supported \revision{our} choice of features, incorporating risk-sensitivity in large-scale IRL algorithms requires automatic feature extraction. There has been some recent work on using deep neural nets \revision{within} the MaxEnt IRL framework~\citep{WulfmeierOndruskaEtAl2015} as well as using general non-linear cost representations~\citep{FinnLevineEtAl2016}. A promising area of future research then is to embed the semi-parametric approach \revision{proposed in this paper} within deep cost networks to yield RS-IRL algorithms for high-dimensional systems. 

\revision{Second, our}  inference framework assumes that the human expert is subject to an independent (non-interactive) source of disturbance. The natural extension therefore is to modularize the entire risk-sensitive IRL algorithm within a game-theoretic \emph{interactive} setting involving multiple human agents and robotic systems. The key challenge here is to efficiently balance offline and online learning (see e.g., \citep{SadighSastryEtAl2016b, WaughZiebartEtAl2011}), to enable the autonomous robot to actively infer intent and risk-sensitive preferences for the human agents, and use the resulting information to consequently \emph{influence} the human agents.

\revision{Finally,} as a direct extension of the driving game and the game-theoretic adaptation of this work, we plan on testing our algorithms on an autonomous car testbed.


We believe that the approach described here along with the indicated future directions represent an important step towards endowing \revision{future} robotic systems with the ability to predict, infer, and mimic risk-sensitive behavior, which is crucial for safety-critical applications where humans and robots interact.

\footnotesize{
\section*{Acknowledgments}
The authors were partially supported by the Office of Naval Research, Science of Autonomy Program, under Contract N00014-15-1-2673, and by the Toyota Research Institute (``TRI"). This article solely reflects the opinions and conclusions of its authors and not ONR, TRI or any other Toyota entity.

The authors would also like to acknowledge the contributions of Ajay Mandlekar towards the conference version of this paper presented at RSS 2017.
}


\normalsize

\bibliographystyle{SageH}
\bibliography{../../../bib/main,../../../bib/ASL_papers}

\begin{appendices}
\section{Theoretical Guarantees}\label{app:theoretical}
\subsection{Proof of Theorem \ref{thm: consistency algorithm}}

In order to prove the algorithm's consistency, we need to establish \revision{a few} intermediate results. Since $\Pp_{\infty}$ is an intersection of convex (respectively, compact) sets, it is also convex (respectively, compact). Moreover, since each $\mathcal{P}_D$ contains the expert's polytope $\mathcal{P}$, then $\Pp_{\infty}$ is also an outer approximation of $\mathcal{P}$. In particular, $\Pp_{\infty}$ is not empty. 

Denote with $d_H$ the Hausdorff distance between subsets of $\mathbb{R}^L$ associated with the Euclidean norm $||.||$, i.e., for two subsets $A$ and $B$ of $\mathbb{R}^L$, 
\begin{equation} \label{eq: hausdorff distance} 
d_H(A,B) := \max\left\{ \sup_{b \in B} \inf_{a \in A} \|a-b\|, \, \, \sup_{a \in A} \inf_{b \in B} \|a-b\| \right\}. \end{equation} 
The Hausdorff distance defines a metric on the set of non-empty compact subsets of $\mathbb{R}^L$, that we use to measure the distance between risk envelopes. The sequence $\mathcal{P}_{d}$ is non-increasing (for set inclusion). Therefore, $\left\{d_H\left( \mathcal{P}_{d}, \mathcal{P}_{\infty}\right)\right\}_{d \ge 1}$ is a non-increasing sequence of non-negative real numbers. In particular, we have the following result:
\begin{lemma}[Convergence in Hausdorff metric]
The sequence  $\left\{d_H\left( \mathcal{P}_{d}, \mathcal{P}_{\infty}\right)\right\}$ goes to $0$ when $d \to \infty$.
\end{lemma}

\begin{lemma}[Compact uniform convergence]
\label{thm: compact uniform convergence}
Consider a sequence $\left\{\mathcal{Q}_n\right\}_{n \ge 1}$ of compact convex subsets of $\Delta^L$ such that for all $n \ge 1$, $\mathcal{P}_{\infty} \subseteq \mathcal{Q}_n$ and $\lim_{n \to \infty }d_H(\mathcal{Q}_n, \mathcal{P}_{\infty}) = 0$. Consider also a sequence of states $\{x_n\}_{n \ge 1}$ such that $x_n \to x^*$ when $n \to \infty$. 
Define the functions $\varphi_n(u) := \max_{v \in \mathcal{Q}_n} v^T g(x_n,u)$, $\varphi_{n,\infty}(u) := \max_{v \in \mathcal{P}_{\infty}} v^T g(x_n,u)$ and $\varphi(u) := \max_{v \in \mathcal{P}_{\infty}} v^T g(x^*,u)$. Then, for any compact \revision{set} $\mathcal{K} \subseteq \mathcal{U}$:
\begin{equation}
\lim_{n \to \infty} \sup_{u \in \mathcal{K}}|\varphi_n(u) - \varphi(u)| = 0.
\end{equation}
\end{lemma}
\begin{proof}
Fix $u \in \mathcal{U}$. For any $n \ge 1$, denote with $v_n \in \mathcal{Q}_n$ a point such that $\varphi_n(u) = v_n^T g(x_n,u)$, and \revision{with} $v_{n,\infty} \in \mathcal{P}_{\infty}$ a point such that $\varphi_{n,\infty}(u) = v_{n,\infty}^T \, g(x_n,u)$. Let $\Gamma_{\infty}$ be the projection operator onto the compact convex set $\mathcal{P}_{\infty}$. Then,
\begin{equation}
\label{eq: some inequalities}
\Gamma_{\infty}(v_n)^T g(x_n,u) \le v_{n, \infty}^T g(x_n,u) \le v_n^T g(x_n,u).
\end{equation}
The second inequality in \eqref{eq: some inequalities} results from the fact that $\mathcal{P}_{\infty} \subseteq \mathcal{Q}_n$. By Cauchy-Schwarz inequality,
\begin{equation}
\label{eq: cs}
|\left(\Gamma_{\infty}(v_n)-v_n\right)^T g(x_n,u)|  \le ||\Gamma_{\infty}(v_n) - v_n||_2 ||g(x_n,u)||_2. 
\end{equation}
Since $||\Gamma_{\infty}(v_n) - v_n||_2 \le d_H(\mathcal{Q}_n, \mathcal{P}_{\infty})$ and $\lim_{n\to\infty}x_n = x^*$, we get that the left-hand side (LHS) in \eqref{eq: cs} tends to $0$ when $n \rightarrow \infty$, which implies that \begin{equation}\label{eq:common limit} \lim_{n \to \infty} \varphi_n(u) - \varphi_{n, \infty}(u) = 0.\end{equation}

\revision{Since $g$ is continuous with respect to the state variable $x$ and $\mathcal{P}_{\infty}$ is compact, we get that $x \mapsto \max_{v \in \mathcal{P}_{\infty}} v^T g(x,u)$ is also continuous with respect to $x$}. Therefore, $\lim_{n \to \infty}\varphi_{n, \infty}(u) = \varphi(u)$. Using \eqref{eq:common limit}, $\lim_{n \to \infty} \varphi_{n}(u) = \varphi(u)$. But, by {Theorem 10.8} in \citep{Rockafellar2007}, pointwise convergence of a sequence of convex functions over $\mathcal{U}$ implies uniform convergence over any compact set $\mathcal{K} \subseteq \mathcal{U}$, which is the desired result.
\end{proof}

\revision{Since we assumed the cost functions to be strictly convex with respect to $u$, the risk-sensitive optimization problem admits a unique optimal control. 
\begin{lemma}[Strict convexity of risk]
\label{thm: strict convexity}
For any compact subset $\mathcal{B} \subseteq \Delta^L$ and at any state $x$, the function $u \mapsto \max_{v \in \mathcal{B}} v^T g(x,u)$ is strictly convex. In particular, it admits a unique minimizer.
\end{lemma}
\begin{proof}
Fix a state $x$. Take $u_1, u_2 \in \mathcal{U}$ and $\alpha \in [0,1]$. Denote $u_{\alpha} = (1-\alpha) u_1 + \alpha u_2$. Since $\mathcal{B}$ is compact, there exists $\bar{v}\in \mathcal{B}$ such that: $\bar{v}^T g(x,u_{\alpha}) = \max_{v \in \mathcal{B}} v^T g(x,u_{\alpha})$. By strict convexity of $u \mapsto \bar{v}^T g(x,u)$, it follows that: $\bar{v}^T g(x,u_{\alpha}) < (1-\alpha) \bar{v}^T g(x,u_1) + \alpha \bar{v}^T g(x,u_2)$. Taking the worst-case for both terms of the previous inequality's right-hand side, we get that: $\bar{v}^T g(x,u_{\alpha}) < (1-\alpha) \max_{v \in \mathcal{B}} v^T g(x, u_1) + \alpha \max_{v \in \mathcal{B}} v^T g(x,u_2)$, which proves strict convexity of the function $u \mapsto \max_{v \in \mathcal{B}} v^T g(x,u)$. Since the latter function has, by assumption, bounded level sets, it admits a minimizer, which is unique by strict convexity.
\end{proof}}

\begin{lemma}[Optimal control convergence]
\label{thm: argmin convergence}
Define the sequence of functions $\varphi_n$ and $\varphi$ as in Lemma \ref{thm: compact uniform convergence}. Denote $u_n := \argmin_{u\in \U} \varphi_n(u)$ and $u^* := \argmin_{u\in \U} \varphi(u)$ (each of these minima are unique by strict convexity). Then:
\begin{equation}
\label{eq: argmin limit}
\lim_{n \to \infty} u_n = u^*.
\end{equation}
\end{lemma}
\begin{proof}
\revision{
Let $\tau_n := \min_{u\in \U} \varphi_n(u)$ and $\tau := \min_{u \in \U} \varphi(u)$. By construction, the control set $\U$ is a compact convex set. Thus, by Lemma~\ref{thm: compact uniform convergence}, the function $\varphi_n$ converges to $\varphi$ uniformly over $\U$. Thus, for any given $\epsilon > 0$, there exists an $n_0(\epsilon) \in \mathbb{N}$ such that for all $n > n_0(\epsilon)$, 
\[
	| \varphi_n(u) - \varphi(u) | \leq \epsilon \quad \forall u \in \U. 
\]
It follows that for $n > n_0 (\epsilon)$, we have:
\begin{equation}
\begin{split}
	| \varphi_n(u^*) - \varphi(u^*) |  = | \varphi_n(u^*) - \tau | &\leq \epsilon\\
	| \varphi_n(u_n) - \varphi (u_n) | = | \tau_n - \varphi (u_n) | &\leq \epsilon.
\end{split}
\label{eq:var_n_bnd}
\end{equation}
Furthermore, 
\[
	\tau_n - \varphi(u_n) \leq \tau_n - \tau \leq \varphi_n(u^*) - \tau,
\]
since $\varphi(u) \geq \tau$ for all $u\in \U$ and $\tau_n \leq \varphi_n(u)$ for all $u \in \U$. Combining this with eq.~\eqref{eq:var_n_bnd}, we have:
\[
	-\epsilon \leq \tau_n - \tau \leq \epsilon, \quad \forall n  > n_0(\epsilon).
\]
Thus, $\tau_n \rightarrow \tau$ as $n\rightarrow \infty$. }
We proceed by contradiction to prove \eqref{eq: argmin limit}. Assume that $\{u_n\}_{n \ge 1}$ does not converge to $u^*$. Without loss of generality, assume that there exists some $\eta >0$ such that for all $n$, $||u_n - u^*|| \ge \eta$. Define $u^{\prime}_n$ and $\alpha_n \in [0,1]$ such that $u^{\prime}_n = \alpha_n u_n + (1-\alpha_n) u^*$, and $||u^* - u^{\prime}_n || = \frac{\eta}{2}$. By convexity of $\varphi_n$: 
\begin{equation}
\label{eq:convexitymax}
\varphi_n(u^{\prime}_n) \le \alpha_n \varphi_n(u_n) + (1-\alpha_n) \varphi_n(u^*) = \alpha_n \tau_n + (1-\alpha_n) \varphi_n(u^*).
\end{equation}
By the pointwise convergence property of $\varphi_n$, we have that $ \lim_{n\rightarrow \infty} \varphi_n(u^*)  = \varphi(u^*) = \tau$. \revision{Furthermore, we have also established that $\lim_{n\rightarrow \infty} \tau_n = \tau^*$.} Thus, the RHS in eq.~\eqref{eq:convexitymax} converges to $\varphi(u^*)$. 
For the LHS, assume that the sequence $\{u^{\prime}_n\}_{n \ge 1}$ converges to $u^{\prime}$ \revision{(or consider a converging subsequence, which is allowed since $\U$ is compact)}. Then, by continuity of $\varphi$ and uniform convergence of $\varphi_n$, it follows that: $\lim_{n \to \infty} \varphi_n(u^{\prime}_n) = \varphi(u^{\prime}) $. Using \eqref{eq:convexitymax}, we get: $\varphi(u^{\prime}) \le \varphi(u^*)$. But $||u^{\prime} - u^*|| = \frac{\eta}{2}$ and by uniqueness of the minimum, this is a contradiction.
\end{proof}

We are now ready to prove Theorem \ref{thm: consistency algorithm}.

\begin{proof}{(Theorem \ref{thm: consistency algorithm}).}
Fix any $x^* \in \mathcal{S}$ and choose a subsequence of $\left\{\mathcal{P}_d\right\}_{d \ge 1}$ (that we still denote $\left\{\mathcal{P}_d\right\}$ for simplicity) such that $x^{*,d} \to x^*$. For any $d \ge 1$ and corresponding demonstration $(x^{*,d}, u^{*,d})$, according to the update of the outer approximation $\mathcal{P}_d$ in Algorithm \ref{a:bound polytope},
\begin{equation} \label{eq: crucial observation} u(\mathcal{P}_d, x^{*,d}) = u(\mathcal{P}, x^{*,d}).\end{equation}
\revision{We justify \eqref{eq: crucial observation}. Given the demonstration pair $(x^{*,d}, u^{*,d})$ where by definition, $u^{*,d} = u(\mathcal{P}, x^{*,d})$, let $\bar{v} \in \mathcal{P}_{d-1}$,  $\bar{\sigma}_+$, $\bar{\sigma}_-$ be solutions of: 
\begin{alignat}{3}
\label{eq: optimal vertex}
\max\limits_{\substack{v \in \mathcal{P}_{d-1} \\ \sigma_+, \sigma_- \geq 0}}& & \quad &g(x^{*,d},u^{*,d})^T v    \\
s.t.& & &0 = \Eval{\nabla_{u(j)} g(x,u)^Tv}{x^{*,d},u^{*,d}}{} + \sigma_+(j), \forall j \in \J^+  & & \nonumber \\
& & &0 = \Eval{\nabla_{u(j)} g(x,u)^Tv}{x^{*,d},u^{*,d}}{} - \sigma_-(j), \forall j \in \J^- \nonumber \\
& & &0 = \Eval{\nabla_{u(j)} g(x,u)^Tv}{x^{*,d},u^{*,d}}{},  \forall j  \notin \J^+, j \notin \J^- \nonumber  \\
& & & \sigma_{+}(j) = 0 , \  \sigma_{-} (j) = 0, \quad  \forall j \notin \J^+, j \notin \J^- \nonumber
\end{alignat}
where $\J^+ = \{j \in \{1,\dots,m\} \,|\, u^{*,d}(j) = u^+(j) \}$ and $\J^- = \{j \in \{1,\dots,m\} \,|\, u^{*,d}(j) = u^-(j) \}$.
According to Algorithm \ref{a:bound polytope}, $\mathcal{P}_d = \{v \in \mathcal{P}_{d-1}\,|\, v^T g(x^{*,d}, u^{*,d}) \le \bar{v}^T g(x^{*,d}, u^{*,d})\}$. Moreover, $\bar{v} \in \mathcal{P}_d$, which implies: 
\begin{equation}
\label{eq: vbar characterization}
\max_{v \in \mathcal{P}_d} v^T g(x^{*,d}, u^{*,d}) = \bar{v}^T g(x^{*,d}, u^{*,d}).
\end{equation}
The set of equations given by the constraints of Problem \eqref{eq: optimal vertex}, with $v$, $\sigma_+$ and $\sigma_-$ fixed and respectively equal to $\bar{v}$, $\bar{\sigma}_+$ and $\bar{\sigma}_-$, are exactly the optimality conditions for the solution of the following convex optimization problem: 
\begin{equation}
\min_{u \in \mathcal{U}} \bar{v}^T g(x^{*,d}, u).
\end{equation} 
Since $u^{*,d}$ satisfies those optimality conditions and using \eqref{eq: vbar characterization}, we have:
\begin{equation}
\label{eq: saddle point}
\begin{split}
\max_{v \in \mathcal{P}_d} v^T g(x^{*,d}, u^{*,d}) & = \bar{v}^T g(x^{*,d}, u^{*,d})\\
&  = \min_{u \in \mathcal{U}} \bar{v}^T g(x^{*,d}, u) \\
& \le \min_{u \in \mathcal{U}} \max_{v \in \mathcal{P}_d} v^T g(x^{*,d}, u) \\
&\le \max_{v \in \mathcal{P}_d} v^T g(x^{*,d}, u^{*,d})
\end{split}
\end{equation}
Hence, all inequalities in \eqref{eq: saddle point} are equalities. In particular, 
\begin{equation}
\min_{u \in \mathcal{U}} \max_{v \in \mathcal{P}_d}v^T g(x^{*,d}, u) = \max_{v \in \mathcal{P}_d} v^T g(x^{*,d}, u^{*,d})
\end{equation}. 
By uniqueness of the minimum as given by Lemma \ref{thm: strict convexity}, it follows that $u^{*,d} = u(\mathcal{P}_d, x^{*,d})$.}
From Lemma \ref{thm: argmin convergence}, we have that $\lim_{d \to \infty}u\left(\mathcal{P}_d,x^{*,d}\right) = u\left(\mathcal{P}_{\infty}, x^*\right)$. From equation \eqref{eq: crucial observation}, we get that $\lim_{d\to \infty}u\left(\mathcal{P}_d,x^{*,d}\right) = u\left(\mathcal{P}, x^*\right)$. Combining the two previous observations, we get the desired result, i.e., $u\left(\mathcal{P}, x^*\right) = u\left(\mathcal{P}_{\infty},x^*\right)$.
\end{proof}

\section{Derivation of Prepare-React Policy Likelihood}\label{app:full_like}

Recall the multi-stage optimization problem objective, repeated here for convenience:
\[\small
 C_{0:N-n_d} + \rho_0\bigg( C_{N-n_d+1:N-1} + C_{N:2N-n_d} + \rho_1 \big( C_{2N-n_d+1:2N-1} + \cdots + \rho_{T-1}\left( C_{TN-n_d+1:TN-1} \right) \cdots \big) \bigg).  
\]
For a ``prepare'' -- ``react'' policy $\hat{\pi}_t$ at stage $t$, let $\hat{\pi}_{t|\mathfrak{p}}$ denote the ``prepare'' portion and $\hat{\pi}_{t|\mathfrak{r}}$ denote the ``react'' portion. Then, we can re-write the objective above to show explicit dependence as follows:
\[
\begin{split}
 C_{0:N-n_d}(\cdot,\hat{\pi}_{0|\mathfrak{p}}) + \rho_0\bigg( &C_{N-n_d+1:N-1}(\cdot, \hat{\pi}_{0|\mathfrak{r}}) + C_{N:2N-n_d}(\cdot,\hat{\pi}_{1|\mathfrak{p}}) +  \\
 									      &\rho_1 \big( C_{2N-n_d+1:2N-1}(\cdot,\hat{\pi}_{1|\mathfrak{r}}) + \cdots + \rho_{T-1}\left( C_{TN-n_d+1:TN-1}(\cdot,\hat{\pi}_{T-1|\mathfrak{r}}) \right) \cdots \big) \bigg).
 \end{split}  
\]
%
Note that the expression within the large brackets is a random variable in $\RR^L$, indexed by all possible realizations of $w_0'$, and a function of the first ``prepare'' sequence $\hat{\pi}_{0|\mathfrak{p}}$. Thus, for a given ``prepare'' sequence $\hat{\pi}_{0|\mathfrak{p}}$ and first disturbance mode $w_0'$, define the optimal tail cost as a function of $\hat{\pi}_{0|\mathfrak{r}}$:
\[
\begin{split}
	\tau[\hat{\pi}_{0|\mathfrak{p}}, w_0'](\hat{\pi}_{0|\mathfrak{r}}) &:= C_{N-n_d+1:N-1}(\cdot, \hat{\pi}_{0|\mathfrak{r}}) + \\
	&\min\limits_{\substack{\hat{\pi}_t \\ t \in [1,T-1]}} \rho_1 \big( C_{N:2N-1}(\cdot, \hat{\pi}_1) + \cdots + \rho_{T-1}( C_{(T-1)N:TN-1}( \cdot,\hat{\pi}_{T-1}) ) \big).
\end{split}
\]
Then, we define the \emph{conditional} distribution for the ``react'' sequence corresponding to disturbance mode $w_0'$, given the first ``prepare'' sequence, as
\[
	\mathrm{Pr} (\hat{\pi}_{0|\mathfrak{r}} \mid  \hat{\pi}_{0|\mathfrak{p}} \ ; w_0' ) \propto \exp\left( - \tau[\hat{\pi}_{0|\mathfrak{p}}, w_0'](\hat{\pi}_{0|\mathfrak{r}}) \right).
\]
The distribution for the first ``prepare'' sequence is \revision{then given by}
\[
	\mathrm{Pr} (\hat{\pi}_{0|\mathfrak{p}}) \propto \exp\left( - \left[ C_{0:N-n_d}(\cdot,\hat{\pi}_{0|\mathfrak{p}}) + \rho^r \left( \softmin_{\hat{\pi}_{0|\mathfrak{r}}} \tau[\hat{\pi}_{0|\mathfrak{p}}, w_0'](\hat{\pi}_{0|\mathfrak{r}}) \right) \right] \right),
\]
where we use $\softmin$ in place of $\min$ to ensure differentiability. Thus, for an observed ``prepare'' -- ``react'' sequence $\hat{\pi}_{0}^*$ associated with the observed disturbance mode $w^*_0$, we obtain
\[
\mathrm{Pr}(\hat{\pi}^*_0) = 	\mathrm{Pr}(\hat{\pi}_{0|\mathfrak{p}}^*) \cdot \mathrm{Pr}(\hat{\pi}_{0|\mathfrak{r}}^* \mid \hat{\pi}_{0|\mathfrak{p}}^* \ ; w_{0}^* ).
\]

\section{Likelihood Gradient Computations}\label{app:gradients}

Define
\begin{equation}
	\sigma [w^*_{-1|tN}](\hat{u}) := \dfrac{\exp \left(-\beta \tilde{\tau}[w_{-1|tN}^*] (\hat{u}) \right)}{\sum_{\hat{u}'} \exp\left(-\beta  \tilde{\tau}[w_{-1|tN}^*] (\hat{u}')\right)}
\label{boltz}
\end{equation}
to be the probability of choosing action trajectory $\hat{u}$ at time-step $tN$ as assumed by the Boltzmann likelihood model in~\eqref{boltz_def}. Then, the gradient of the log-likelihood in~\eqref{like_traj} with respect to parameter $s \in \{r,c\}$ is given by
\begin{equation}
	\dfrac{\beta}{|\mathcal{T}^*|} \sum_{\hat{u}^*_t \in \mathcal{T}^*} \left[ \sum_{\hat{u} \neq \hat{u}^*_t} \sigma[w^*_{-1|tN}](\hat{u}) \nabla_s \tilde{\tau}[w^*_{-1|tN}](\hat{u}) + \left(\sigma[w^*_{-1|tN}](\hat{u}^*_t) -1 \right) \nabla_s \tilde{\tau}[w^*_{-1|tN}](\hat{u}^*_t)  \right].
\label{like_deriv}
\end{equation}

For notational clarity, we use $t$ to denote the $t^{\text{th}}$ $N$-step segment in the demonstrated trajectory $\mathcal{T}^*$ and $t'$ as the stage-wise index within the multi-step planning problem. From equations~\eqref{tau_term_soft},~\eqref{tau_int_soft}, and~\eqref{tau_soft}, we see that the derivative of $\tilde{\tau}[w^*_{-1|tN}](\hat{u})$ can be computed through a recursive implementation of the chain rule, starting from the terminal stage. In the event that all nested LPs are non-degenerate, we obtain the following recursive set of equations for computing the gradients.
\vspace{\baselineskip}

\noindent \emph{Terminal Stage}: From eq.~\eqref{tau_term_soft}, $\tilde{\tau}[\bm{u}_{T-2}, \bm{\omega}_{T-2}](\hat{u})$ is the optimal value of the LP:
\begin{equation}
	\max_{v\in \Pp_r} \left( g(w_{T}', \hat{u} ; c) \right)^T v,
\label{term_LP}
\end{equation}
where $g \in \RR^L$ is the accumulated cost vector over the terminal stage, $C_{(T-1)N: TN-1}$, indexed by the terminal disturbance mode $w_T'$. Let $v^*$ denote the optimal primal solution and $\lambda^*$ the optimal dual variables for the constraints defined in~\eqref{P_param}. Then,
\begin{align}
\nabla_r \tilde{\tau}[\bm{u}_{T-2}, \bm{\omega}_{T-2}](\hat{u}) &=  - \lambda^*, \label{left_r_term} \\
\nabla_c \tilde{\tau}[\bm{u}_{T-2}, \bm{\omega}_{T-2}](\hat{u}) &= \sum_{j = 1}^{L} v^*(j) \, \left[ \Phi^{[j]} (\hat{u}) \right], \label{right_c_term}
\end{align}
where $\Phi^{[j]}(\hat{u})$ is the feature vector sum over $N$ time steps corresponding to $C_{(T-1)N:TN-1}$, given\footnote{For notational convenience, we omit the obvious dependence on state.} action trajectory $\hat{u}$ and disturbance $w^{[j]}$.
\vspace{\baselineskip}

\noindent \emph{Recursion}: For  $t' \in \{0,\ldots,T-2\}$, $\tilde{\tau}[\bm{u}_{t'-1}, \bm{\omega}_{t'-1}](\hat{u})$ is the optimal value of the LP\footnote{For notational simplicity, take $\bm{u}_{-1} = \{\}$.}:
\[
	\max_{v \in \Pp_r} \left( g(w_{t}', \hat{u} ; c) + \tilde{g}(w_{t}', \hat{u} ; c, r) \right)^T v,
\]
where $g \in \RR^L$ is the accumulated cost vector $C_{t'N:(t'+1)N-1}$, and $\tilde{g} \in \RR^L$ is the vector of $\softmin_{\hat{u}'}$ over the tail risk-sensitive costs, i.e., $\tilde{\tau}[\{\bm{u}_{t'-1}, \hat{u} \}, \{\bm{\omega}_{t'-1}, w_t' \}] (\hat{u}')$, indexed by the next disturbance mode $w_t'$, and parameterized with respect to $r,c$. Recall that $w_{-1}' = w^*_{-1|tN}$. Let $v^*$ denote the optimal primal solution and $\lambda^*$ the optimal dual variables for the constraints in eq.~\eqref{P_param}. Then,
\begin{align}
\nabla_r \tilde{\tau}[\bm{u}_{t'-1}, \bm{\omega}_{t'-1}](\hat{u}) &= \sum_{j = 1}^{L}  v^*(j) \, \nabla_r \tilde{g} (w^{[j]}, \hat{u} ; c, r) - \lambda^*, \label{left_r} \\
\nabla_c\, \tilde{\tau}[\bm{u}_{t'-1}, \bm{\omega}_{t'-1}](\hat{u}) &= \sum_{j = 1}^{L} v^*(j)\, \left[ \Phi^{[j]} (\hat{u}) + \nabla_c\, \tilde{g} (w^{[j]}, \hat{u} ; c, r) \right], \label{right_c}
\end{align}
where $\Phi^{[j]}$ is the feature vector sum corresponding to $C_{t'N: (t'+1)N-1}$. The gradients of the softmin vector are given by:
\begin{align}
\nabla_r \tilde{g} (w^{[j]}, \hat{u} ; c, r) &= \mathbb{E}_{\hat{u}' \sim \sigma\left[\tilde{\tau}[\{\bm{u}_{t'-1}, \hat{u} \}, \{\bm{\omega}_{t'-1}, w^{[j]} \}]\right]}  \left[ \nabla_r \tilde{\tau}[\{\bm{u}_{t'-1}, \hat{u} \}, \{\bm{\omega}_{t'-1}, w^{[j]} \}] (\hat{u}') \right], \label{left_r_cont} \\
\nabla_c \tilde{g} (w^{[j]}, \hat{u} ; c, r) &= \mathbb{E}_{\hat{u}' \sim \sigma\left[\tilde{\tau}[\{\bm{u}_{t'-1}, \hat{u} \}, \{\bm{\omega}_{t'-1}, w^{[j]} \}]\right]}  \left[ \nabla_c \tilde{\tau}[\{\bm{u}_{t'-1}, \hat{u} \}, \{\bm{\omega}_{t'-1}, w^{[j]} \}] (\hat{u}') \right] \label{right_c_cont},
\end{align}
where $\sigma\left[\tilde{\tau}[ \bm{u}_{t'} , \bm{\omega}_{t'} ]\right]$ is the discrete (cost-based) Boltzmann distribution, i.e., 
\[
	\sigma\left[ \tilde{\tau}[ \bm{u}_{t'} , \bm{\omega}_{t'}  ]\right] (\hat{u}') \propto \exp\left( -\tilde{\tau}[ \bm{u}_{t'} , \bm{\omega}_{t'}]  (\hat{u}') \right) . 
\]
For our experiments, we found a different form of the $\softmin$ function to be more numerically stable. In particular, we used
\[
	\softmin_{\beta} f(x) = \mathbb{E}_{x \sim \sigma_{\beta}[f] } f(x),
\]
where $\sigma_{\beta}[f]$ is the Boltzmann distribution defined with inverse temperature $\beta$ as:
\[
	\sigma_{\beta}[f](x) \propto \exp\left( -\beta f(x) \right).
\]
The gradients take the exact same form as~\eqref{left_r_cont} and~\eqref{right_c_cont} with $\sigma$ replaced by $\sigma_{\beta}$.

Thus, to compute the gradient in~\eqref{like_deriv} for the $t^{\text{th}}$ demonstration, one would start with the terminal stage derivatives in~\eqref{left_r_term} and~\eqref{right_c_term} for the planning problem defined at time step $tN$, and proceed backwards inductively using~\eqref{left_r}--\eqref{right_c_cont} to arrive at $\nabla \tilde{\tau}[w^*_{-1|tN}](\hat{u})$.

Note that the derivation above assumes non-degeneracy of the LPs. It is readily observed that by the piecewise linearity of LPs with respect to \revision{the} objective coefficients and constraint right-hand-sides, and the Lipschitz property of the $\softmin$ function, $\tilde{\tau}$ is locally Lipschitz. Consequently, the log-likelihood too is locally Lipschitz. Thus, by the Rademacher theorem, the log-likelihood is non-differentiable only over a Lebesgue set of measure zero. If, however, during the updates, any (nested) LP is primal degenerate (multiple dual optimal solutions) or dual degenerate (multiple primal optimal solutions), the log-likelihood is non-differentiable (despite directional-derivatives existing in all directions). Thus, in its full generality, the max likelihood problem corresponds to a non-convex, non-smooth optimization and at points of non-differentiability, one must do extra work to compute a suitable descent direction. Some proposed approaches in the literature include penalized smoothing~\citep{Chen2012}, sampling-based estimation~\citep{BurkeLewisEtAl2005} to approximate the Clarke generalized subdifferential and compute a descent direction, and majorization-minimization~\citep{LanzaMorigiEtAl2017} to iteratively optimize an upper-bound on a minimization problem. Arguably, the field is an active area of research. 

In our implementation, we implemented the following two heuristics to avoid non-degeneracy (and consequent non-differentiability of the log-likelihood): 
\begin{itemize}
\item During the projection step of the projected gradient method for $r$, if a constraint $a_j^T v \leq b(j) - r(j)$ is found to be redundant, the parameter $r(j)$ is re-adjusted as $r(j) \leftarrow r(j) + 0.01$, provided that the resulting polytope is not empty. This has the effect of eliminating the possibility of redundant constraints at primal optimal vertices (thereby eliminating primal degeneracy).
\item In the event that the objective vector is parallel to one of the bounding hyperplanes of the region $\Pp_r$ (i.e., dual degeneracy), we added a small distortion to the objective vector.
\end{itemize}

\end{appendices}\end{document}